\documentclass[10pt,twocolumn,letterpaper]{article}

\usepackage[pagenumbers]{cvpr} 

\definecolor{cvprblue}{rgb}{0.21,0.49,0.74}
\usepackage[pagebackref,breaklinks,colorlinks,allcolors=cvprblue]
{hyperref}
\usepackage{wrapfig}
\def\paperID{13720} 
\def\confName{CVPR}
\def\confYear{2025}

\usepackage[utf8]{inputenc} 
\usepackage[T1]{fontenc}    
\usepackage{url}            
\usepackage{booktabs}       
\usepackage{amsfonts}       
\usepackage{nicefrac}       
\usepackage{microtype}      
\usepackage{xcolor}         
\usepackage{amsmath}
\usepackage{bm}
\usepackage{amsthm}
\usepackage{graphicx}
\usepackage{wrapfig}
\usepackage{mdframed}
\usepackage{enumitem}
\usepackage{subcaption}
\usepackage{caption}
\usepackage{booktabs}
\usepackage{amsmath}
\usepackage{algorithm}
\usepackage{algpseudocode}
\DeclareMathOperator*{\argmax}{arg\,max}
\DeclareMathOperator*{\argmin}{arg\,min}
\usepackage{multirow}
\usepackage{amsthm}  

\newtheorem{theorem}{Theorem}[section]

\newtheorem{lemma}[theorem]{Lemma}
\newtheorem{remark}[theorem]{Remark}
\newtheorem{corollary}[theorem]{Corollary}

\newtheorem{definition}[theorem]{Definition}

\newenvironment{customboxeddefinition}
  {\begin{mdframed}[linewidth=0.5pt, linecolor=blue, backgroundcolor=white!20, roundcorner=5pt, nobreak=true] 
   \begin{definition}}
  {\end{definition}\end{mdframed}}

\newcommand\Tstrut{\rule{0pt}{2.6ex}}         
\newcommand\Bstrut{\rule[-0.9ex]{0pt}{0pt}}   

\newcommand{\R}{\mathbb{R}}

\newcommand{\Natural}{\mathbb{N}}

\newcommand{\p}{\mathbf{p}}

\newcommand{\I}{\mathbf{I}}

\newcommand{\M}{\mathbf{M}}
\newcommand{\N}{\mathbb{N}}

\newcommand{\Wmat}{\mathbf{W}}

\newcommand{\Xcal}{\mathcal{X}}
\newcommand{\Ycal}{\mathcal{Y}}
\newcommand{\Zcal}{\mathcal{Z}}

\usepackage{amssymb}
\usepackage{fontawesome5}
\newcommand{\Vtar}{V^{\text{tar}}}
\newcommand{\Dtar}{\mathcal{D}^{\text{tar}}}
\newcommand{\tar}{{\text{target}}}
\newcommand{\Vaux}{V^{\text{aux}}}
\newcommand{\Daux}{\mathcal{D}^{\text{aux}}}
\newcommand{\aux}{{\text{auxiliary}}}

\title{COBRA: COmBinatorial Retrieval Augmentation for Few-Shot Adaptation}

\author{Arnav M. Das$\thanks{These authors contributed equally to this work.} \hspace{1.5mm}$ \quad Gantavya Bhatt\footnotemark[1] \quad Lilly Kumari \quad Sahil Verma \quad Jeff Bilmes \\[.5ex] 
University of Washington, Seattle \\[1ex]
 \texttt{\small \{arnavmd2, gbhatt2, bilmes\}@uw.edu} 
}

\begin{document}
\maketitle
\begin{abstract}
Retrieval augmentation, the practice of retrieving additional data from large auxiliary pools, has emerged as an effective technique for enhancing model performance in the low-data regime. Prior approaches have employed only nearest-neighbor based strategies for data selection, which retrieve auxiliary samples with high similarity to instances in the target task. However, these approaches are prone to selecting highly redundant samples, since they fail to incorporate any notion of diversity. In our work, we first demonstrate that data selection strategies used in prior retrieval-augmented few-shot adaptation settings can be generalized using a class of functions known as Combinatorial Mutual Information (CMI) measures. We then propose COBRA (COmBinatorial Retrieval Augmentation), which employs an alternative CMI measure that considers both diversity and similarity to a target dataset. COBRA consistently outperforms previous retrieval approaches across image classification tasks and few-shot learning techniques when used to retrieve samples from LAION-2B. COBRA introduces negligible computational overhead to the cost of retrieval while providing significant gains in downstream model performance. 
\vspace{-.2in}
\end{abstract}    
\section{Introduction}
\label{sec:intro}

\begin{figure}
    \centering
    \includegraphics[width=0.44\textwidth]{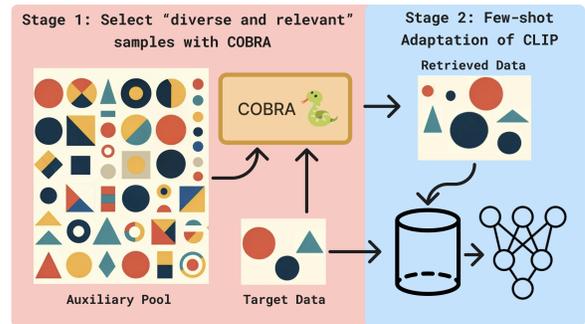}
    \caption{\small COBRA uses the target pool to select diverse and relevant samples from a large web-scale auxiliary pool. The retrieved data and the target data are then used to train a few-shot learner with a CLIP backbone.}
    \label{fig:main-figure}
    \vspace{-.23in}
\end{figure}

With the emergence of web-scale data sources, retrieval has become an immensely popular technique to improve model performance. Retrieval augmented generation (RAG) has achieved much success in natural language, where additional data is retrieved from auxiliary sources to supplement model knowledge and guide the output of large language models (LLMs) at inference time \citep{lewis2020retrieval, izacard2022atlas, shi2023replug, ram2023context}. More recently, retrieval has been leveraged to select additional weakly labeled training samples from external image-caption data sources that are relevant to a target image classification task \citep{liu2023learning, zancato2023train, wallingford2023neural}. This strategy enables practitioners to train performant models on domains where labeled datasets are scarce and small, 
without incurring any extra labeling costs.


\begin{figure*}[h!]
    \centering
    \vspace{-.1in}
    \includegraphics[width=.91\textwidth]{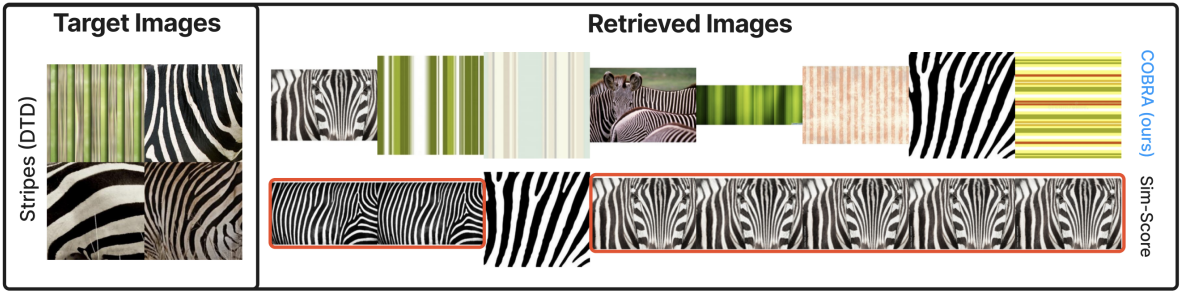}
    \caption{\small \textit{COBRA} retrieves samples from LAION that are diverse and relevant to the target dataset. In contrast, \textit{Sim-Score} retrieves redundant samples and sometimes even exact duplicates (shown in red boxes). Additional visualizations are presented in~\Cref{appen sec: more qualitative}.} 
    \label{fig:qualitative}
    \vspace{-.2in}
\end{figure*}


Web-scale datasets are highly heterogeneous and often contain massive amounts of data that are irrelevant to a specific downstream task \citep{gadre2023datacomp, schuhmann2022laion5b}. Therefore, it is natural for retrieval strategies to assess the relevance of each sample from an auxiliary 
retrieval
pool using some notion of similarity to the target dataset. 
To this end, past work typically considers \emph{nearest-neighbor} based retrieval strategies, which score each auxiliary sample 
from the retrieval pool
based on its similarity to a class of images in the target task, and retrieve the samples with the highest scores \citep{wallingford2023neural, liu2023learning, zancato2023train, cafo, geng2024unmetpromisesynthetictraining}. However, independently scoring the auxiliary data in a sample-wise manner neglects notions of diversity of the retrieved set and may induce redundancy. Diversity has been shown to be a key ingredient in data subset selection for many different tasks~\citep{bhatt2024experimental, bilmes2022submodularity, ash2019deep, sener2017active}, but has not been employed in the context of retrieval
augmentation. In this work, we seek to answer the following question: \emph{\textbf{ Can we improve the effectiveness of retrieval augmentation by using a diversity-aware data selection strategy?}}





We propose the use of a class of functions called Combinatorial Mutual Information (CMI) functions for retrieval. CMI functions capture the similarity between sets~\citep{bilmes2022submodularity, iyer-cmi-alt-2021} and have been employed in various contexts that require targeted data subset selection~\citep{iyer-cmi-alt-2021, kothawade2022prism, koth2021similar}. We demonstrate that CMI measures are inherently well-suited for retrieval augmentation by showing that the underlying objective used in \emph{nearest-neighbor} based retrieval strategies are instances of this function class. To address the deficiencies of \emph{nearest-neighbor} based retrieval, we propose COmBinatorial Retrieval Augmentation (COBRA) which uses an alternative CMI function that retrieves a set of similar and diverse samples. Rather than independently scoring each auxiliary sample, COBRA evaluates sets and can model interactions between selected samples. As shown in Figure~\ref{fig:qualitative}, this property allows COBRA to avoid retrieving semantically redundant images and select a set of images that well represent all samples in the target task. Crucially, the COBRA objective is submodular and therefore can be efficiently optimized with a constant factor approximation guarantee \citep{nemhauser1978analysis}. 



In our work, we evaluate COBRA as a retrieval augmentation strategy within the few-shot setting for image classification, where a target dataset has only a few labeled images (shots) per class and the objective is to improve model performance by retrieving samples from an auxiliary pool. Our experimental pipeline (shown in~\Cref{fig:main-figure}) involves: (1) utilizing a retrieval strategy to select weakly labeled samples from the auxiliary image-caption dataset, such as LAION-2B, to augment a small image classification dataset and (2) using a few-shot adaptation strategy such as~\cite{tipadapter, clipadapter, coop, cocoop, cafo} to adapt CLIP~\citep{radford2021learning} with both the labeled and retrieved samples. Our findings highlight that COBRA consistently outperforms alternative previous retrieval strategies across various downstream/target datasets and few-shot adaptation methods, establishing the necessity of diversity in the data selection strategy. Overall, we make the following contributions: 
\begin{itemize}[itemsep=0.5ex, leftmargin=*]
  \item We propose the use of combinatorial mutual information (CMI) functions in retrieval augmentation, and demonstrate that most existing retrieval strategies are a particular instance of CMI functions.
  \item We introduce COmBinatorial Retrieval Augmentation (COBRA), an alternative CMI measure that considers similarity as well as the previously overlooked notion of diversity.
  \item We demonstrate the efficacy of COBRA in the few-shot setting and show that it consistently outperforms previous retrieval techniques across target datasets and models.
\end{itemize}

\section{Background}
\label{sec: background}

\paragraph{Retrieval Augmentation} We first describe the problem of retrieval augmentation for discriminative classification where we are given a $\tar$ dataset ($\Dtar$) and an $\aux$ dataset ($\Daux$). $\Dtar$ consists of labeled samples and is much smaller in scale compared to $\Daux$, which is generally a pool of heterogeneous data that can come from any arbitrary source and in general has a distinct label space from $\Dtar$. For example, in vision applications where the $\tar$ dataset consists of images and corresponding labels, the $\aux$ dataset may consist of images with natural language descriptions which may or may not be semantically relevant to $\Dtar$. Formally, $\Dtar \triangleq \{(x_i, y_i) \mid x_i \in \Xcal \;, y_i \in \Ycal \}_{i=1}^m$ where we define $\Xcal$ as the domain of input (images) and $\Ycal = \{1, 2, \ldots, C\}$ as the domain of labels. Let $\Daux \triangleq \{z_i \mid z_i \in \Zcal\}_{i=1}^n$ where we define $\Zcal$ as the domain of examples in $\aux$ dataset, which are images with associated metadata. Note that $z_i \in \Daux$ are pairs but their labels need not conform to $\Ycal$. 

For notational simplicity, we jointly index $\Dtar$ and $\Daux$ using $V \triangleq \{1, 2, \ldots, m, m+1, m+2, \ldots, m+n\} = [m+n]$, and define $\Vtar \triangleq \{1, 2, \ldots, m\}$ and $\Vaux \triangleq \{m+1, m+2, \ldots, m+n\}$. For $\Daux$, given any $A\subseteq \Vaux$, we define $\Daux_A =\{z_j |\, j \in A\}$. The goal of retrieval augmentation in our setting is to find $A \subseteq \Vaux$, such that a classifier trained on $\Dtar \cup \Daux_A $ is more performant than the one trained solely on $\Dtar$. Lastly, given any matrix $\M \in \R^{n \times n}$ and two sets $A \subseteq [n]$ and $B \subseteq [n]$, $\M[A, B] \triangleq [m_{i, j}]_{\substack{i \in A \\ j \in B}}$ denotes slicing matrix $\M$'s rows and columns with $A$ and $B$.

Across different works on retrieval augmentation, including classification \citep{wallingford2023neural,zancato2023train, liu2023learning}, text generation~\citep{lewis2020retrieval, izacard2022atlas, shi2023replug, ram2023context}, and image generation~\citep{diffusion_rag} nearest-neighbor based retrieval stands out as being the most widely used strategy. Nearest-neighbor based retrieval strategies assume a similarity matrix $\Wmat \in \mathbb{R}^{(m+n) \times (m+n)}$ where each $w_{i, j}$ is some notion of similarity between examples indexed with $i$ and $j$ and aim to retrieve samples based on the optimization problem defined below.

\begin{definition}[\textit{Nearest-Neighbor Based Retrieval}] ~
\label{def: kNN}
Given a matrix $\Wmat \in \mathbb{R}^{(m+n) \times (m+n)} $ modeling the similarity between elements of $\Dtar$ and $\Daux$, $k \in \N$. Nearest-neighbor based selection aims to choose a subset satisfying the following optimization --
\begin{align}
\label{eq: knn}
    g(A; \Vtar, \Wmat) &= \sum_{j \in A} \sum_{i \in \Vtar} w_{ij} \\
    A^{\ast} &= \underset{\substack{A \subseteq \Vaux \\ |A| \leq k}}{\operatorname{argmax}} \, g(A; \Vtar, \Wmat)
\end{align}
\end{definition}

Many different notions of similarity can be used to instantiate $\Wmat$. In our experiments, we follow the work of~\cite{zancato2023train, wallingford2023neural,liu2023learning} and use CLIP encoders to featurize each sample. As these models are multimodal, we can use image-to-image similarity or text-to-image similarity to instantiate $\Wmat$. In the below, we refer to nearest-neighbor based retrieval instantiated with an image-to-image similarity matrix as \emph{Sim-Score} and nearest-neighbor based retrieval instantiated with a text-to-image similarity matrix as \emph{CLIP-score}.

Note that $g(A)$ has no incentive to select diverse samples (i.e., samples that are mutually dissimilar to each other), the reason being that each $j \in A$ has a score $\sum_{i \in \Vtar} w_{ij}$ independent of any other $j' \in A$, $j' \neq j$. Thus, $A^*$ may be highly redundant and fail to capture the full information present in the dataset. To incorporate diversity, we propose using a class of functions called \textit{Combinatorial Mutual Information (CMI) Functions} that find similar \textbf{and} diverse samples for retrieval augmentation. We first begin by providing a primer on submodular functions and CMI measures. 
We then demonstrate that $g(A; \Vaux, \Vtar)$ defined above is a limited special case of CMI (Lemma~\ref{thm: kNN_as_GCMI}), demonstrating that CMI is a natural choice for retrieval. We then give an alternate instantiation of CMI that is better suited for our problem.\looseness-1
\begin{figure}[h!]
    \centering
    \includegraphics[width=.49\textwidth]{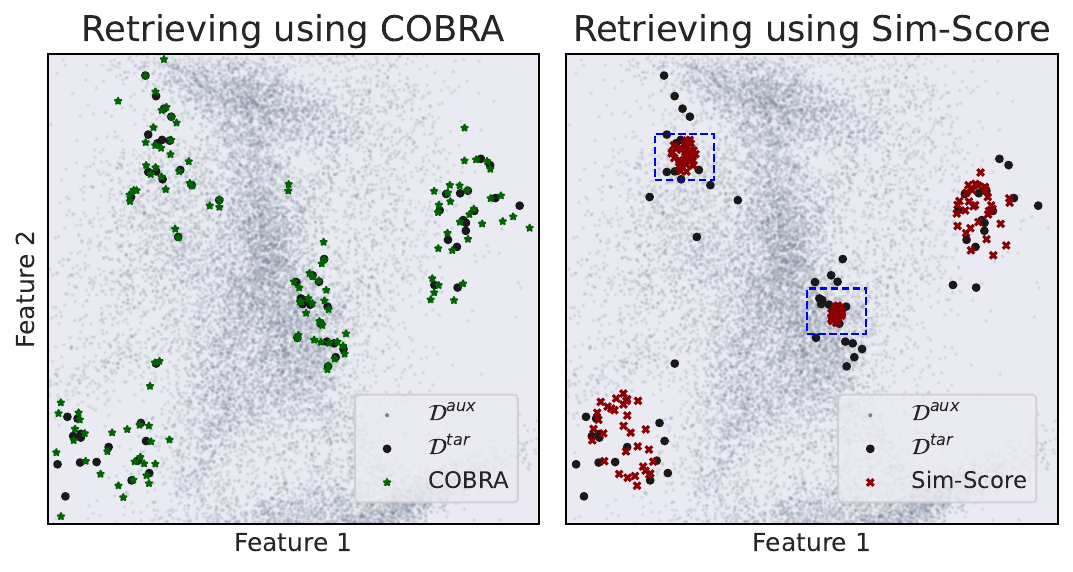}
    \vspace{-.2in}
    \caption{\textbf{2D Example} We consider a simple example where $|\Dtar| = 64$ and $|\Daux| = 25000$. From $\Daux$, we retrieve a subset of size 128 based on $\Dtar$. COBRA (left) effectively covers the target $\Dtar$, on the other hand, \emph{Sim-Score} (right) selects clumpy examples, as highlighted by the bounding boxes. Refer to \Cref{appen sec: toy example setup} for more details.}
    \label{fig:toy}
    \vspace{-.1in}
\end{figure}

\paragraph{Submodular Functions} Submodular functions, defined over a ground set $V$, must satisfy $f(A \cup \{v\}) - f(A) \geq f(B \cup \{v\}) - f(B)$ for subsets $A \subseteq B \subseteq V$ and $v \in V \setminus B$. This diminishing returns property makes submodular functions particularly effective in modeling concepts such as coverage and diversity. Moreover, submodular functions that are nonnegative ($f(A) \geq 0, \forall A \subseteq V$) and monotone nondecreasing ($f(A) \leq f(B), \forall A \subseteq B \subseteq V$) can be efficiently maximized with a greedy algorithm (outlined in Algorithm \ref{alg: greedy} in Appendix \ref{appen sec: submodular max}) with a constant factor approximation guarantee of $1 - e^{-1}$~\citep{nemhauser1978analysis, minoux2005accelerated}. Due to their ability to model desirable properties like diversity and ease of optimization, submodular functions have become widely used in summarization~\citep{lin2011class, lin2012learning, mirzasoleiman2020coresets}, feature selection~\citep{liu2013submodular}, unlabeled data selection~\citep{wei2015submodularity, das2023accelerating, bhatt2024experimental, bhatt2024deep}, and many other applications in machine learning~\citep{bilmes2022submodularity}.

\paragraph{Combinatorial Mutual Information} In information theory, mutual information is used to quantify the information that sets of random variables contain about one another. Mathematically, given a set of random variables $\{X_1, \ldots, X_n\}$ and a joint distribution $p$, entropy is defined as $H(X_1, \ldots, X_n) = -\mathbb{E}_{X_1, \ldots, X_n \sim p} [ \log p(X_1, \ldots, X_n) ]$, and is in fact a known submodular function~\citep{fujishige2005submodular}. The mutual information between two sets of random variables indexed by $A$ and $B$ can then be expressed as $I_H (A; B) = I(X_A; X_B) = H(X_A) + H(X_B) - H(X_{A \cup B})$~\citep{Cover2006}, where $H(X_A) = H(\{X_a \mid a \in A\})$. Combinatorial mutual information (CMI) generalizes the information-theoretic notion of mutual information and can be instantiated over any submodular function $f$. Formally, CMI is defined as $I_f(A; B) = f(A) + f(B) - f(A \cup B)$ for any $A,B \subseteq V$. Intuitively, CMI measures the similarity between sets $A$ and $B$ based on some submodular information measure.

CMI has previously been applied in settings that mandate some form of targeted data subset selection~\citep{kothawade2021Similar, Karanam2022orient, iyer-cmi-alt-2021, subinfomeasures2022, bilmes2022submodularity, koth2021similar}. Within retrieval for augmentation, however, CMI has never been applied \textbf{in its full generality}.
That is, we find that the nearest-neighbor based retrieval strategies (Def.~\ref{def: kNN}) used by~\cite{wallingford2023neural, zancato2023train} and~\cite{liu2023learning} are all just simple instances of CMI, instantiated with a graph cut function (a well known submodular function applied and studied in~\cite{jegelka2010, jegelka2011submodularity, joshi23bdata, bilmes2022submodularity}).



\begin{lemma} [\textit{Graph Cut Mutual Information (GCMI)}]
\label{thm: kNN_as_GCMI}
Let $G = (V, E)$ be a graph with edge weights defined with symmetric $\Wmat \in \mathbb{R}^{(m+n) \times (m+n)} $. For any set $A \subseteq V$ of vertices, let $f(A) = \sum_{i \in A} \sum_{j \in V \setminus A} w_{ij}$ be the graph cut function. Given any two sets $A$ and $B$ such that $ A \cap B = \emptyset $,\looseness-1 
\begin{equation}
\label{eq:GCMI}
    I_{f}(A;B) = 2 \sum_{i \in A} \sum_{j \in B} w_{ij}
\end{equation}
\end{lemma}
\begin{proof}
    Please refer to \Cref{sec: appen background and lemma} for the proof.
\end{proof}
\begin{corollary}(\emph{Sim-Score} as GCMI)
\label{thm: knn}    
For $B = \Vtar$ and similarity matrix $\Wmat$ defined as per definition~\ref{def: kNN}, any optimization done over \Cref{eq: knn} reduces to optimizing over GCMI \Cref{eq:GCMI}.
\end{corollary}

Similarly, we can reduce the commonly used CLIP-score retrieval strategy to an instantiation of GCMI.  

\begin{corollary}(\emph{CLIP-Score} as GCMI)
\label{thm: knn_clip}    
Let the similarity matrix $\Wmat$ be such that we use \textbf{text-embeddings} of $\Dtar$ (obtained using templates \cite{radford2021learning}) and image representations of $\Daux$ for generating the similarity between $\Vtar$ and $\Vaux$ (i.e., $\Wmat[\Vtar, \Vaux]$). Then, maximizing \emph{CLIP-score} reduces to maximizing over GCMI \Cref{eq:GCMI} with $B=\Vtar$.
\end{corollary}

\begin{remark}
    All the methods above only care about $\Wmat[\Vtar, \Vaux]$, a submatrix of $\Wmat$ sliced using $\Vtar$ and $\Vaux$, making the remaining entries of the matrix non-important.
\end{remark}

The above lemma and corollaries motivates studying CMI in its full generality as a natural framework for retrieval.

\section{COmBinatorial Retrieval Augmentation}
\label{sec: methods}

COBRA is a retrieval procedure that uses a submodular objective which is a combination of: (1) a Facility Location function based combinatorial Mutual Information function (FLMI) (2) a submodular function that encourages the retrieved set to be class balanced and (3) an optional function to score the quality of each sample. We describe each of them in detail below.



\paragraph{FLMI} The facility location (FL) function is a well-known submodular function and is effective at modeling diversity. In~\Cref{eq: fl}, every \emph{client} $i \in V$ must be represented by a \emph{facility} in $A \subseteq V$, which is chosen to be the element $j \in A$ closest to $i$ based on the similarity matrix $\Wmat$. In contrast, FLMI (Def.~\ref{def: FLMI}) is the CMI function defined on FL and captures both diversity and relevance. Intuitively, samples $i \in V$ which are not relevant to $\Dtar$ ($ \max_{j \in \Vtar} w_{i, j} \approx 0$) will not contribute to increasing $I_{FL}(A; V, \Wmat)$, since $\min{(\max_{j \in A} w_{i, j}, \max_{j \in \Vtar} w_{i, j})} = \max_{j \in \Vtar} w_{i, j}$. On the other hand, samples that are relevant to $\Dtar$ are likely to be such that $\min{(\max_{j \in A} w_{i, j}, \max_{j \in \Vtar} w_{i, j})} = \max_{j \in A} w_{i, j}$, which is maximized by selecting diverse samples (similar to~\Cref{eq: fl}). Thus, maximizing FLMI selects samples from $\Daux$ that are semantically similar to the target $\Dtar$ while being diverse within $\Daux$. This property is useful since it ameliorates the redundancy issue with sets that are selected by GCMI functions (\Cref{eq:GCMI}), and is also demonstrated in \Cref{fig:toy}. Since FLMI is submodular, we can still use a greedy algorithm to maximize it with a constant factor approximation guarantee~\citep {kothawade2022prism,iyer-cmi-alt-2021}.

\begin{definition}[\textit{FLMI}]
\label{def: FLMI}
Given a matrix $\Wmat = [w]_{i, j} \in \R_{+}\cup \{0\} $ modeling the similarity between elements of $\Dtar$ and $\Daux$. For any arbitrary subset $A \subseteq V$, a facility location function $f$ is defined as\looseness-1
\begin{equation}
\label{eq: fl}
f(A;V, \Wmat) = \sum_{i \in V} \max_{\substack{j \in A}} w_{i, j}
\end{equation}
Moreover, $f$ is submodular in $A$. For any $A \subseteq \Vaux$, FLMI is defined as the CMI of the Facility Location function --  
\begin{equation}
\label{eq: flmi}
    I_{FL}(A; V, \Wmat) \triangleq \sum_{i \in V} \min{(\max_{j \in A} w_{i, j}, \max_{j \in \Vtar} w_{i, j})}, 
\end{equation}
where $V = \Vtar \cup \Vaux$.  Moreover, $I_{FL}(A; \Vtar)$ is also submodular in $A$.

\end{definition}



\paragraph{Soft Class Balancing Constraint} Since we primarily consider classification tasks, it is also desirable for the set of retrieved samples from $\Daux$ to be reasonably class-balanced. While submodular functions may be maximized under hard class balance constraints, forcing the FLMI function to select a perfectly class-balanced set of samples can decrease performance since different classes in $\Daux$ will contain varying amounts of information and different levels of noise.  To this end, we encourage the retrieved sets to be class-balanced \emph{softly} using another submodular function, as shown in Lemma \ref{lemma: class_balancing}.



\begin{lemma}[\textit{Soft Class Balancing}]
\label{lemma: class_balancing}
Let $h : V \to [C]$ map any image (indexed using V) to the corresponding (pseudo)label among C classes. For any subset $A \subseteq V$, define the count for \textit{u}-th class in set A as $m_u(A) \triangleq \sum_{a \in A} \I[h(a) = u]$ and normalized count as $\hat{p}_u(A) \triangleq m_u(A)/|A|$. Further denote $\hat{\p}(A) = (\hat{p}_1(A), \hat{p}_2(A), \ldots, \hat{p}_C(A))$ the empirical probability based on normalized counts. For any given probability distribution $\p$ defined over C classes, and $k \in \Natural$, we have the following -
\vspace{-.25in}
\begin{equation}
    \argmin_{\substack{A \subseteq V \\ |A| = k}}\mathbb{D}_{\text{KL}} \left( \p \mid \mid \hat{\p}(A)\right) = \argmax_{\substack{A \subseteq V \\ |A| = k}} \sum_{u = 1}^C p_u \log{(m_u(A))}
\end{equation}
In fact for $\p = \left(\frac{1}{C}, \ldots, \frac{1}{C}\right)$, that is, uniform distribution over each class label, maximizing $\sum_{u = 1}^C \frac{\log{(m_u(A))}}{C}$ is equivalent to finding an $A \subseteq V$, $|A|=k$ that is class balanced. 
\end{lemma}
\begin{proof}
    Please refer to \Cref{sec: appen background and lemma} for the proof.
\end{proof}

\paragraph{Quality Score (Optional)} Auxiliary datasets such as LAION-2B contain many low-quality and noisy images, that could hamper performance if retrieved. To avoid this, a simple quality function ($q$) such as CLIP-score or Sim-Score could be employed alongside the existing optimization goals to further enhance the quality and relevance of the resulting summary. This leads us toward the final objective:

\begin{customboxeddefinition}[\textbf{COBRA Objective}] Using the same notations from definition \ref{def: FLMI} and lemma \ref{lemma: class_balancing}, and given any $\lambda \geq 0$ (soft-balancing weightage), $\mu \geq 0$ (an optional weightage to control the quality of selected points), $q(A) = \sum_{a \in A} q(a)$ (an optional quality function for selected points) COBRA optimizes for retrieving semantically similar samples to $\Dtar$ from $\Daux$ while maintaining a soft class balance, by solving the following optimization problem: 
\vspace{-.1in}
\begin{equation}
\begin{split}
             \argmax_{\substack{A \subseteq V \\ |A| = k}} \, & \mu \cdot q(A)  + (1-\mu) \cdot \\
           & \left\{\sum_{i \in V} \min{\left(\max_{j \in A} w_{i, j}, \max_{j \in \Vtar} w_{i, j}\right)} \right. \\
           & \left. + \lambda \sum_{u = 1}^C \frac{\log{(1 + m_u(A))}}{C} \right\}
\end{split}
\end{equation}
\end{customboxeddefinition}
\section{Experiments}
\label{sec: experiments}

In all our experiments, we employ the following procedure:
\begin{enumerate}[leftmargin=*,labelindent=0em,partopsep=-2pt,topsep=1pt,itemsep=2pt]
    \item \textbf{Collect Target Dataset}: We collect a small target dataset by sampling a standard image classification dataset uniformly at random, retaining 1-16 images per class.
    \item \textbf{Retrieval from LAION-2B}: We collect a set of auxiliary samples by doing the following.
        \begin{itemize}
            \item \textbf{Initial text-based filtering}: We follow the prefiltering step proposed by~\cite{wallingford2023neural} and use string matching to discard images with captions that do not contain the name of any class name in the target dataset. This stage circumvents the need to compute features for the full auxiliary pool while filtering out images that are unlikely to contain any semantically relevant information. We then use the class name contained in the caption as the label for the image.
            \item \textbf{Retrieve}: We use a retrieval strategy to select a fixed budget of samples from the filtered pool that are relevant to the target dataset. We retrieve approximately 16 images per class.
        \end{itemize}
    \item \textbf{Apply Few-Shot Adaptation Techniques}: We start with a CLIP model with a ViT-B/16 backbone (unless otherwise indicated), and use a few-shot adaptation strategy such as~\citep{tipadapter,clipadapter} to train on the target and retrieved samples. 
\end{enumerate}

\subsubsection*{Baseline Methods}
We briefly describe the alternative retrieval strategies that are compared against COBRA. All strategies that require similarity computations use representations extracted with a CLIP ViT-B/16 unless otherwise specified. 
\vspace{-.1in}
\paragraph{Sim-Score:} For a given auxiliary sample, we compute the cosine similarity to each target sample in the same class. The sum of the similarities is used as a score, and the samples with the highest scores are retrieved. This approach is the most ubiquitous retrieval strategy and has been employed in the context of VLM's in~\cite{wallingford2023neural, zancato2023train, geng2024unmetpromisesynthetictraining}. \vspace{-.15in}
\paragraph{CLIP-score:} Using a prompt template from~\citep{radford2021learning}, we craft a text prompt for each class in the target dataset. We then measure the cosine similarity between the CLIP representation of the class prompt and the representation for each auxiliary image. Auxiliary samples with the highest scores per class are retrieved~\cite {wallingford2023neural, geng2024unmetpromisesynthetictraining}.
\vspace{-.15in}
\paragraph{Random:} We select auxiliary samples uniformly at random from each class. This is a naive strategy that is prone to choosing noisy and irrelevant samples.
\vspace{-.15in}
\paragraph{SDXL-Aug} Leveraging text-to-image models to generate additional training data has become a technique of recent interest in vision~\citep{bansal2023leaving, cafo, geng2024unmetpromisesynthetictraining}. We therefore craft a prompt for each class in the target dataset, feed it to Stable Diffusion XL~\citep{podell2023sdxl}, and create a large pool of 105 generated images per class. We then use the CLIP-score to select the highest quality samples (as done in~\cite{cafo}) and add them to the target dataset for training. This was used as a retrieval strategy in~\cite{cafo}, albeit with a far weaker diffusion model.
\vspace{-.15in}
\paragraph{No Retrieve:} We also consider training only on the target task, without retrieving any additional samples from the auxiliary pool. This serves as a lower bound.

\subsubsection*{Other Implementation Details}
COBRA requires computing a pairwise similarity matrix ($\Wmat$) which does not need any constraint other than $w_{i, j} \geq 0$. Therefore, we opt to use sparse matrices instead of computing the full matrix and instantiate them by computing the $k$ nearest neighbors for each image using FAISS~\citep{johnson2019billion}. It also helps to avoid saturation problems in the facility location function, as highlighted in~\cite{bhatt2024experimental}, and also significantly accelerates this process. This process takes approximately 10 minutes on a single 80G NVIDIA A100 GPU even at our experimental scale, and is only a one-time cost. Greedy selection typically takes under a minute using~\cite{bilmes2025-submarine}. Few-shot training of CLIP models is also done on a single 80G NVIDIA A100 GPU and takes approximately 1 hour for a full training. We defer the discussion of hyperparameters to Appendix \ref{appen sec: hyperparameters} and \ref{appen sec: sensitivity}.

\begin{figure}[htp]
    \centering
    \vspace{-.1in}
    \includegraphics[scale=0.4]{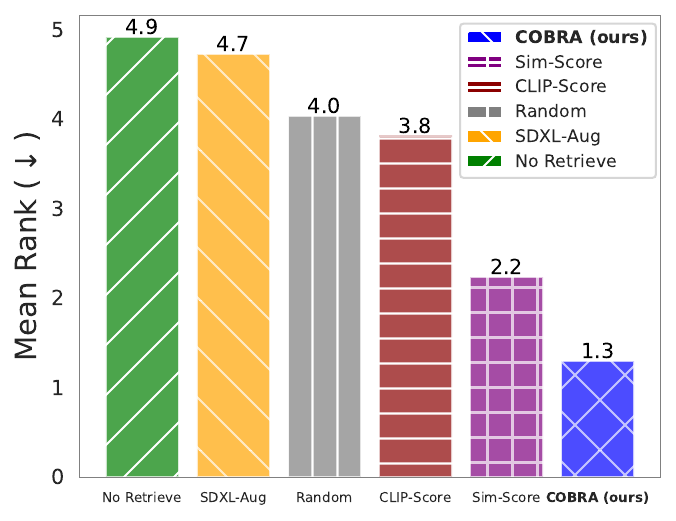}
    \caption{\textbf{Aggregated Ranking ($\downarrow$)} Average ranking of each retrieval strategy (lower is better) across  different levels of data scarcity, six datasets, and three random seeds. In over 90 experimental settings, COBRA generally outperforms any baseline we test.\looseness-1}
    \label{fig:ranking}  
    \vspace{-.1in}
\end{figure}

\begin{figure*}[htp!]
    \centering
    \begin{subfigure}{.3\textwidth}
        \includegraphics[width=\textwidth]{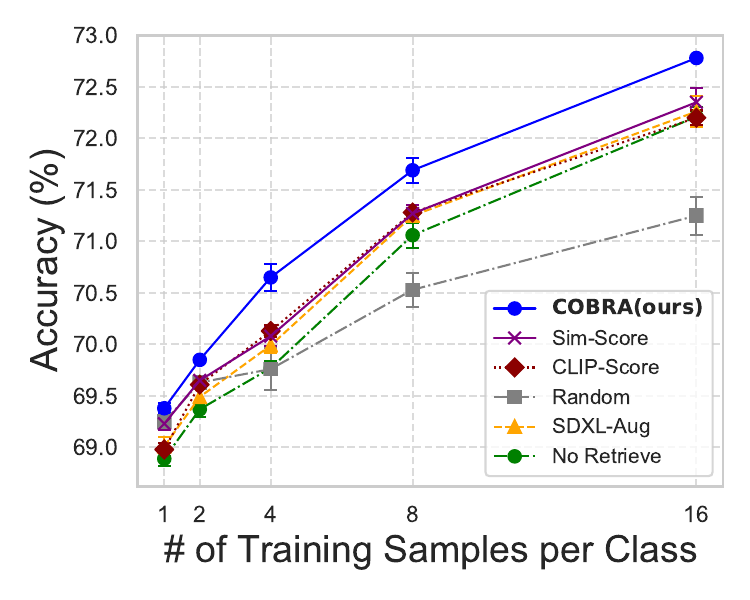}
        \caption{Accuracy on Imagenet}
    \end{subfigure}\quad
    \begin{subfigure}{.3\textwidth}
        \includegraphics[width=\textwidth]{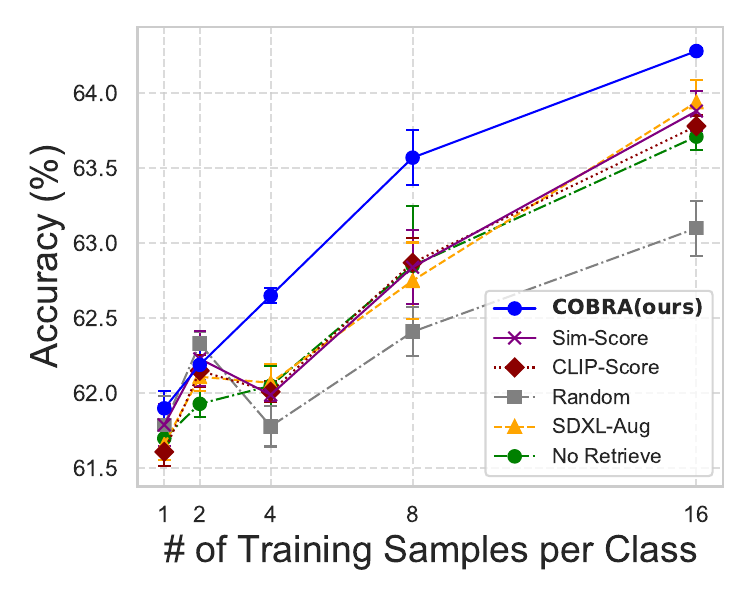}
        \caption{Accuracy on Imagenet-V2}
    \end{subfigure}\quad
    \begin{subfigure}{.3\textwidth}
        \includegraphics[width=\textwidth]{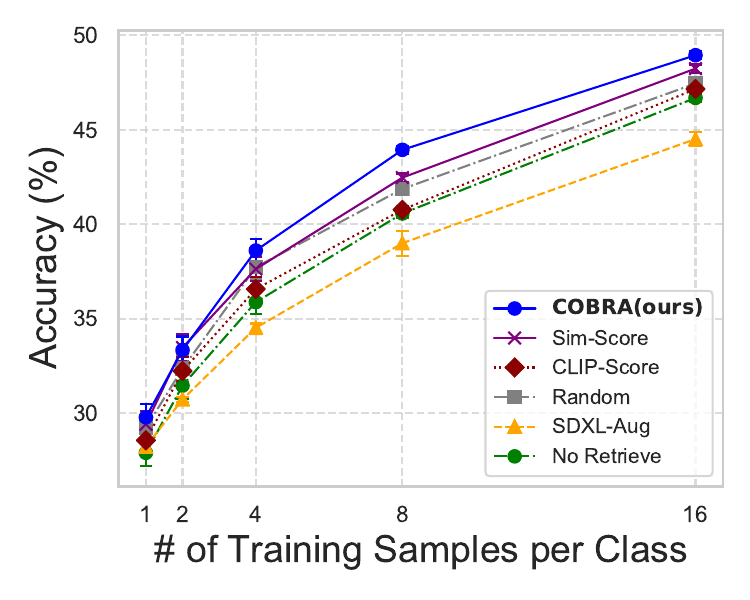}
        \caption{Accuracy on FGVC-Aircraft}
    \end{subfigure}
    
    
    \begin{subfigure}{.3\textwidth}
        \includegraphics[width=\textwidth]{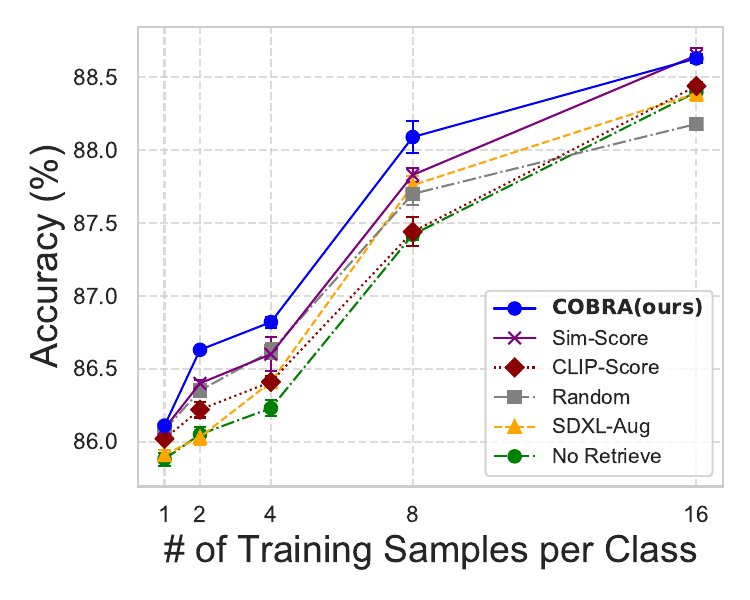}
        \caption{Accuracy on Food-101}
        \label{fig:food}
    \end{subfigure} \quad
    \begin{subfigure}{.3\textwidth}
        \includegraphics[width=\textwidth]{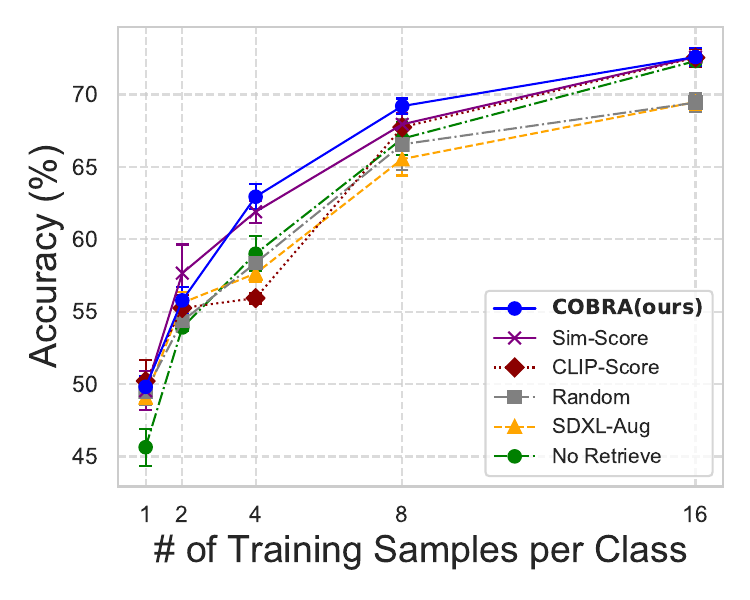}
        \caption{Accuracy on DTD}
    \end{subfigure} \quad
    \begin{subfigure}{.3\textwidth}
        \includegraphics[width=\textwidth]{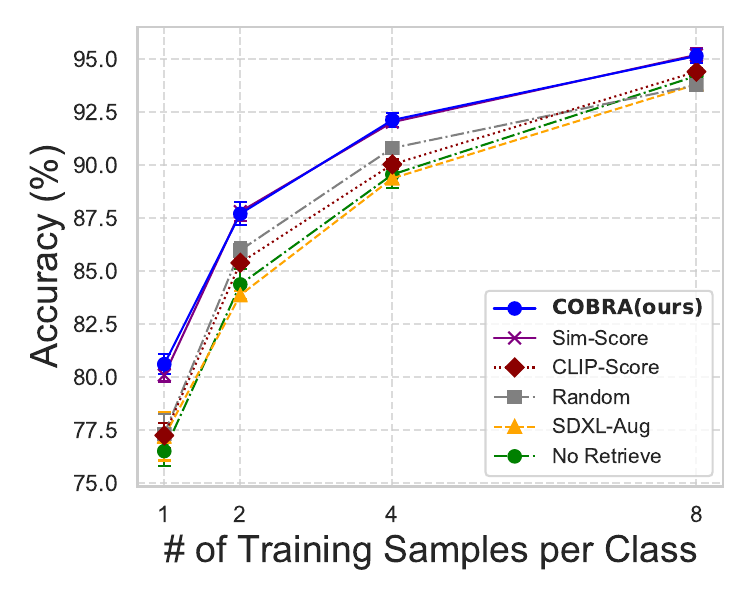}
        \caption{Accuracy on Flowers-102}
    \end{subfigure}
    
    \caption{\small \textbf{Results Across Target Datasets} We compare various retrieval strategies over LAION-2B as $\Daux$ using Tip-Adapter-F~\citep{tipadapter} to adapt CLIP to a small target dataset. We find that COBRA generally outperforms all other retrieval strategies, which is also reflected in mean-rank analysis in \Cref{fig:ranking}. Confidence intervals are based on standard errors computed over three trials.} 
    \label{fig:main_results}
    \vspace{-.15in}
\end{figure*}

\subsection{Main Results}
\paragraph{Performance across datasets} In~\Cref{fig:main_results}, we evaluate all retrieval strategies and consider several target tasks consisting of subsets from Imagenet-1k~\citep{deng2009imagenet}, Imagenet-V2~\citep{imagenetv2}, Flowers102~\citep{flowers102}, FGVC-Aircraft~\citep{fgvc_aircraft}, DTD~\citep{dtddataset}, and Food101~\citep{food101}. Note that, Imagenet-V2 is only a test set, so we report the performance of models trained on Imagenet-1k. Generally, we find that COBRA outperforms alternative retrieval strategies across tasks. On datasets with high intraclass diversity such as Imagenet, we find that retrieving auxiliary samples with COBRA provides the highest relative benefits. On Flowers102, we observe that there is no statistically significant difference between COBRA and \emph{Sim-Score}. We speculate that this is because both the test and train sets for Flowers102 are \textit{very homogeneous}, diminishing the utility of diversity for retrieval in this setting (see~\Cref{appen sec: qualitative flowers}). However, COBRA and \emph{Sim-Score} both outperform all other baseline strategies we consider. This is further reflected by our \textbf{\textit{average ranking analysis}} from Fig.\ref{fig:ranking}, in which we find the rank of each retrieval method for every shot setting and average it across every dataset. Since this is a ranking metric, a \emph{smaller} number is desirable, again showing COBRA is a strong candidate for this task. See~\Cref{app:nlp} for additional results on language tasks. 


\paragraph{Performance across few-shot algorithms}  Since the choice of retrieval algorithm is orthogonal to the choice of few-shot adaptation strategy, we evaluate the efficacy of COBRA in selecting additional training samples for a simple linear model~\citep{radford2021learning}, Clip-Adapter~\citep{clipadapter}, Tip-Adapter-F~\citep{tipadapter}, and CaFo~\citep{cafo} in Fig.~\ref{fig:fewshot_results}. These techniques all add a few learnable parameters on top of a pretrained CLIP image encoder and update them while keeping the CLIP encoders frozen. The results demonstrate a consistent trend in the relative performance of retrieval strategies across the few-shot adaptation techniques, with COBRA as the consistent winner. 

\paragraph{Impact of synthetic training samples} Interestingly, our results in Fig.~\ref{fig:main_results} and Fig.~\ref{fig:fewshot_results} also demonstrate that retrieving from an auxiliary pool using a proper retrieval strategy is far more effective for training than using synthetically generated images. Despite using the strongest available generative model, we find that SDXL-Aug never outperforms the best retrieval strategy for a given task (COBRA). We speculate that this is due to (1) outputs produced by even the most powerful generative models tend to contain highly unrealistic artifacts \citep{cao2024synartifact} and (2) diffusion models tend to generate sets of images with very low diversity \citep{zameshina2023diverse}. A more comprehensive study evaluating the effectiveness of using generated images as training data is left for future work. See~\Cref{appen sec: more qualitative} for qualitative examples.

\begin{figure*}[h!]
    \centering
    \begin{subfigure}{.23\textwidth}
        \includegraphics[width=\textwidth]{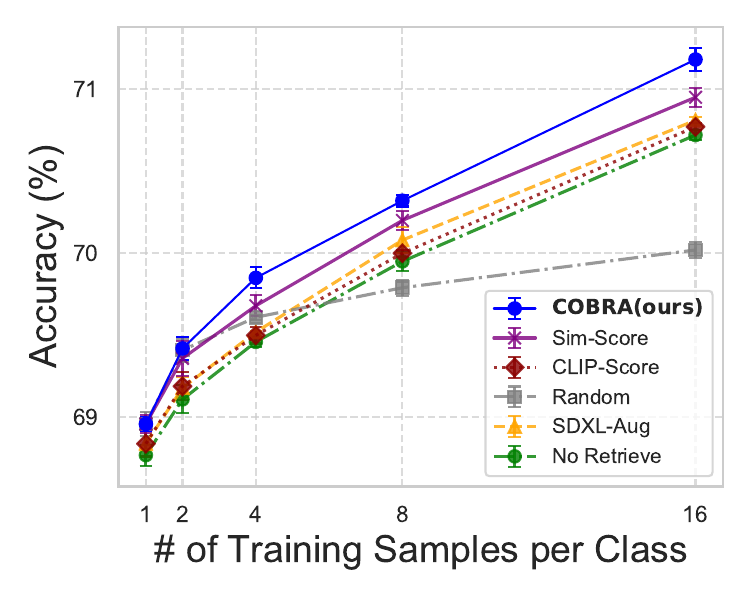}
        \caption{Acc. with Linear Probe}
    \end{subfigure}
    \begin{subfigure}{.23\textwidth}
        \includegraphics[width=\textwidth]{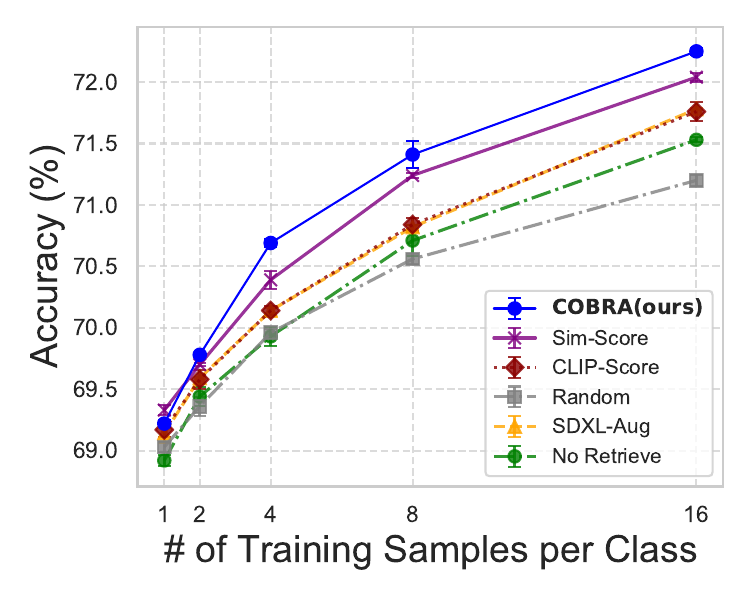}
        \caption{Acc. with Clip-Adapter}
    \end{subfigure}\quad
    \begin{subfigure}{.23\textwidth}
        \includegraphics[width=\textwidth]{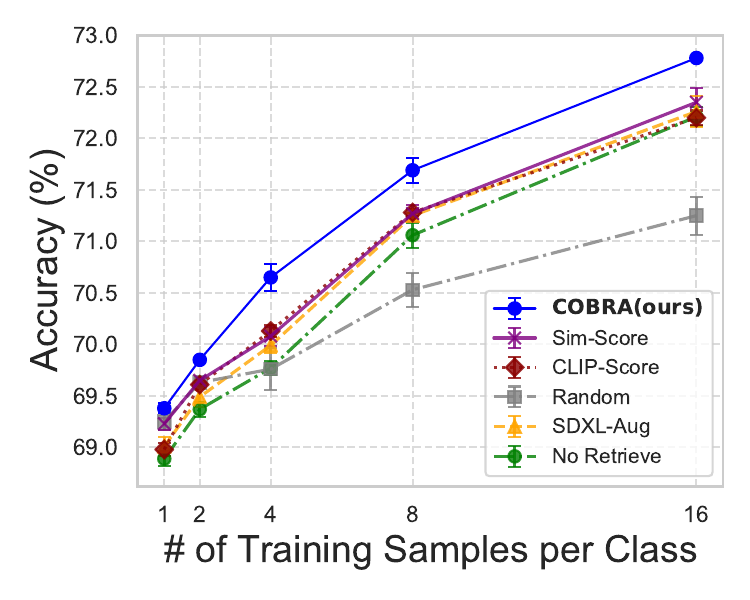}
        \caption{Acc. with Tip-Adapter-F}
    \end{subfigure}\quad
    \begin{subfigure}{.23\textwidth}
        \includegraphics[width=\textwidth]{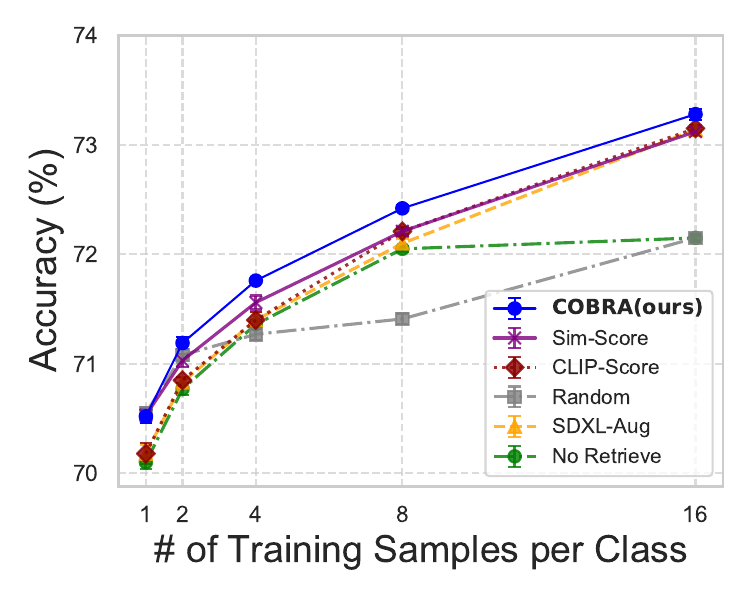}
        \caption{Acc. with CaFo}
    \end{subfigure}\quad

    \caption{\small \textbf{Imagenet Results with Different Few-Shot Techniques} After retrieving samples from LAION-2B for Imagenet, we train either a Linear Probe, Clip-Adapter~\citep{clipadapter}, Tip-Adapter-F~\citep{tipadapter}, and CaFo~\citep{cafo} on top of a pretrained CLIP model. We find that COBRA-based retrieval consistently outperforms alternative strategies.} 
    \label{fig:fewshot_results}
    \vspace{-.15in}
\end{figure*}

\subsection{Additional Baselines}
\label{sec:additional_baselines}
In this section, we test several additional baselines that are not commonly used for retrieval in the context of VLMs but can be adapted for this purpose. We use these baselines to retrieve additional data in the 16-shot Imagenet-1k setting to train Tip-Adapter-F with a ViT-B/16 backbone. Results are shown in~\Cref{tab:ablation_balancing}. 

\paragraph{Soft Class Balancing} We also consider the effect of imposing a soft class balance constraint as opposed to strictly enforcing it as done in past work~\citep{wallingford2023neural, zancato2023train, liu2023learning}. As mentioned in \Cref{sec: methods}, the COBRA objective uses soft class balancing by adding the submodular term, $\sum_{u=1}^C\frac{\log{(1 + m_u(A))}}{C}$. For a fair comparison, we also add this term to Sim-Score and CLIP-Score in \Cref{tab:ablation_balancing}. Note that this creates a new submodular objective that we maximize using the greedy algorithm. In this study, we observe that even the baseline algorithms benefit from the inclusion of a soft class balancing term, demonstrating that it is not necessary to impose a hard class balance constraint for retrieval. However, COBRA retains its lead demonstrating the importance of the FLMI objective.


\paragraph{Other CMI Functions} There are other instances of CMI functions, other than FLMI, that can model diversity and relevance. Here we study the CMI instantiated based on the log determinant (i.e., a determinantal point process, or DPP) function, which has also been used in previous works to induce diversity \citep{kulesza2012determinantal, chen2018fast, pmlr-v37-sharma15, ash2019deep, saran2023streaming, wang2024diversity}. Given a positive definite matrix $\Wmat$ and two sets A and B such that $A \cap B = \emptyset$,  submodular function $f(A) = \log{\det{\Wmat[A, A]}}$, $\log \det$-MI is defined as \cite{kothawade2022prism}.
\begin{equation}
\label{eq: logdet}
\begin{split}
    I_{f}(A;B) &= \log \left( \det \left(\Wmat[A, A] \right) \right) \\
    &- \log \left( \det \left(\Wmat[A, A] \right)  \right)\\ 
    &- \Wmat[A, B] \Wmat[B, B]^{-1} \Wmat[A, B]^T
\end{split}
\end{equation}

A positive definite constraint ($\Wmat \succ 0$) and its expensive computation severely prohibit its utility in large scale settings. To this end, using a special kernel (sum of outer products) and \emph{Woodbury matrix identity} we scale $\log \det$-MI to the largest sizes ever used, to the best of our knowledge (using~\cite{bilmes2025-submarine}). From \Cref{tab:ablation_balancing}, we observe that  $\log \det$-MI severely underperforms COBRA (and other baselines), which can be attributed to the relatively simple kernel, and the positive definite restriction on similarity matrix required by  $\log \det$ function, unlike COBRA. 

\paragraph{Maximum Marginal Relevance (MMR)} MMR is a classical method~\citep{carbonell1998} that can be used to select relevant and diverse samples, but has not been used in the context of VLMs. MMR introduces a trade-off parameter $\lambda_{MMR}$, which balances relevance and diversity (increasing $\lambda_{MMR}$ increases relevance). We report MMR's performance in \Cref{tab:ablation_balancing} for different $\lambda_{MMR}$ values and show that COBRA outperforms all of them. We discuss more details of MMR in \Cref{app:baseline}. 
\vspace{-.2in}

\paragraph{GPT-3 Class Descriptions} A method to enhance the diversity of the retrieved samples is to use an LLM to generate varied descriptions for a given class~\cite{cafo}. These diverse descriptions can then be used as prompts for evaluating with \emph{CLIP-Score}~\cite{udandarao2023susx} or to generate images, that is, \emph{SDXL-Aug}. We compare these two methods in \Cref{tab:ablation_balancing} and show that COBRA outperforms both baselines with prompt-based diversity. We discuss this in more detail in \Cref{app:baseline}. 


\begin{table}[htp]
    \centering
    \small
        \begin{tabular}{cc}
            \toprule[1.5pt]
            \textbf{Algorithm} & \textbf{Imagenet Acc. (\%)} \\
            \midrule
            Sim-Score & 72.35 \small $\pm$ 0.01  \\
            Sim-Score w/ Soft Bal.  & \underline{72.46 \small $\pm$ 0.03} \\
            CLIP-Score & 72.20 \small $\pm$ 0.09  \\
            CLIP-Score w/ Soft Bal. & 72.22 \small $\pm$ 0.04  \Bstrut \\
            \hline 
            $\log \det$-MI & 71.28 \small $\pm$ 0.07  \Tstrut \Bstrut \\
            \hline
            MMR ($\lambda_{MMR} = .25$) & 71.14 \small $\pm$ 0.04  \Tstrut \Bstrut \\
            MMR ($\lambda_{MMR} = .50$) & 72.41 \small $\pm$ 0.11  \Tstrut \Bstrut \\
            MMR ($\lambda_{MMR} = .75$) & 72.25 \small $\pm$ 0.03 \Tstrut \Bstrut \\
            \hline 
            CLIP-Score w/ GPT-3 & 72.27 \small $\pm$ 0.02  \Tstrut \Bstrut \\
            SDXL-Aug w/ GPT-3 & 72.33 \small $\pm$ 0.01  \Tstrut \Bstrut \\
            \hline
            COBRA (Ours) & \textbf{72.78 \small $\pm$ 0.13}  \Tstrut \\
            \bottomrule
        \end{tabular}
        \caption{\textbf{Additional Baselines:} Soft class-balancing boosts performance when used with other retrieval algorithms, though COBRA is still superior. We also observe that COBRA outperforms other retrieval methods that consider diversity. The standard error is computed over three trials.}
        \label{tab:ablation_balancing}
\end{table}

\subsection{Ablations}
\begin{table}[htp]
    \small
    \centering

    \begin{tabular}{ccc|c}
        \toprule[1.5pt]
        \textbf{FLMI} & \textbf{Class-Balancing} & \textbf{Quality Score} & \textbf{Acc. (\%)} \\
        \midrule
        \checkmark & $\times$ & $\times$ & 72.56 $\pm$ 0.03 \\
        \checkmark & $\times$ & \checkmark & \underline{72.60 $\pm$ 0.02} \\
        $\times$ & \checkmark & \checkmark & 72.46 $\pm$ 0.03 \\
        \checkmark & \checkmark & $\times$ & \textbf{72.78 $\pm$ 0.13} \\
        \checkmark & \checkmark & \checkmark & 72.52 $\pm$ 0.11 \\
        \bottomrule
    \end{tabular}
    \caption{\textbf{Component Analysis:} COBRA has three terms. In this set of experiments, we examine the importance of each component in the ablation study. See \Cref{appen sec: hyperparameters} for more analysis.}
    \label{tab:component_analysis}
\end{table}

\begin{table}[htbp]
    \centering
    \resizebox{\linewidth}{!}{
    \begin{tabular}{cccc}
        \toprule[1.5pt]
        \textbf{Algorithm} & \multicolumn{3}{c}{\textbf{Imagenet Acc. (\%)}} \\
        \midrule
         & \textbf{ViT-B/32} & \textbf{ViT-B/16} & \textbf{ViT-L/14} \Bstrut \\
        \hline 
        \emph{No-retrieve} & 67.26 \small $\pm$ 0.05 & 72.21 \small $\pm$ 0.01 & 77.98 \small $\pm$ 0.02 \Tstrut \\
        \emph{Sim-Score} & \underline{67.50 \small $\pm$ 0.04} & \underline{72.35 \small $\pm$ 0.01} & \underline{78.15 \small $\pm$ 0.07} \Bstrut \\
        \emph{CLIP-score} & 67.48 \small $\pm$ 0.04 & 72.20 \small $\pm$ 0.09 & 78.12 \small $\pm$ 0.04 \Bstrut \\
        \hline
        COBRA (Ours) & \textbf{67.76 \small $\pm$ 0.06} & \textbf{72.78 \small $\pm$ 0.13} & \textbf{78.48 \small $\pm$ 0.01} \Tstrut \\
        \bottomrule
    \end{tabular}}
    \caption{\textbf{Choice of Backbone:} COBRA provides improvement over baseline algorithms even when the CLIP backbone is changed for the retriever and the learner.}
    \label{tab:ablation_backbone}
\end{table}

\paragraph{Component Analysis}
We include a set of experiments on assessing the importance on each component, by running experiments on the 16-shot Imagenet case in \Cref{tab:component_analysis}. In this set of experiments, we find that COBRA outperforms all baselines in \Cref{tab:ablation_balancing} even without all of the components. We also provide more hyperparameter sensitivity analyses in \Cref{appen sec: hyperparameters}.
\vspace{-.15in}

\paragraph{Backbone Ablation}
 We only considered ViT-B/16 in our main experiments, but we include additional results with ViT-B/32 and ViT-L/14 backbones.  Our results in \Cref{tab:ablation_backbone} showcase that COBRA is robust to the choice of backbone.

\section{Related Work}
\label{sec:related_works}
\paragraph{Retrieval-Augmented Models} Retrieval augmentation has demonstrated impressive performance across a wide range of natural language tasks, leveraging external non-parametric knowledge to guide the output of large language models~\citep{lewis2020retrieval, izacard2022atlas, shi2023replug, ram2023context}. Notable applications include question answering~\citep{izacard2021leveraging, kandpal2023large, wang2023augmenting}, dialog modeling~\citep{shuster2021retrieval, peng2023check, cheng2024lift}, and code generation~\citep{zhou2022docprompting, zhang2023repocoder}. In the few-shot adaptation setting studied in COBRA, we leverage the retrieved data to acquire additional training data. Furthermore, these prior works focus on similarity-based retrieval approaches, thus failing to address the potential issue of redundancy among the retrieved entities. More recently, retrieval augmentation has been successfully applied to vision tasks~\citep{zancato2023train, wallingford2023neural, liu2023learning, geng2024unmetpromisesynthetictraining}. Retrieval approaches for VLMs frequently assume that the auxiliary dataset consists of captioned images, and includes some form of text-filtering where images with irrelevant captions are discarded~\citep{wallingford2023neural, udandarao2023susx, geng2024unmetpromisesynthetictraining}. Other work also leverages similarity-based approaches, where auxiliary samples with high similarity to samples in the training set are retrieved~\citep{wallingford2023neural, liu2023learning, zancato2023train, geng2024unmetpromisesynthetictraining}. Unlike past work in retrieval for VLM's, we are the first, to the best of our knowledge, to consider retrieval strategies beyond text-filtering and similarity-based approaches.  
\vspace{-.15in}
\paragraph{Diversity in Information Retrieval} While retrieval strategies in the context of VLMs overlook the concept of diversity, the broader field of information retrieval has explored methods for balancing diversity with relevance. Classical work in information retrieval demonstrates that diversity can improve user experience, and propose techniques that rank documents such that the top elements are not redundant~\citep{carbonell1998, Deselaers2009joint, tollari2009comparative}. Some of these techniques can be adapted to retrieval-augmented few-shot adaptation as well as we show in \Cref{app:baseline}.  Recent information retrieval work also strives to incorporate diversity by learning separate encoders that are more amenable to diverse retrieval~\citep{colt2023, vmig2023}, though learning an entirely different encoder introduces significant computational overhead. To the best of our knowledge, no approach has explicitly used CMI functions for the task of retrieval. 
\vspace{-.1in}

\paragraph{Other Related Work} We include more discussion on related works in~\Cref{app:more_related_work}.
\section{Conclusion}
\label{sec:conclusion}
In this work, we present COBRA, a novel application of Combinatorial Mutual Information (CMI) measures for retrieval augmented few-shot adaptation. We not only showed that the most popular existing methods such as similarity-based retrieval algorithms are part of the CMI framework but also demonstrated that alternative instances of this family can retrieve diverse and relevant samples. In our experiments, we demonstrate that COBRA outperforms existing retrieval strategies across several visual tasks. By demonstrating the importance of diversity in retrieval, we hope that future work in this field will move beyond simple nearest-neighbor based retrieval strategies. Furthermore, we plan to explore diversity-aware retrieval schemes for generative tasks (in standard RAG workflows) in future work. 
\section{Acknowledgements}
This work is supported by NSF Grant Nos.
IIS-2106937 and IIS-2148367. The authors also thank Matt Wallingford, Vivek Ramanujan, and Scott Geng for their help with dataset preparation and discussions.
{
    \small
    \bibliographystyle{ieeenat_fullname}
    \bibliography{main}

\begin{thebibliography}{111}
\providecommand{\natexlab}[1]{#1}
\providecommand{\url}[1]{\texttt{#1}}
\expandafter\ifx\csname urlstyle\endcsname\relax
  \providecommand{\doi}[1]{doi: #1}\else
  \providecommand{\doi}{doi: \begingroup \urlstyle{rm}\Url}\fi

\bibitem[Ash et~al.(2021)Ash, Goel, Krishnamurthy, and Kakade]{ash2021gone}
Jordan Ash, Surbhi Goel, Akshay Krishnamurthy, and Sham Kakade.
\newblock Gone fishing: Neural active learning with fisher embeddings.
\newblock \emph{Advances in Neural Information Processing Systems}, 34:\penalty0 8927--8939, 2021.

\bibitem[Ash et~al.(2019)Ash, Zhang, Krishnamurthy, Langford, and Agarwal]{ash2019deep}
Jordan~T Ash, Chicheng Zhang, Akshay Krishnamurthy, John Langford, and Alekh Agarwal.
\newblock Deep batch active learning by diverse, uncertain gradient lower bounds.
\newblock \emph{arXiv preprint arXiv:1906.03671}, 2019.

\bibitem[Atlas et~al.(1989)Atlas, Cohn, and Ladner]{atlas1989training}
Les Atlas, David Cohn, and Richard Ladner.
\newblock Training connectionist networks with queries and selective sampling.
\newblock \emph{Advances in neural information processing systems}, 2, 1989.

\bibitem[Bansal and Grover(2023)]{bansal2023leaving}
Hritik Bansal and Aditya Grover.
\newblock Leaving reality to imagination: Robust classification via generated datasets.
\newblock \emph{arXiv preprint arXiv:2302.02503}, 2023.

\bibitem[Beluch et~al.(2018)Beluch, Genewein, N{\"u}rnberger, and K{\"o}hler]{beluch2018power}
William~H Beluch, Tim Genewein, Andreas N{\"u}rnberger, and Jan~M K{\"o}hler.
\newblock The power of ensembles for active learning in image classification.
\newblock In \emph{Proceedings of the IEEE conference on computer vision and pattern recognition}, pages 9368--9377, 2018.

\bibitem[Bentivogli et~al.(2009)Bentivogli, Magnini, Dagan, Dang, and Giampiccolo]{rte}
Luisa Bentivogli, Bernardo Magnini, Ido Dagan, Hoa~Trang Dang, and Danilo Giampiccolo.
\newblock The fifth {PASCAL} recognizing textual entailment challenge.
\newblock In \emph{Proceedings of the Second Text Analysis Conference, {TAC} 2009, Gaithersburg, Maryland, USA, November 16-17, 2009}. {NIST}, 2009.

\bibitem[Bhatt et~al.(2024{\natexlab{a}})Bhatt, Chen, Das, Zhang, Truong, Mussmann, Zhu, Bilmes, Du, Jamieson, Ash, and Nowak]{bhatt2024experimental}
Gantavya Bhatt, Yifang Chen, Arnav~Mohanty Das, Jifan Zhang, Sang~T. Truong, Stephen Mussmann, Yinglun Zhu, Jeff~A. Bilmes, Simon~S. Du, Kevin~G. Jamieson, Jordan~T. Ash, and Robert~D. Nowak.
\newblock An experimental design framework for label-efficient supervised finetuning of large language models.
\newblock In \emph{ACL (Findings)}, pages 6549--6560, 2024{\natexlab{a}}.

\bibitem[Bhatt et~al.(2024{\natexlab{b}})Bhatt, Das, and Bilmes]{bhatt2024deep}
Gantavya Bhatt, Arnav~Mohanty Das, and Jeff Bilmes.
\newblock Deep submodular peripteral networks.
\newblock In \emph{The Thirty-eighth Annual Conference on Neural Information Processing Systems}, 2024{\natexlab{b}}.

\bibitem[Bilmes(2022)]{bilmes2022submodularity}
Jeff Bilmes.
\newblock Submodularity in machine learning and artificial intelligence.
\newblock \emph{arXiv preprint arXiv:2202.00132}, 2022.

\bibitem[Bilmes(2025)]{bilmes2025-submarine}
Jeff Bilmes.
\newblock Submarine: {SUBM}odularity for {AR}tificial {IN}telligenc{E} and machine learning.
\newblock Online Software System, 2025.
\newblock \url{https://submarine.page}.

\bibitem[Blattmann et~al.(2022)Blattmann, Rombach, Oktay, M\"{u}ller, and Ommer]{diffusion_rag}
Andreas Blattmann, Robin Rombach, Kaan Oktay, Jonas M\"{u}ller, and Bj\"{o}rn Ommer.
\newblock Retrieval-augmented diffusion models.
\newblock In \emph{Advances in Neural Information Processing Systems}, pages 15309--15324. Curran Associates, Inc., 2022.

\bibitem[Bossard et~al.(2014)Bossard, Guillaumin, and Van~Gool]{food101}
Lukas Bossard, Matthieu Guillaumin, and Luc Van~Gool.
\newblock Food-101 -- mining discriminative components with random forests.
\newblock In \emph{Computer Vision -- ECCV 2014}, pages 446--461, Cham, 2014. Springer International Publishing.

\bibitem[Brown et~al.(2020)Brown, Mann, Ryder, Subbiah, Kaplan, Dhariwal, Neelakantan, Shyam, Sastry, Askell, Agarwal, Herbert-Voss, Krueger, Henighan, Child, Ramesh, Ziegler, Wu, Winter, Hesse, Chen, Sigler, Litwin, Gray, Chess, Clark, Berner, McCandlish, Radford, Sutskever, and Amodei]{brown2020languagemodelsfewshotlearners}
Tom~B. Brown, Benjamin Mann, Nick Ryder, Melanie Subbiah, Jared Kaplan, Prafulla Dhariwal, Arvind Neelakantan, Pranav Shyam, Girish Sastry, Amanda Askell, Sandhini Agarwal, Ariel Herbert-Voss, Gretchen Krueger, Tom Henighan, Rewon Child, Aditya Ramesh, Daniel~M. Ziegler, Jeffrey Wu, Clemens Winter, Christopher Hesse, Mark Chen, Eric Sigler, Mateusz Litwin, Scott Gray, Benjamin Chess, Jack Clark, Christopher Berner, Sam McCandlish, Alec Radford, Ilya Sutskever, and Dario Amodei.
\newblock Language models are few-shot learners, 2020.

\bibitem[Cao et~al.(2024)Cao, Yuan, Liu, Li, Sun, Liu, and Zhao]{cao2024synartifact}
Bin Cao, Jianhao Yuan, Yexin Liu, Jian Li, Shuyang Sun, Jing Liu, and Bo Zhao.
\newblock Synartifact: Classifying and alleviating artifacts in synthetic images via vision-language model.
\newblock \emph{arXiv preprint arXiv:2402.18068}, 2024.

\bibitem[Carbonell and Goldstein(1998)]{carbonell1998}
Jaime Carbonell and Jade Goldstein.
\newblock The use of mmr, diversity-based reranking for reordering documents and producing summaries.
\newblock In \emph{Proceedings of the 21st Annual International ACM SIGIR Conference on Research and Development in Information Retrieval}, page 335–336, New York, NY, USA, 1998. Association for Computing Machinery.

\bibitem[Chen et~al.(2018)Chen, Zhang, and Zhou]{chen2018fast}
Laming Chen, Guoxin Zhang, and Eric Zhou.
\newblock Fast greedy map inference for determinantal point process to improve recommendation diversity.
\newblock \emph{Advances in Neural Information Processing Systems}, 31, 2018.

\bibitem[Cheng et~al.(2024)Cheng, Luo, Chen, Liu, Zhao, and Yan]{cheng2024lift}
Xin Cheng, Di Luo, Xiuying Chen, Lemao Liu, Dongyan Zhao, and Rui Yan.
\newblock Lift yourself up: Retrieval-augmented text generation with self-memory.
\newblock \emph{Advances in Neural Information Processing Systems}, 36, 2024.

\bibitem[Cimpoi et~al.(2014)Cimpoi, Maji, Kokkinos, Mohamed, , and Vedaldi]{dtddataset}
M. Cimpoi, S. Maji, I. Kokkinos, S. Mohamed, , and A. Vedaldi.
\newblock Describing textures in the wild.
\newblock In \emph{Proceedings of the {IEEE} Conf. on Computer Vision and Pattern Recognition ({CVPR})}, 2014.

\bibitem[Citovsky et~al.(2021)Citovsky, DeSalvo, Gentile, Karydas, Rajagopalan, Rostamizadeh, and Kumar]{citovsky2021batch}
Gui Citovsky, Giulia DeSalvo, Claudio Gentile, Lazaros Karydas, Anand Rajagopalan, Afshin Rostamizadeh, and Sanjiv Kumar.
\newblock Batch active learning at scale.
\newblock \emph{Advances in Neural Information Processing Systems}, 34:\penalty0 11933--11944, 2021.

\bibitem[Coleman et~al.(2021)Coleman, Chou, Katz-Samuels, Culatana, Bailis, Berg, Nowak, Sumbaly, Zaharia, and Yalniz]{coleman2021similarity}
Cody Coleman, Edward Chou, Julian Katz-Samuels, Sean Culatana, Peter Bailis, Alexander~C. Berg, Robert Nowak, Roshan Sumbaly, Matei Zaharia, and I.~Zeki Yalniz.
\newblock Similarity search for efficient active learning and search of rare concepts, 2021.

\bibitem[Cortes and Vapnik(1995)]{cortes1995support}
Corinna Cortes and Vladimir Vapnik.
\newblock Support-vector networks.
\newblock \emph{Machine learning}, 20:\penalty0 273--297, 1995.

\bibitem[Cover and Thomas(2006)]{Cover2006}
Thomas~M. Cover and Joy~A. Thomas.
\newblock \emph{Elements of Information Theory 2nd Edition (Wiley Series in Telecommunications and Signal Processing)}.
\newblock Wiley-Interscience, 2006.

\bibitem[Das et~al.(2023)Das, Bhatt, Bhalerao, Gao, Yang, and Bilmes]{das2023accelerating}
Arnav~Mohanty Das, Gantavya Bhatt, Megh~Manoj Bhalerao, Vianne~R. Gao, Rui Yang, and Jeff Bilmes.
\newblock Accelerating batch active learning using continual learning techniques.
\newblock \emph{Transactions on Machine Learning Research}, 2023.

\bibitem[Deng et~al.(2009)Deng, Dong, Socher, Li, Li, and Fei-Fei]{deng2009imagenet}
Jia Deng, Wei Dong, Richard Socher, Li-Jia Li, Kai Li, and Li Fei-Fei.
\newblock Imagenet: A large-scale hierarchical image database.
\newblock In \emph{2009 IEEE conference on computer vision and pattern recognition}, pages 248--255. Ieee, 2009.

\bibitem[Deselaers et~al.(2009)Deselaers, Gass, Dreuw, and Ney]{Deselaers2009joint}
Thomas Deselaers, Tobias Gass, Philippe Dreuw, and Hermann Ney.
\newblock Jointly optimising relevance and diversity in image retrieval.
\newblock In \emph{Proceedings of the ACM International Conference on Image and Video Retrieval}, New York, NY, USA, 2009. Association for Computing Machinery.

\bibitem[Dolan and Brockett(2005)]{mrpc}
William~B. Dolan and Chris Brockett.
\newblock Automatically constructing a corpus of sentential paraphrases.
\newblock In \emph{Proceedings of the Third International Workshop on Paraphrasing ({IWP}2005)}, 2005.

\bibitem[Ducoffe and Precioso(2018)]{ducoffe2018adversarial}
Melanie Ducoffe and Frederic Precioso.
\newblock Adversarial active learning for deep networks: a margin based approach.
\newblock \emph{arXiv preprint arXiv:1802.09841}, 2018.

\bibitem[et~al.(2023)]{APE2023}
Xiangyang~Zhu et al.
\newblock Not all features matter: Enhancing few-shot clip with adaptive prior refinement.
\newblock In \emph{ICCV}, 2023.

\bibitem[et~al.(2024)]{AMU2024}
Yuwei~Tang et al.
\newblock Amu-tuning: Effective logit bias for clip-based few-shot learning.
\newblock In \emph{CVPR}, 2024.

\bibitem[Finn et~al.(2017)Finn, Abbeel, and Levine]{finn17maml}
Chelsea Finn, Pieter Abbeel, and Sergey Levine.
\newblock Model-agnostic meta-learning for fast adaptation of deep networks.
\newblock In \emph{Proceedings of the 34th International Conference on Machine Learning}, pages 1126--1135. PMLR, 2017.

\bibitem[Friedman and Dieng(2022)]{friedman2022vendi}
Dan Friedman and Adji Dieng.
\newblock The vendi score: A diversity evaluation metric for machine learning.
\newblock \emph{TMLR}, 2022.

\bibitem[Fujishige(2005)]{fujishige2005submodular}
Satoru Fujishige.
\newblock \emph{Submodular functions and optimization}.
\newblock Elsevier, 2005.

\bibitem[Gadre et~al.(2023)Gadre, Ilharco, Fang, Hayase, Smyrnis, Nguyen, Marten, Wortsman, Ghosh, Zhang, Orgad, Entezari, Daras, Pratt, Ramanujan, Bitton, Marathe, Mussmann, Vencu, Cherti, Krishna, Koh, Saukh, Ratner, Song, Hajishirzi, Farhadi, Beaumont, Oh, Dimakis, Jitsev, Carmon, Shankar, and Schmidt]{gadre2023datacomp}
Samir~Yitzhak Gadre, Gabriel Ilharco, Alex Fang, Jonathan Hayase, Georgios Smyrnis, Thao Nguyen, Ryan Marten, Mitchell Wortsman, Dhruba Ghosh, Jieyu Zhang, Eyal Orgad, Rahim Entezari, Giannis Daras, Sarah Pratt, Vivek Ramanujan, Yonatan Bitton, Kalyani Marathe, Stephen Mussmann, Richard Vencu, Mehdi Cherti, Ranjay Krishna, Pang~Wei Koh, Olga Saukh, Alexander Ratner, Shuran Song, Hannaneh Hajishirzi, Ali Farhadi, Romain Beaumont, Sewoong Oh, Alex Dimakis, Jenia Jitsev, Yair Carmon, Vaishaal Shankar, and Ludwig Schmidt.
\newblock Datacomp: In search of the next generation of multimodal datasets, 2023.

\bibitem[Gal et~al.(2017)Gal, Islam, and Ghahramani]{gal2017deep}
Yarin Gal, Riashat Islam, and Zoubin Ghahramani.
\newblock Deep bayesian active learning with image data.
\newblock In \emph{International Conference on Machine Learning}, pages 1183--1192. PMLR, 2017.

\bibitem[Gao et~al.(2021)Gao, Geng, Zhang, Ma, Fang, Zhang, Li, and Qiao]{clipadapter}
Peng Gao, Shijie Geng, Renrui Zhang, Teli Ma, Rongyao Fang, Yongfeng Zhang, Hongsheng Li, and Yu Qiao.
\newblock Clip-adapter: Better vision-language models with feature adapters, 2021.

\bibitem[Geifman and El-Yaniv(2017)]{geifman2017deep}
Yonatan Geifman and Ran El-Yaniv.
\newblock Deep active learning over the long tail.
\newblock \emph{arXiv preprint arXiv:1711.00941}, 2017.

\bibitem[Geng et~al.(2024)Geng, Hsieh, Ramanujan, Wallingford, Li, Koh, and Krishna]{geng2024unmetpromisesynthetictraining}
Scott Geng, Cheng-Yu Hsieh, Vivek Ramanujan, Matthew Wallingford, Chun-Liang Li, Pang~Wei Koh, and Ranjay Krishna.
\newblock The unmet promise of synthetic training images: Using retrieved real images performs better, 2024.

\bibitem[Iyer et~al.(2021)Iyer, Khargonkar, Bilmes, and Asnani]{iyer-cmi-alt-2021}
Rishabh Iyer, Ninad~A Khargonkar, Jeffrey~A. Bilmes, and Himanshu Asnani.
\newblock Submodular combinatorial information measures with applications in machine learning.
\newblock In \emph{The 32nd International Conference on Algorithmic Learning Theory}, Virtual Conference, 2021.

\bibitem[Iyer et~al.(2022)Iyer, Khargonkar, Bilmes, and Asnani]{subinfomeasures2022}
Rishabh Iyer, Ninad Khargonkar, Jeff Bilmes, and Himanshu Asnani.
\newblock Generalized submodular information measures: Theoretical properties, examples, optimization algorithms, and applications.
\newblock \emph{IEEE Transactions on Information Theory}, 68\penalty0 (2):\penalty0 752 -- 781, 2022.

\bibitem[Izacard and Grave(2021)]{izacard2021leveraging}
Gautier Izacard and {\'E}douard Grave.
\newblock Leveraging passage retrieval with generative models for open domain question answering.
\newblock In \emph{Proceedings of the 16th Conference of the European Chapter of the Association for Computational Linguistics: Main Volume}, pages 874--880, 2021.

\bibitem[Izacard et~al.(2022)Izacard, Lewis, Lomeli, Hosseini, Petroni, Schick, Dwivedi-Yu, Joulin, Riedel, and Grave]{izacard2022atlas}
Gautier Izacard, Patrick Lewis, Maria Lomeli, Lucas Hosseini, Fabio Petroni, Timo Schick, Jane Dwivedi-Yu, Armand Joulin, Sebastian Riedel, and Edouard Grave.
\newblock Atlas: Few-shot learning with retrieval augmented language models, 2022.

\bibitem[Jamal and Qi(2019)]{Jamal2019task}
Muhammad~Abdullah Jamal and Guo-Jun Qi.
\newblock Task agnostic meta-learning for few-shot learning.
\newblock In \emph{2019 IEEE/CVF Conference on Computer Vision and Pattern Recognition (CVPR)}, pages 11711--11719, 2019.

\bibitem[Jegelka and Bilmes(2010)]{jegelka2010}
S. Jegelka and J. Bilmes.
\newblock Cooperative cuts: Graph cuts with submodular edge weights.
\newblock Technical Report 189, Max Planck Institute for Biological Cybernetics, Tuebingen, Germany, 2010.

\bibitem[Jegelka and Bilmes(2011)]{jegelka2011submodularity}
Stefanie Jegelka and Jeff Bilmes.
\newblock Submodularity beyond submodular energies: Coupling edges in graph cuts.
\newblock In \emph{CVPR 2011}, pages 1897--1904, 2011.

\bibitem[Jia et~al.(2021)Jia, Yang, Xia, Chen, Parekh, Pham, Le, Sung, Li, and Duerig]{jia2021scaling}
Chao Jia, Yinfei Yang, Ye Xia, Yi-Ting Chen, Zarana Parekh, Hieu Pham, Quoc~V. Le, Yunhsuan Sung, Zhen Li, and Tom Duerig.
\newblock Scaling up visual and vision-language representation learning with noisy text supervision, 2021.

\bibitem[Johnson et~al.(2019)Johnson, Douze, and J{\'e}gou]{johnson2019billion}
Jeff Johnson, Matthijs Douze, and Herv{\'e} J{\'e}gou.
\newblock Billion-scale similarity search with {GPUs}.
\newblock \emph{IEEE Transactions on Big Data}, 7\penalty0 (3):\penalty0 535--547, 2019.

\bibitem[Joshi and Mirzasoleiman(2023)]{joshi23bdata}
Siddharth Joshi and Baharan Mirzasoleiman.
\newblock Data-efficient contrastive self-supervised learning: Most beneficial examples for supervised learning contribute the least.
\newblock In \emph{Proceedings of the 40th International Conference on Machine Learning}, pages 15356--15370. PMLR, 2023.

\bibitem[Kandpal et~al.(2023)Kandpal, Deng, Roberts, Wallace, and Raffel]{kandpal2023large}
Nikhil Kandpal, Haikang Deng, Adam Roberts, Eric Wallace, and Colin Raffel.
\newblock Large language models struggle to learn long-tail knowledge.
\newblock In \emph{International Conference on Machine Learning}, pages 15696--15707. PMLR, 2023.

\bibitem[Karanam et~al.(2022)Karanam, Killamsetty, Kokel, and Iyer]{Karanam2022orient}
Athresh Karanam, Krishnateja Killamsetty, Harsha Kokel, and Rishabh Iyer.
\newblock Orient: Submodular mutual information measures for data subset selection under distribution shift.
\newblock In \emph{Advances in Neural Information Processing Systems}, pages 31796--31808. Curran Associates, Inc., 2022.

\bibitem[Kothawade et~al.(2021{\natexlab{a}})Kothawade, Beck, Killamsetty, and Iyer]{koth2021similar}
Suraj Kothawade, Nathan Beck, Krishnateja Killamsetty, and Rishabh Iyer.
\newblock Similar: Submodular information measures based active learning in realistic scenarios.
\newblock In \emph{Advances in Neural Information Processing Systems}, pages 18685--18697. Curran Associates, Inc., 2021{\natexlab{a}}.

\bibitem[Kothawade et~al.(2021{\natexlab{b}})Kothawade, Beck, Killamsetty, and Iyer]{kothawade2021Similar}
Suraj Kothawade, Nathan Beck, Krishnateja Killamsetty, and Rishabh Iyer.
\newblock Similar: Submodular information measures based active learning in realistic scenarios.
\newblock In \emph{Advances in Neural Information Processing Systems}, pages 18685--18697. Curran Associates, Inc., 2021{\natexlab{b}}.

\bibitem[Kothawade et~al.(2022)Kothawade, Kaushal, Ramakrishnan, Bilmes, and Iyer]{kothawade2022prism}
Suraj Kothawade, Vishal Kaushal, Ganesh Ramakrishnan, Jeff Bilmes, and Rishabh Iyer.
\newblock Prism: A rich class of parameterized submodular information measures for guided subset selection, 2022.

\bibitem[Kulesza et~al.(2012)Kulesza, Taskar, et~al.]{kulesza2012determinantal}
Alex Kulesza, Ben Taskar, et~al.
\newblock Determinantal point processes for machine learning.
\newblock \emph{Foundations and Trends{\textregistered} in Machine Learning}, 5\penalty0 (2--3):\penalty0 123--286, 2012.

\bibitem[Lewis et~al.(2020)Lewis, Perez, Piktus, Petroni, Karpukhin, Goyal, K{\"u}ttler, Lewis, Yih, Rockt{\"a}schel, et~al.]{lewis2020retrieval}
Patrick Lewis, Ethan Perez, Aleksandra Piktus, Fabio Petroni, Vladimir Karpukhin, Naman Goyal, Heinrich K{\"u}ttler, Mike Lewis, Wen-tau Yih, Tim Rockt{\"a}schel, et~al.
\newblock Retrieval-augmented generation for knowledge-intensive nlp tasks.
\newblock \emph{Advances in Neural Information Processing Systems}, 33:\penalty0 9459--9474, 2020.

\bibitem[Li and Roth(2002)]{trec}
Xin Li and Dan Roth.
\newblock Learning question classifiers.
\newblock In \emph{{COLING} 2002: The 19th International Conference on Computational Linguistics}, 2002.

\bibitem[Lin and Bilmes(2011)]{lin2011class}
Hui Lin and Jeff Bilmes.
\newblock A class of submodular functions for document summarization.
\newblock In \emph{Proceedings of the 49th Annual Meeting of the Association for Computational Linguistics: Human Language Technologies}, pages 510--520, Portland, Oregon, USA, 2011. Association for Computational Linguistics.

\bibitem[Lin and Bilmes(2012)]{lin2012learning}
Hui Lin and Jeff~A Bilmes.
\newblock Learning mixtures of submodular shells with application to document summarization.
\newblock \emph{arXiv preprint arXiv:1210.4871}, 2012.

\bibitem[Liu et~al.(2023)Liu, Son, Yang, Liu, Gao, Lee, and Li]{liu2023learning}
Haotian Liu, Kilho Son, Jianwei Yang, Ce Liu, Jianfeng Gao, Yong~Jae Lee, and Chunyuan Li.
\newblock Learning customized visual models with retrieval-augmented knowledge, 2023.

\bibitem[Liu et~al.(2022)Liu, Shen, Zhang, Dolan, Carin, and Chen]{liu-etal-2022-makes}
Jiachang Liu, Dinghan Shen, Yizhe Zhang, Bill Dolan, Lawrence Carin, and Weizhu Chen.
\newblock What makes good in-context examples for {GPT}-3?
\newblock In \emph{Proceedings of Deep Learning Inside Out (DeeLIO 2022): The 3rd Workshop on Knowledge Extraction and Integration for Deep Learning Architectures}, Dublin, Ireland and Online, 2022. Association for Computational Linguistics.

\bibitem[Liu et~al.(2013)Liu, Wei, Kirchhoff, Song, and Bilmes]{liu2013submodular}
Yuzong Liu, Kai Wei, Katrin Kirchhoff, Yisong Song, and Jeff Bilmes.
\newblock Submodular feature selection for high-dimensional acoustic score spaces.
\newblock In \emph{2013 IEEE International Conference on Acoustics, Speech and Signal Processing}, pages 7184--7188, 2013.

\bibitem[Maji et~al.(2013)Maji, Rahtu, Kannala, Blaschko, and Vedaldi]{fgvc_aircraft}
Subhransu Maji, Esa Rahtu, Juho Kannala, Matthew Blaschko, and Andrea Vedaldi.
\newblock Fine-grained visual classification of aircraft, 2013.

\bibitem[Margatina et~al.(2023)Margatina, Schick, Aletras, and Dwivedi-Yu]{margatina-etal-2023-active}
Katerina Margatina, Timo Schick, Nikolaos Aletras, and Jane Dwivedi-Yu.
\newblock Active learning principles for in-context learning with large language models.
\newblock In \emph{Findings of the Association for Computational Linguistics: EMNLP 2023}, pages 5011--5034, Singapore, 2023. Association for Computational Linguistics.

\bibitem[Minoux(2005)]{minoux2005accelerated}
Michel Minoux.
\newblock Accelerated greedy algorithms for maximizing submodular set functions.
\newblock In \emph{Optimization Techniques: Proceedings of the 8th IFIP Conference on Optimization Techniques W{\"u}rzburg, September 5--9, 1977}, pages 234--243. Springer, 2005.

\bibitem[Mirzasoleiman et~al.(2020)Mirzasoleiman, Bilmes, and Leskovec]{mirzasoleiman2020coresets}
Baharan Mirzasoleiman, Jeff Bilmes, and Jure Leskovec.
\newblock Coresets for data-efficient training of machine learning models.
\newblock In \emph{International Conference on Machine Learning}, pages 6950--6960. PMLR, 2020.

\bibitem[Nemhauser et~al.(1978)Nemhauser, Wolsey, and Fisher]{nemhauser1978analysis}
George~L Nemhauser, Laurence~A Wolsey, and Marshall~L Fisher.
\newblock An analysis of approximations for maximizing submodular set functions—i.
\newblock \emph{Mathematical programming}, 14:\penalty0 265--294, 1978.

\bibitem[Nichol et~al.(2018)Nichol, Achiam, and Schulman]{nichol2018reptile}
Alex Nichol, Joshua Achiam, and John Schulman.
\newblock On first-order meta-learning algorithms.
\newblock \emph{CoRR}, abs/1803.02999, 2018.

\bibitem[Nilsback and Zisserman(2008)]{flowers102}
Maria-Elena Nilsback and Andrew Zisserman.
\newblock Automated flower classification over a large number of classes.
\newblock In \emph{Indian Conference on Computer Vision, Graphics and Image Processing}, 2008.

\bibitem[Nuggehalli et~al.(2023)Nuggehalli, Zhang, Jain, and Nowak]{nuggehalli2023direct}
Shyam Nuggehalli, Jifan Zhang, Lalit Jain, and Robert Nowak.
\newblock Direct: Deep active learning under imbalance and label noise.
\newblock \emph{arXiv preprint arXiv:2312.09196}, 2023.

\bibitem[Oreshkin et~al.(2018)Oreshkin, L{\'{o}}pez, and Lacoste]{oreshkin2018TADAM}
Boris~N. Oreshkin, Pau~Rodr{\'{\i}}guez L{\'{o}}pez, and Alexandre Lacoste.
\newblock {TADAM:} task dependent adaptive metric for improved few-shot learning.
\newblock \emph{CoRR}, abs/1805.10123, 2018.

\bibitem[Peng et~al.(2023)Peng, Galley, He, Cheng, Xie, Hu, Huang, Liden, Yu, Chen, et~al.]{peng2023check}
Baolin Peng, Michel Galley, Pengcheng He, Hao Cheng, Yujia Xie, Yu Hu, Qiuyuan Huang, Lars Liden, Zhou Yu, Weizhu Chen, et~al.
\newblock Check your facts and try again: Improving large language models with external knowledge and automated feedback.
\newblock \emph{arXiv preprint arXiv:2302.12813}, 2023.

\bibitem[Podell et~al.(2023)Podell, English, Lacey, Blattmann, Dockhorn, M{\"u}ller, Penna, and Rombach]{podell2023sdxl}
Dustin Podell, Zion English, Kyle Lacey, Andreas Blattmann, Tim Dockhorn, Jonas M{\"u}ller, Joe Penna, and Robin Rombach.
\newblock Sdxl: Improving latent diffusion models for high-resolution image synthesis.
\newblock \emph{arXiv preprint arXiv:2307.01952}, 2023.

\bibitem[Pratt et~al.(2022)Pratt, Liu, and Farhadi]{pratt2022does}
Sarah Pratt, Rosanne Liu, and Ali Farhadi.
\newblock What does a platypus look like? generating customized prompts for zero-shot image classification.
\newblock \emph{arXiv preprint arXiv:2209.03320}, 2022.

\bibitem[Radford et~al.(2021)Radford, Kim, Hallacy, Ramesh, Goh, Agarwal, Sastry, Askell, Mishkin, Clark, Krueger, and Sutskever]{radford2021learning}
Alec Radford, Jong~Wook Kim, Chris Hallacy, Aditya Ramesh, Gabriel Goh, Sandhini Agarwal, Girish Sastry, Amanda Askell, Pamela Mishkin, Jack Clark, Gretchen Krueger, and Ilya Sutskever.
\newblock Learning transferable visual models from natural language supervision, 2021.

\bibitem[Rajeswaran et~al.(2019)Rajeswaran, Finn, Kakade, and Levine]{rajeswaran2019metalearning}
Aravind Rajeswaran, Chelsea Finn, Sham Kakade, and Sergey Levine.
\newblock Meta-learning with implicit gradients, 2019.

\bibitem[Ram et~al.(2023)Ram, Levine, Dalmedigos, Muhlgay, Shashua, Leyton-Brown, and Shoham]{ram2023context}
Ori Ram, Yoav Levine, Itay Dalmedigos, Dor Muhlgay, Amnon Shashua, Kevin Leyton-Brown, and Yoav Shoham.
\newblock In-context retrieval-augmented language models.
\newblock \emph{Transactions of the Association for Computational Linguistics}, 11:\penalty0 1316--1331, 2023.

\bibitem[Recht et~al.(2019)Recht, Roelofs, Schmidt, and Shankar]{imagenetv2}
Benjamin Recht, Rebecca Roelofs, Ludwig Schmidt, and Vaishaal Shankar.
\newblock Do imagenet classifiers generalize to imagenet?
\newblock \emph{CoRR}, abs/1902.10811, 2019.

\bibitem[Reimers and Gurevych(2019)]{sentencebert}
Nils Reimers and Iryna Gurevych.
\newblock Sentence-bert: Sentence embeddings using siamese bert-networks.
\newblock \emph{arXiv preprint arXiv:1908.10084}, 2019.

\bibitem[Rubin et~al.(2022)Rubin, Herzig, and Berant]{rubin2022learningretrievepromptsincontext}
Ohad Rubin, Jonathan Herzig, and Jonathan Berant.
\newblock Learning to retrieve prompts for in-context learning, 2022.

\bibitem[Saran et~al.(2023)Saran, Yousefi, Krishnamurthy, Langford, and Ash]{saran2023streaming}
Akanksha Saran, Safoora Yousefi, Akshay Krishnamurthy, John Langford, and Jordan~T. Ash.
\newblock Streaming active learning with deep neural networks, 2023.

\bibitem[Schuhmann et~al.(2022)Schuhmann, Beaumont, Vencu, Gordon, Wightman, Cherti, Coombes, Katta, Mullis, Wortsman, Schramowski, Kundurthy, Crowson, Schmidt, Kaczmarczyk, and Jitsev]{schuhmann2022laion5b}
Christoph Schuhmann, Romain Beaumont, Richard Vencu, Cade Gordon, Ross Wightman, Mehdi Cherti, Theo Coombes, Aarush Katta, Clayton Mullis, Mitchell Wortsman, Patrick Schramowski, Srivatsa Kundurthy, Katherine Crowson, Ludwig Schmidt, Robert Kaczmarczyk, and Jenia Jitsev.
\newblock Laion-5b: An open large-scale dataset for training next generation image-text models, 2022.

\bibitem[Sener and Savarese(2017)]{sener2017active}
Ozan Sener and Silvio Savarese.
\newblock Active learning for convolutional neural networks: A core-set approach.
\newblock \emph{arXiv preprint arXiv:1708.00489}, 2017.

\bibitem[Sener and Savarese(2018)]{k_center_coreset}
Ozan Sener and Silvio Savarese.
\newblock Active learning for convolutional neural networks: A core-set approach.
\newblock In \emph{International Conference on Learning Representations}, 2018.

\bibitem[Settles(2011)]{settles2010}
Burr Settles.
\newblock From theories to queries: Active learning in practice.
\newblock In \emph{Active Learning and Experimental Design workshop In conjunction with AISTATS 2010}, pages 1--18, Sardinia, Italy, 2011. PMLR.

\bibitem[Sharma et~al.(2015)Sharma, Kapoor, and Deshpande]{pmlr-v37-sharma15}
Dravyansh Sharma, Ashish Kapoor, and Amit Deshpande.
\newblock On greedy maximization of entropy.
\newblock In \emph{Proceedings of the 32nd International Conference on Machine Learning}, pages 1330--1338, Lille, France, 2015. PMLR.

\bibitem[Shi et~al.(2023)Shi, Min, Yasunaga, Seo, James, Lewis, Zettlemoyer, and Yih]{shi2023replug}
Weijia Shi, Sewon Min, Michihiro Yasunaga, Minjoon Seo, Rich James, Mike Lewis, Luke Zettlemoyer, and Wen-tau Yih.
\newblock Replug: Retrieval-augmented black-box language models.
\newblock \emph{arXiv preprint arXiv:2301.12652}, 2023.

\bibitem[Shuster et~al.(2021)Shuster, Poff, Chen, Kiela, and Weston]{shuster2021retrieval}
Kurt Shuster, Spencer Poff, Moya Chen, Douwe Kiela, and Jason Weston.
\newblock Retrieval augmentation reduces hallucination in conversation.
\newblock In \emph{Findings of the Association for Computational Linguistics: EMNLP 2021}, pages 3784--3803, 2021.

\bibitem[Snell et~al.(2017)Snell, Swersky, and Zemel]{snell2017prototypical}
Jake Snell, Kevin Swersky, and Richard~S. Zemel.
\newblock Prototypical networks for few-shot learning, 2017.

\bibitem[Socher et~al.(2013)Socher, Perelygin, Wu, Chuang, Manning, Ng, and Potts]{sst2}
Richard Socher, Alex Perelygin, Jean Wu, Jason Chuang, Christopher~D. Manning, Andrew Ng, and Christopher Potts.
\newblock Recursive deep models for semantic compositionality over a sentiment treebank.
\newblock In \emph{Proceedings of the 2013 Conference on Empirical Methods in Natural Language Processing}, pages 1631--1642, Seattle, Washington, USA, 2013. Association for Computational Linguistics.

\bibitem[Su et~al.(2023)Su, Kasai, Wu, Shi, Wang, Xin, Zhang, Ostendorf, Zettlemoyer, Smith, and Yu]{su2023selective}
Hongjin Su, Jungo Kasai, Chen~Henry Wu, Weijia Shi, Tianlu Wang, Jiayi Xin, Rui Zhang, Mari Ostendorf, Luke Zettlemoyer, Noah~A. Smith, and Tao Yu.
\newblock Selective annotation makes language models better few-shot learners.
\newblock In \emph{The Eleventh International Conference on Learning Representations}, 2023.

\bibitem[Sung et~al.(2018)Sung, Yang, Zhang, Xiang, Torr, and Hospedales]{sung2018learning}
Flood Sung, Yongxin Yang, Li Zhang, Tao Xiang, Philip H.~S. Torr, and Timothy~M. Hospedales.
\newblock Learning to compare: Relation network for few-shot learning, 2018.

\bibitem[Tollari et~al.(2009)Tollari, Mulhem, Ferecatu, Glotin, Detyniecki, Gallinari, Sahbi, and Zhao]{tollari2009comparative}
Sabrina Tollari, Philippe Mulhem, Marin Ferecatu, Herv{\'e} Glotin, Marcin Detyniecki, Patrick Gallinari, Hichem Sahbi, and Zhong-Qiu Zhao.
\newblock A comparative study of diversity methods for hybrid text and image retrieval approaches.
\newblock In \emph{Evaluating Systems for Multilingual and Multimodal Information Access}, pages 585--592, Berlin, Heidelberg, 2009. Springer Berlin Heidelberg.

\bibitem[Udandarao et~al.(2023)Udandarao, Gupta, and Albanie]{udandarao2023susx}
Vishaal Udandarao, Ankush Gupta, and Samuel Albanie.
\newblock Sus-x: Training-free name-only transfer of vision-language models, 2023.

\bibitem[Vinyals et~al.(2016)Vinyals, Blundell, Lillicrap, kavukcuoglu, and Wierstra]{vinyals2016Matching}
Oriol Vinyals, Charles Blundell, Timothy Lillicrap, koray kavukcuoglu, and Daan Wierstra.
\newblock Matching networks for one shot learning.
\newblock In \emph{Advances in Neural Information Processing Systems}. Curran Associates, Inc., 2016.

\bibitem[Wallingford et~al.(2023)Wallingford, Ramanujan, Fang, Kusupati, Mottaghi, Kembhavi, Schmidt, and Farhadi]{wallingford2023neural}
Matthew Wallingford, Vivek Ramanujan, Alex Fang, Aditya Kusupati, Roozbeh Mottaghi, Aniruddha Kembhavi, Ludwig Schmidt, and Ali Farhadi.
\newblock Neural priming for sample-efficient adaptation.
\newblock In \emph{Thirty-seventh Conference on Neural Information Processing Systems}, 2023.

\bibitem[Wang and Komatsuzaki(2021)]{gpt-j}
Ben Wang and Aran Komatsuzaki.
\newblock {GPT-J-6B: A 6 Billion Parameter Autoregressive Language Model}.
\newblock \url{https://github.com/kingoflolz/mesh-transformer-jax}, 2021.

\bibitem[Wang et~al.(2024)Wang, Shen, Guo, Stallone, Kim, Golland, and Panda]{wang2024diversity}
Peiqi Wang, Yikang Shen, Zhen Guo, Matthew Stallone, Yoon Kim, Polina Golland, and Rameswar Panda.
\newblock Diversity measurement and subset selection for instruction tuning datasets.
\newblock \emph{arXiv preprint arXiv:2402.02318}, 2024.

\bibitem[Wang et~al.(2023)Wang, Ma, and Chen]{wang2023augmenting}
Yubo Wang, Xueguang Ma, and Wenhu Chen.
\newblock Augmenting black-box llms with medical textbooks for clinical question answering.
\newblock \emph{arXiv preprint arXiv:2309.02233}, 2023.

\bibitem[Wei et~al.(2015)Wei, Iyer, and Bilmes]{wei2015submodularity}
Kai Wei, Rishabh Iyer, and Jeff Bilmes.
\newblock Submodularity in data subset selection and active learning.
\newblock In \emph{International conference on machine learning}, pages 1954--1963. PMLR, 2015.

\bibitem[Zameshina et~al.(2023)Zameshina, Teytaud, and Najman]{zameshina2023diverse}
Mariia Zameshina, Olivier Teytaud, and Laurent Najman.
\newblock Diverse diffusion: Enhancing image diversity in text-to-image generation.
\newblock \emph{arXiv preprint arXiv:2310.12583}, 2023.

\bibitem[Zancato et~al.(2023)Zancato, Achille, Liu, Trager, Perera, and Soatto]{zancato2023train}
L. Zancato, A. Achille, T. Liu, M. Trager, P. Perera, and S. Soatto.
\newblock Train/test-time adaptation with retrieval.
\newblock In \emph{2023 IEEE/CVF Conference on Computer Vision and Pattern Recognition (CVPR)}, pages 15911--15921, Los Alamitos, CA, USA, 2023. IEEE Computer Society.

\bibitem[Zeng et~al.(2023)Zeng, Wang, Liao, Li, Huang, Xu, Cao, and Man]{vmig2023}
Yawen Zeng, Yiru Wang, Dongliang Liao, Gongfu Li, Weijie Huang, Jin Xu, Da Cao, and Hong Man.
\newblock Keyword-based diverse image retrieval with variational multiple instance graph.
\newblock \emph{IEEE Transactions on Neural Networks and Learning Systems}, 34\penalty0 (12):\penalty0 10528--10537, 2023.

\bibitem[Zhang et~al.(2023{\natexlab{a}})Zhang, Chen, Zhang, Keung, Liu, Zan, Mao, Lou, and Chen]{zhang2023repocoder}
Fengji Zhang, Bei Chen, Yue Zhang, Jacky Keung, Jin Liu, Daoguang Zan, Yi Mao, Jian-Guang Lou, and Weizhu Chen.
\newblock Repocoder: Repository-level code completion through iterative retrieval and generation.
\newblock In \emph{Proceedings of the 2023 Conference on Empirical Methods in Natural Language Processing}, pages 2471--2484, 2023{\natexlab{a}}.

\bibitem[Zhang et~al.(2022)Zhang, Katz-Samuels, and Nowak]{zhang2022galaxy}
Jifan Zhang, Julian Katz-Samuels, and Robert Nowak.
\newblock Galaxy: Graph-based active learning at the extreme.
\newblock In \emph{International Conference on Machine Learning}, pages 26223--26238. PMLR, 2022.

\bibitem[Zhang et~al.(2023{\natexlab{b}})Zhang, Canal, Zhu, Nowak, Chen, Das, Bhatt, Mussmann, Bilmes, Du, et~al.]{zhanglabelbench}
Jifan Zhang, Gregory Canal, Yinglun Zhu, Robert~D Nowak, Yifang Chen, Arnav~M Das, Gantavya Bhatt, Stephen Mussmann, Jeffrey Bilmes, Simon~S Du, et~al.
\newblock Labelbench: A comprehensive framework for benchmarking adaptive label-efficient learning.
\newblock \emph{arXiv preprint arXiv:2306.09910}, 2023{\natexlab{b}}.

\bibitem[Zhang et~al.(2023{\natexlab{c}})Zhang, Shao, Verma, and Nowak]{zhang2023algorithm}
Jifan Zhang, Shuai Shao, Saurabh Verma, and Robert Nowak.
\newblock Algorithm selection for deep active learning with imbalanced datasets, 2023{\natexlab{c}}.

\bibitem[Zhang et~al.(2021)Zhang, Fang, Zhang, Gao, Li, Dai, Qiao, and Li]{tipadapter}
Renrui Zhang, Rongyao Fang, Wei Zhang, Peng Gao, Kunchang Li, Jifeng Dai, Yu Qiao, and Hongsheng Li.
\newblock Tip-adapter: Training-free clip-adapter for better vision-language modeling, 2021.

\bibitem[Zhang et~al.(2023{\natexlab{d}})Zhang, Hu, Li, Huang, Deng, Li, Qiao, and Gao]{cafo}
Renrui Zhang, Xiangfei Hu, Bohao Li, Siyuan Huang, Hanqiu Deng, Hongsheng Li, Yu Qiao, and Peng Gao.
\newblock Prompt, generate, then cache: Cascade of foundation models makes strong few-shot learners.
\newblock \emph{arXiv preprint arXiv:2303.02151}, 2023{\natexlab{d}}.

\bibitem[Zhao et~al.(2023)Zhao, Wang, Liao, Wang, Duan, and Zhou]{colt2023}
Minyi Zhao, Jinpeng Wang, Dongliang Liao, Yiru Wang, Huanzhong Duan, and Shuigeng Zhou.
\newblock Keyword-based diverse image retrieval by semantics-aware contrastive learning and transformer, 2023.

\bibitem[Zhou et~al.(2022{\natexlab{a}})Zhou, Yang, Loy, and Liu]{cocoop}
Kaiyang Zhou, Jingkang Yang, Chen~Change Loy, and Ziwei Liu.
\newblock Conditional prompt learning for vision-language models.
\newblock In \emph{IEEE/CVF Conference on Computer Vision and Pattern Recognition (CVPR)}, 2022{\natexlab{a}}.

\bibitem[Zhou et~al.(2022{\natexlab{b}})Zhou, Yang, Loy, and Liu]{coop}
Kaiyang Zhou, Jingkang Yang, Chen~Change Loy, and Ziwei Liu.
\newblock Learning to prompt for vision-language models.
\newblock \emph{International Journal of Computer Vision (IJCV)}, 2022{\natexlab{b}}.

\bibitem[Zhou et~al.(2022{\natexlab{c}})Zhou, Alon, Xu, Jiang, and Neubig]{zhou2022docprompting}
Shuyan Zhou, Uri Alon, Frank~F Xu, Zhengbao Jiang, and Graham Neubig.
\newblock Docprompting: Generating code by retrieving the docs.
\newblock In \emph{The Eleventh International Conference on Learning Representations}, 2022{\natexlab{c}}.

\end{thebibliography}
}

\newpage
\onecolumn
\begin{center}
    {\Large
    \textbf{\thetitle}\\
    \vspace{0.5em}Supplementary Material \\
    \vspace{1.0em}}
\end{center}
 
\section{Summary of Notations}
We provide a list of notations in the paper in the \cref{tab:summary}.
\begin{table}[htp]
    \small
    \centering
        \begin{tabular}{c|c}
        \toprule
        Notation                   & Meaning                     
        \\
        \midrule 
        $\mathcal{D}^{\text{aux}}$ & Auxiliary dataset                    \\
        $\mathcal{D}^{\text{tar}}$ & Target dataset                       \\
        $V^{\text{aux}}$           & $\Daux$ indices      \\
        $V^{\text{tar}}$           & $\Dtar$ indices         \\
        $\mathcal{X}$              & Domain of images  \\
        $\mathcal{Y}$              & Domain of labels     \\
        $\mathcal{Z}$              & Domain of $\Daux$ elements             \\
        $\Wmat$                    & Pairwise similarity matrix           \\
        $w_{ij}$                   & Element $i,j$ of $\Wmat$             \\
        $f$                        & Submodular function                  \\
        $I_f$                      & Submodular mutual information      \\
        \bottomrule
        \end{tabular}
        \caption{\textbf{Summary of Notations.}}
        \label{tab:summary}
    \vspace{-.2in}
\end{table}
\section{Application to In-Context Learning}
\label{app:nlp}
While our work primarily focuses on vision-related applications, COBRA is a versatile retrieval strategy that can be applied across a wide range of settings. In this section, we explore its use within the framework of in-context learning (ICL) with large language models (LLMs). ICL refers to the process by which an LLM leverages information presented in the input prompt or context to perform tasks or make predictions without requiring additional training or parameter updates~\citep{brown2020languagemodelsfewshotlearners}. For a given test query $x$, a prediction $\hat{y}$ can be obtained by constructing a prompt that incorporates a set of labeled samples $\{(x_1, y_1), ..., (x_n, y_n)\}$. However, due to the limited context window of LLMs, only a subset of labeled samples can be used in the prompt. This subset selection problem can be interpreted as a retrieval task. Using the previously defined notation, the labeled pool corresponds to $\Vaux$, while the query to be labeled is represented as $\Vtar$ where typically $|\Vtar| = 1$. The strategy commonly employed to retrieve labeled samples for prompting is \emph{Sim-Score}~\citep{brown2020languagemodelsfewshotlearners, margatina-etal-2023-active, su2023selective, liu-etal-2022-makes, rubin2022learningretrievepromptsincontext}. However, COBRA can easily be used for this task as well to introduce diversity into the set of selected in-context exemplars.

\begin{wraptable}{r}{0.5\textwidth}
    \small
    \centering
    \begin{tabular}{l|c|c|c|c}
        \toprule[1.5pt]
        \textbf{Algorithm} & \textbf{MRPC} & \textbf{SST-2} & \textbf{RTE} & \textbf{TREC} \\
        \midrule
        \emph{Sim-Score} & 67.79 & 90.25 & 54.15 & 86.80 \\
        COBRA & \textbf{68.62} & \textbf{90.71} & \textbf{56.68} & \textbf{87.90} \\
        \bottomrule
    \end{tabular}
    \caption{\textbf{ICL performance}. We consider four classification tasks and evaluate how effective various retrieval appoaches are at selecting samples for prompting the GPT-J-6B model. Test accuracy (in \%) is reported for all settings. Since ICL takes place during test time and we use a pretrained LLM, all numbers are deterministic so error bars are not included.}
    \label{tab:icl}
    \vspace{-.05in}
\end{wraptable}

We consider several text classification datasets including MRPC~\citep{mrpc}, SST-2~\citep{sst2}, RTE~\citep{rte}, and TREC~\citep{trec}. We first use Sentence-BERT~\citep{sentencebert} to create embeddings for the full dataset, which are then used to instantiate the similarity matrix. For each sample in the evaluation set, we use the retrieval strategy of choice to select as many samples as we can fit in the context window from the training set.
These samples are then used to construct a prompt following the template and ordering schemes used in ~\cite{su2023selective}.
Finally, we feed the prompt into the GPT-J-6B model~\citep{gpt-j} and use the verbalizer from~\cite{su2023selective} to recover a prediction. In this set of experiments, we find that COBRA consistently outperforms Sim-Score as shown in Table~\ref{tab:icl}. We include qualitative examples of the types of examples COBRA retrieves in \Cref{app:text_qual}.

\section{Additional Related Work}
\label{app:more_related_work}
\paragraph{Few-shot Learning} Few-shot learning techniques are designed to enhance the data efficiency of neural networks, enabling them to generalize effectively from a limited amount of data. Older lines of research investigate using learned metrics to quantify the distances between images at test time and examplar images belonging to novel categories for classification tasks~\citep{vinyals2016Matching, snell2017prototypical, sung2018learning, oreshkin2018TADAM}, while others have leveraged meta-learning based approaches to learn a set of initial weights that can be quickly adapted towards novel downstream tasks~\citep{finn17maml, Jamal2019task, nichol2018reptile, rajeswaran2019metalearning}. Upon the emergence of foundation models, research in this area has shifted towards leveraging the transferability of large-scale vision-language models that have been trained on web-scale datasets such as CLIP~\citep{radford2021learning} or ALIGN~\citep{jia2021scaling} to learn from limited data. This area is typically referred to as \emph{few-shot adaptation}. Modern techniques in few-shot adaptation typically add a limited number of learnable parameters, while keeping the main image and/or text encoders frozen when adapting to a new task~\citep{coop, cocoop, clipadapter, tipadapter, cafo}. These recent methods are \emph{model-centric} approaches towards improving adaptation, while retrieval methods that focus on acquiring new data points to train on are \emph{data-centric} methods, which are complementary. 

\paragraph{Active Learning} Active learning methods seek to minimize the amount of data that needs to be labeled by selectively querying the most informative data points from a pool of unlabeled samples. Such techniques typically employ model-dependent measures of uncertainty~\citep{atlas1989training, settles2010,gal2017deep,ducoffe2018adversarial,beluch2018power}, diversity~\citep{k_center_coreset,geifman2017deep,citovsky2021batch}, or some combination of  both~\citep{wei2015submodularity,ash2019deep,ash2021gone,zhang2022galaxy,nuggehalli2023direct} for the selection criteria. While most prior active learning works consider settings where the unlabeled and initial labeled pool have the same distribution~\citep{k_center_coreset,geifman2017deep,citovsky2021batch,zhanglabelbench}, some recent work considers class-imbalanced active learning~\citep{kothawade2021Similar, coleman2021similarity, zhang2022galaxy, zhang2023algorithm} where the unlabeled pool is class-imbalanced and may even include out-of-distribution samples. However, active learning assumes that task-specific human annotations are provided for the selected set of samples. This is distinct from the retrieval setting where human annotations are free-form captions that may not be specific to the downstream task.

\paragraph{Distinction from Neural Priming \citep{wallingford2023neural}}
Neural Priming is a recent approach that seeks to improve the zero-shot accuracy of pretrained VLMs by retrieving from samples from LAION-2B. However, this approach differs from the setting we consider in the following ways. \cite{wallingford2023neural} considers a transductive setting where the samples are retrieved based on the test data. This is distinct from our work, where samples are retrieved based only on the available training dataset. Moreover, \cite{wallingford2023neural} only uses a single retrieval strategy which is a combination of \emph{CLIP-score} and \emph{Sim-Score}. In contrast, our work seeks to comprehensively evaluate existing retrieval strategies and demonstrate the superiority of COBRA. 

\paragraph{Distinction from REACT \citep{liu2023learning}} REACT~\citep{liu2023learning} proposes an end-to-end scheme for improving the performance of VLMs. REACT retrieves around 6-10 million based on either \emph{CLIP-Score} or \emph{Sim-Score} based on a classification task, and updates a significant number of the original parameters by training contrastively on the retrieved image/caption pairs.  REACT fixes the retrieval strategy and focuses on which architectural and procedural design choices are important for improving zero-shot performance. Unlike REACT, we explore different retrieval strategies while keeping the few-shot learning strategy fixed. These two directions are complementary to one another, and REACT could benefit from a diversity-aware retrieval strategy such as COBRA. 

\section{Limitations}
\label{app_sec: limitations}
A key limitation of COBRA is that it assumes that the auxiliary dataset has some samples that are relevant to the target task. However, this assumption may not hold for highly novel tasks or specialized domains. Furthermore, diverse retrieval may not offer significant improvements in performance if the train set and test set of the target are highly homogeneous (see~\Cref{appen sec: qualitative flowers}). 

\section{Background and Lemmas from the Main Paper}
\label{sec: appen background and lemma}
In this section, we re-write the lemmas from the main paper (\Cref{sec: background} and \Cref{sec: methods} to be precise). 

\begin{lemma} [\textit{Graph Cut Mutual Information}]
\label{app_thm: kNN_as_GCMI}
Let $G = (V, E)$ be a graph with edge weights defined with symmetric $\Wmat = [w]_{i, j} \in \R_{+}\cup \{0\} $. For any set $A \subseteq V$ of vertices, let $f(A) = \sum_{i \in A} \sum_{j \in V \setminus A} w_{ij}$ be the graph cut function. Now given any two sets A and B, such that, $ A \cap B = \emptyset $, then, 
\begin{equation}
    I_{f}(A;B) = 2 \sum_{i \in A} \sum_{j \in B} w_{ij}
\end{equation}
\end{lemma}
\begin{proof}
We first expand the value of $f(A)$ for given any $A$ and $B$ such that $A \cap B = \emptyset$
\begin{align}
  f(A) & = \sum_{i \in A}\sum_{j \in V \setminus A} w_{i, j}\\
  & = \sum_{i \in A}\sum_{j \in B} w_{i, j} + \sum_{i \in A}\sum_{j \in V \setminus (A \cup B)} w_{i, j} 
\end{align}
Similarly, 
\begin{align}
  f(B) & = \sum_{i \in B}\sum_{j \in V \setminus B} w_{i, j}\\
  & = \sum_{i \in B}\sum_{j \in A} w_{i, j} + \sum_{i \in B}\sum_{j \in V \setminus (A \cup B)} w_{i, j} 
\end{align}
and,
\begin{align}
  f(A \cup B) & = \sum_{i \in A \cup B}\sum_{j \in V \setminus (A \cup B)} w_{i, j}\\
  & =  \sum_{i \in A}\sum_{j \in V \setminus (A \cup B)} w_{i, j}  +  \sum_{i \in B}\sum_{j \in V \setminus (A \cup B)} w_{i, j} 
\end{align}
Therefore, 
\begin{align}
  f(A) + f(B) - f(A \cup B) & = \sum_{i \in A}\sum_{j \in B} w_{i, j} + \sum_{i \in A}\sum_{j \in B} w_{j, i}\\
  & =  2 \sum_{i \in A} \sum_{j \in B} w_{i, j} \\
  &= I_{f}(A;B)
\end{align}
Where the last step is due to the symmetric nature of the matrix.
\end{proof}

\begin{lemma}[\textit{Soft Class Balancing}]
\label{app_lemma: class_balancing}
Let $h : V \to [C]$ map any image (indexed using V) to the corresponding (pseudo)label among C classes. For any subset $A \subseteq V$, define the count for \textit{u}-th class in set A as $m_u(A) \triangleq \sum_{a \in A} \I[h(a) = u]$ and normalized count as $\hat{p}_u(A) \triangleq m_u(A)/|A|$. Further denote $\hat{\p}(A) = (\hat{p}_1(A), \hat{p}_2(A), \ldots, \hat{p}_C(A))$ the empirical probability based on normalized counts. For any given probability distribution $\p$ defined over C classes, and $k \in \Natural$, we have the following -

\begin{equation}
    \argmin_{\substack{A \subseteq V \\ |A| = k}}\mathbb{D}_{\text{KL}} \left( \p \mid \mid \hat{\p}(A)\right) = \argmax_{\substack{A \subseteq V \\ |A| = k}} \sum_{u = 1}^C p_u \log{(m_u(A))}
\end{equation}
In fact for $\p = \left(\frac{1}{C}, \ldots, \frac{1}{C}\right)$, that is, uniform distribution over each class label, maximizing $\sum_{u = 1}^C \frac{\log{(m_u(A))}}{C}$ is equivalent to finding an $A \subseteq V$, $|A|=k$ that is class balanced. 
\end{lemma}
\begin{proof}
    The proof step follows from expanding the definition of KL-Divergence. 

\begin{align}
     \mathbb{D}_{\text{KL}} \left( \p \mid \mid \hat{\p}(A)\right) &= \sum_{u=1}^C  p_u \log{\left(\frac{p_u}{\hat{wp}(A)_u}\right)}    \\
     &= \sum_{u=1}^C  p_u \log{p_u} - \sum_{u=1}^C  p_u \log{\hat{p}(A)_u} \\
     &= - H(\p) - \sum_{u=1}^C  p_u \log{\left(\frac{m_u(A)}{|A|}\right)} \\
     &= - H(\p) + \log{|A|} - \sum_{u=1}^C  p_u \log{m_u(A)} \\
     &= - \left(H(\p) - \log{|A|} + \sum_{u=1}^C  p_u \log{m_u(A)}\right)
\end{align}
Therefore, 
\begin{align}
    \argmin_{\substack{A \subseteq V \\ |A| = k}}\mathbb{D}_{\text{KL}} \left( \p \mid \mid \hat{\p}(A)\right) &= \argmin_{\substack{A \subseteq V \\ |A| = k}} - \left(H(\p) - \log{|A|} + \sum_{u=1}^C  p_u \log{m_u(A)}\right)\\
    &= \argmax_{\substack{A \subseteq V \\ |A| = k}}  \left(H(\p) - \log{|A|} + \sum_{u=1}^C  p_u \log{m_u(A)}\right) \\ 
    &= \argmax_{\substack{A \subseteq V \\ |A| = k}} \sum_{u = 1}^C p_u \log{(m_u(A))}
\end{align}

\end{proof}

\section{Complexity Analysis}
\label{appen sec: complexity}

For simplicity, let $M$ ($|\Vtar|$ according to the main paper) be the total number of training samples provided and $N$ ($|\Vaux|$ according to the main paper) be the size of an auxiliary dataset. For our experiments, we use sparse matrices constructed using FAISS with $r$ nearest neighbors in each row, where $r << M+N$; this makes our space complexity $\mathcal{O}(r(M+N))$. For the time complexity, assuming we select $k$ samples, given the greedy procedure (further accelerated by the priority queue \citep{minoux2005accelerated}), it is  $\mathcal{O}(Nk)$, which is of a similar order to Sim-Score. For the time complexity of constructing a similarity matrix, while a brute force method will be $\mathcal{O}(rN^2)$, using FAISS it can be brought down to $\mathcal{O}(rN \text{poly}(\log N))$ with the help of HNSW for approximate nearest neighbor search. For empirical run time, it takes roughly 45 min to construct the similarity matrix in the order of millions, which is not significant compared to the cost of obtaining the features for the full retrieval pool (necessary for all retrieval strategies).

\section{Submodular Maximization}
\label{appen sec: submodular max}
In this section, we provide the pseudocode for maximizing submodular function. Since COBRA uses a monotone non-decreasing and normalized ($f(\emptyset)=0$) submodular function (a.k.a polymatroidal functions) for objective, we outline the greedy algorithm in \Cref{alg: greedy} that offers $1-e^{-1}$ approximation to the true maximum under cardinality constraints \citep{nemhauser1978analysis}.  

\begin{algorithm}
\caption{Greedy Algorithm for Maximizing polymatroidal function under Cardinality Constraint}
\label{alg: greedy}
\begin{algorithmic}[1]
\Function{NEMHAUSER-GREEDY}{$f$, $V$, $k$}
    \State \textbf{Input:} Submodular function $f$, ground set $V$, cardinality constraint  $k$
    \State \textbf{Output:} $S$ that maximizes $f$ under cardinality constraint
    \State Initialize $S_0 \gets \emptyset$
    \State $t \gets 0$
    \While{$|S_t|< k$}
        \State $v \gets \arg\max_{v \in V \setminus S_t} f(v \mid S_t)$ \Comment{$f(x \mid S_t) = f(S_t \cup \{v\}) - f(S_t)$}
        \State $S_{t+1} \gets S_t \cup \{v\}$
        \State $t \gets t + 1$
    \EndWhile
    \State \Return $S_k$
\EndFunction
\end{algorithmic}
\end{algorithm}

\section{Additional Baseline Details}
\label{app:baseline}
In this section, we compare to other retrieval baselines. Below, we describe each additional approach. Note that, similar to the main paper, every additional baseline is conducted on Imagenet-1K. 


\paragraph{Maximum Marginal Relevance (MMR)} Maximal Marginal Relevance is a commonly used approach for information retrieval~\citep{carbonell1998} that balances diversity and relevance. This approach iteratively adds $v \in \Vaux$ to $A$, until a cardinality limit is attained, by solving the following optimization problem:

\begin{equation}
    v \triangleq \operatorname*{arg\,max}_{i \in \Vaux \setminus A} \left[ \lambda_{MMR} \left( \text{sim}_1(i, \Vtar) \right) - (1 - \lambda_{MMR}) \max_{j \in A} \text{sim}_2(i, j) \right]
\end{equation}

Intuitively, the first term encourages selecting elements that are similar to $\Vtar$ while the second term ensures that the new samples are not too close to samples already in $A$. In the implementation we use for experiments, we use $\text{sim}_1(i, \Vtar) = \max_{j \in \Vtar} w_{ij}$ and $\text{sim}_2(i, j) = w_{ij}$.

\paragraph{CLIP-Score w/ Diverse LLM Prompts} It is possible to use LLM-generated prompts with CLIP-Score, as proposed in Sus-X~\citep{udandarao2023susx}. Specifically, we generate several prompts for each class using the templates generated by GPT-3 provided by~\cite{pratt2022does}, retrieve a few samples with high similarity to each prompt, and aggregate all samples to create the final set of retrieval samples. While this may improve the diversity of samples selected when compared to CLIP-Score with a single text prompt, this approach still does not model interactions between the selected samples. Therefore, duplicates or near duplicates may still be selected if there are text prompts that are similar to one another. 

\paragraph{SDXL-Aug w/ Diverse LLM Prompts}
Similar to the previous section, we use LLM-generated prompts to sample synthetic data from Stable Diffusion XL (SDXL). Specifically, we generate several prompts for each class using the templates generated by GPT-3 provided by~\cite{pratt2022does}, and then use them to prompt SDXL. For the same reason as above, this does not guarantee diversity since GPT-3 outputs may still be redundant. 

\paragraph{Sample Efficiency}
We now study if retrieving more samples from the baselines, Sim-score and CLIP-score respectively can outperform COBRA. Specifically, fixing the retrieval budget for COBRA to 16K, i.e.,  16 images per class (IPC), and varying the baseline retrieval budget to 32K, 64K, and 128K, i.e., 32, 64, and 128 IPC, respectively. As expected, retrieving more samples does not mean it would not have redundancy, and even though baselines retrieve 8x more samples than COBRA, they continue to underperform as shown in \Cref{fig:sample_eff}. 

\begin{figure}[htp]
    \centering
    \includegraphics[width=0.5\linewidth]{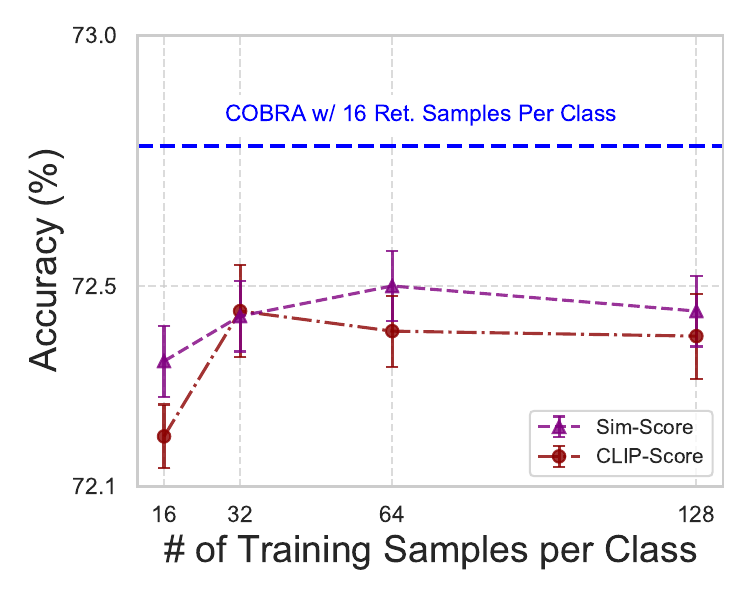}
    \caption{\textbf{Sample Efficiency} We demonstrate that retrieving more samples does not necessarily improve the performance of Sim-Score and CLIP-Score and that COBRA maintains its edge even when retrieving 8x fewer samples.}
    \label{fig:sample_eff}
\end{figure}

\section{Additional Few-Shot Learning Strategies}

In this section, we include results that test various retrieval strategies for other, more recent, few-shot learning strategies~\cite{APE2023, AMU2024}. We demonstrate in Figure~\ref{fig:ape_amu} that COBRA continues to outperform the next best baseline. 

\begin{figure}[htp!]
    \centering
    \begin{subfigure}[t]{0.37\columnwidth}
        \includegraphics[width=\textwidth, height=\textwidth]{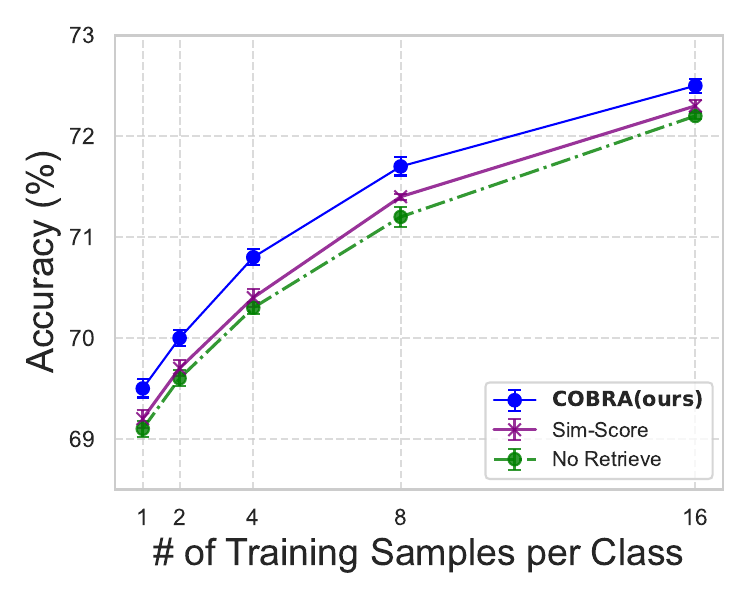}
        \vspace{-.2in}
        \caption{APE~\cite{APE2023}}
        \label{fig:ape}
        \vspace{-.1in}
    \end{subfigure}
    \hspace{0.1\linewidth}
    \begin{subfigure}[t]{0.37\columnwidth}
        \includegraphics[width=\textwidth, height=\textwidth]{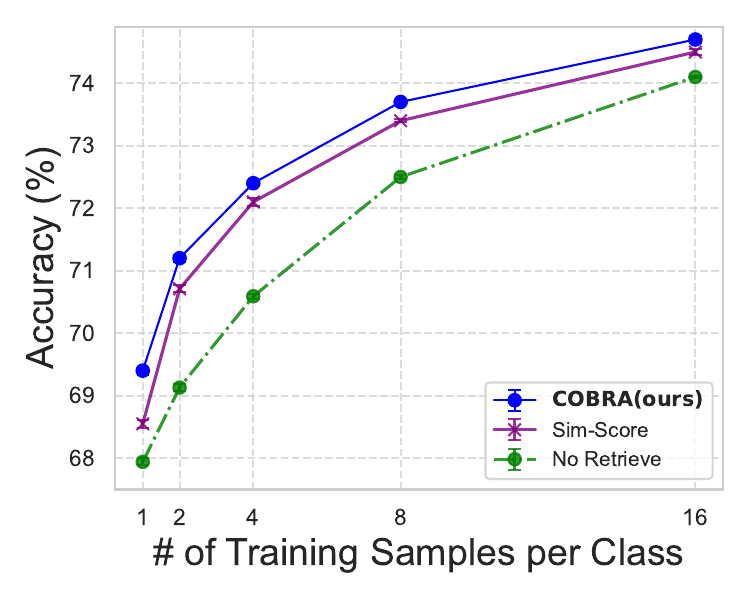}
        \vspace{-.2in}
        \caption{AMU~\cite{AMU2024}}
        \label{fig:amu}
        \vspace{-.1in}
    \end{subfigure}

    \caption{Results of using COBRA with additional few-shot adaptation strategies}
    \label{fig:ape_amu}
\end{figure}
\newpage

\section{Toy Example Setup}
\label{appen sec: toy example setup}

Here we describe the toy example setup from the main paper in greater detail. For $\Dtar$ we use 16 samples from each of 4 different Gaussian, making it 64 samples per class. For the $\Daux$ we sample 25000 points from a mixture of Gaussian such that the $\Daux$ is heavily imbalanced between its modes and reflects how auxiliary data may be quite skewed
in real-world settings. For similarity between any two points, $i$, and $j$, we use a simple $\ell_2$ distance-based kernel, that is, $w_{i, j} = \exp{\left(-\|x_i-x_j\|^2_2\right)}$. Using this we instantiate $\Wmat$ described in the main paper, for both the COBRA as well as \emph{Sim-score}. To provide class information, we need to associate every sample in this $\Daux$ with one of the 4 classes in $\Dtar$. To do so, we fit a Linear SVM \citep{cortes1995support} on 64 samples from $\Dtar$ and use its predicted labels on $\Daux$ to associate each of its samples with one of the 4 labels. Lastly, we do not use any optional quality function.

With the setup above, retrieval is performed to fetch 128 samples from $\Daux$, and is shown in Fig.\ref{fig:toy_appendix}.

\begin{figure}[]
    \centering
    \includegraphics[scale=0.75]{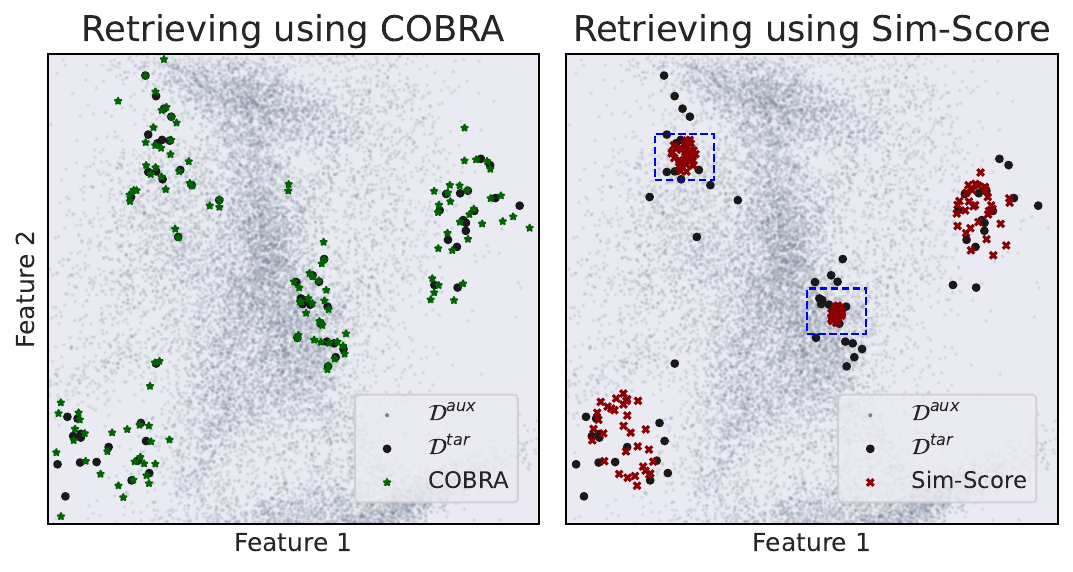}
    \caption{\textbf{2D Example} We consider a simple example where $|\Dtar| = 64$ and $|\Daux| = 25000$. From $\Daux$, we retrieve a subset of size 128 based on $\Dtar$. COBRA (left) effectively covers the target $\Dtar$, on the other hand, \emph{Sim-Score} (right) selects clumpy samples, as highlighted by the bounding boxes.}
    \label{fig:toy_appendix}
\end{figure}

\section{Implementation Details}
\label{appen sec: hyperparameters}
For convenience, we again describe the auxiliary data collection procedure here.

\subsection{Auxiliary Data creation}

We follow the following procedure that gives us $\Daux$ and associated class labels for each sample in $\Daux$ when provided with target data $\Dtar$.

\begin{enumerate}[leftmargin=*,labelindent=0em,partopsep=-2pt,topsep=1pt,itemsep=2pt]
    \item \textbf{Collect Target Dataset} $\Dtar$: We collect a small target dataset by sampling a standard image classification dataset uniformly at random, retaining 1-16 images per class.
    \item \textbf{Collect Auxilary Dataset} $\Daux$: We follow the prefiltering step proposed by~\cite{wallingford2023neural} and use string matching to discard images with captions that do not contain the name of any class name in the target dataset. This stage circumvents the need to compute features for the full auxiliary pool while filtering out images that are unlikely to contain any semantically relevant information. \textbf{\textit{We then use the class name contained in the caption as the label for the image}}. We will refer to this \textit{(pseudo)labels} with \textit{y} from now onwards. 
\end{enumerate}

\subsection{Similarity Construction, COBRA and CLIP-Score}
Both COBRA and Sim-score use a similarity matrix to model similarity between and among the samples in $\Dtar$ and $\Daux$. To this end, we model similarity between any two points, $i, j \in V$ (recall that $V = \Vtar \cup \Vaux$, and indexing defined jointly over $\Dtar$ and $\Daux$), we use a simple dot product based similarity. Since the model we are fine-tuning is a pretrained CLIP model, we use its vision encoder (call it $h_{\text{vision}}$) to provide a mapping of every $i, j \in V$ in some $d$-dimensional space. Therefore, $\Wmat = [w_{i, j}]$ is such that $w_{i, j} = 1 + \frac{\langle h_{\text{vision}}(x_i), h_{\text{vision}}(x_j) \rangle}{\|h_{\text{vision}}(x_i)\| \cdot \|h_{\text{vision}}(x_j)\|}$. Note that adding 1 makes sure that $w_{i, j} \geq 0$ for every pair. We use this directly for \emph{Sim-score}, such that $\text{Sim-score}(i) = \sum_{j \in \Vtar} w_{i, j}$, however, for COBRA we perform further sparsification. 

We sparsify the $\Wmat$ in two stages. First, we set $w_{i, j} = 0$ such that $y_i \neq y_j$ (that is, if $i$ and $j$ do not have the same (pseudo)label). We then further sparsify the remaining matrix, by only retaining the top-$k$ (hyperparameter) entries in each row and setting the rest of them to be 0. This corresponds to finding the top-$k$ nearest neighbors according to cosine distance. Since COBRA can have an optional quality function $q$, we use this to be the Sim-score itself. That is for any $i \in \Vaux$, $q(i) = \sum_{j \in \Vtar} w_{i, j}$, and therefore for any $A \subseteq \Vaux$, $q(A) = \sum_{i \in A} \sum_{j \in \Vtar} w_{i, j}$. 

For CLIP-score, we use the text templates provided by \cite{radford2021learning} for every class in the target dataset. Then based on the (pseudo)label for every sample in $\Daux$, we compute the \emph{CLIP-score} for every sample $i \in \Vaux$ as $\text{CLIP-score}(i) = \frac{\langle h_{\text{vision}}(x_i), h_{\text{text}}(\mathcal{T}(y_i))) \rangle}{\|h_{\text{vision}}(x_i)\| \cdot \|h_{\text{text}}(\mathcal{T}(y_i))) \|}$, where $\mathcal{T}(y)$ is text template for any label $y$ and $\text{text}$ is text encoder for the CLIP model. 

\subsection{Hyperparameters}
We mention the basic learning hyperparameters in the table~\ref{tab:hyperameters_learning}. We determine the hyperparameter configuration to use by running a sweep while training on only the samples from the target dataset in the 16-shot setting. We then fix these hyperparameters across all retrieval algorithms. We provide the COBRA-specific hyperparameters in - \Cref{tab: hyperameters_cobra_in1k_vitb_16}, \Cref{tab: hyperameters_cobra_fgvc_vitb_16}, \Cref{tab: hyperameters_cobra_flowers_vitb_16}, \Cref{tab: hyperameters_cobra_dtd_vitb_16} and \Cref{tab: hyperameters_cobra_food_vitb_16} for each dataset. Lastly, for ablation on different backbones, we keep $\mu$ and $\lambda$ fixed to the original values, while only tuning the k (neighbor) parameter, which we discuss in \Cref{tab: hyperameters_backbone_k}. Note that for reporting numbers we run three trials. 

\begin{table}[]
\small
\centering
\caption{Basic Hyperparameters used for learning algorithm}
\vspace{0.1in}
\resizebox{0.67\textwidth}{!}{%
\begin{tabular}{cccc}
\toprule
\textbf{Dataset}                        & \textbf{Name}          & \textbf{Final Value} & \textbf{Tuning Range}               \\
\midrule
\multirow{3}{*}{Imagenet}      & Learning Rate & 0.01    & \{0.001, 0.01, 0.02, 0.1\} \Tstrut \\
                               & Epochs        & 20          & \{20, 30, 40\}             \\
                               & Optimizer     & SGD         & \{SGD, AdamW\}             \\
\midrule
\multirow{3}{*}{FGVC-Aircraft} & Learning Rate & 0.001    & \{0.001, 0.01, 0.02, 0.1\} \\
                               & Epochs        & 30          & \{20, 30, 40\}             \\
                               & Optimizer     & SGD         & \{SGD, AdamW\}             \\
\midrule
\multirow{3}{*}{Flowers-102}   & Learning Rate &       0.01      & \{0.001, 0.01, 0.02, 0.1\} \\
                               & Epochs        &   20         & \{20, 30, 40\}             \\
                               & Optimizer     &        AdamW     & \{SGD, AdamW\}             \\
\midrule
\multirow{3}{*}{DTD}           & Learning Rate &      0.1       & \{0.001, 0.01, 0.02, 0.1\} \\
                               & Epochs        & 20 & \{20, 30, 40\}             \\
                               & Optimizer     &      SGD       & \{SGD, AdamW\}    \Bstrut  \\      
\midrule
\multirow{3}{*}{Food101}           & Learning Rate &      0.001   & \{0.001, 0.01, 0.02, 0.1\} \\
                  & Epochs        & 20  & \{20, 30, 40\}             \\
                               & Optimizer     & AdamW      & \{SGD, AdamW\}    \Bstrut  \\      

\bottomrule
\end{tabular}}
\label{tab:hyperameters_learning}
\end{table}

\begin{table}[]
\centering
\caption{Hyperparameters used for COBRA in Imagenet-1K}
\vspace{0.1in}
\resizebox{0.5\textwidth}{!}{%
\begin{tabular}{cccc}
\toprule
\textbf{Shots}               & \textbf{Name}         & \textbf{Value} & \textbf{Tuning Range}                        \\
\midrule
\multirow{3}{*}{1}  & k (neighbor) & 5     & \{5, 8, 16, 32, 64, 128, 256, 512\}  \Tstrut\\
                    & $\lambda$    & 1     & Fixed to 1                          \\
                    & $\mu$        & 0     & Fixed to 0                          \\
\midrule
\multirow{3}{*}{2}  & k (neighbor) & 512   & \{5, 8, 16, 32, 64, 128, 256, 512\} \\
                    & $\lambda$    & 1     & Fixed to 1                          \\
                    & $\mu$        & 0     & Fixed to 0                          \\
\midrule
\multirow{3}{*}{4}  & k (neighbor) & 32    & \{5, 8, 16, 32, 64, 128, 256, 512\} \\
                    & $\lambda$    & 1     & Fixed to 1                          \\
                    & $\mu$        & 0     & Fixed to 0                          \\
\midrule
\multirow{3}{*}{8}  & k (neighbor) & 512   & \{5, 8, 16, 32, 64, 128, 256, 512\} \\
                    & $\lambda$    & 1     & Fixed to 1                          \\
                    & $\mu$        & 0     & Fixed to 0                          \\
\midrule
\multirow{3}{*}{16} & k (neighbor) & 512   & \{5, 8, 16, 32, 64, 128, 256, 512\} \\
                    & $\lambda$    & 1     & Fixed to 1                          \\
                    & $\mu$        & 0     & Fixed to 0                          \\
\bottomrule
\end{tabular}}
\label{tab: hyperameters_cobra_in1k_vitb_16}
\end{table}

\begin{table}[h!]
\centering
\caption{Hyperparameters used for COBRA in FGVC-Aircrafts}
\vspace{0.1in}
\resizebox{0.5\textwidth}{!}{%
\begin{tabular}{cccc}
\toprule
\textbf{Shots}               & \textbf{Name}         & \textbf{Value} & \textbf{Tuning Range}                        \\
\midrule
\multirow{3}{*}{1}  & k (neighbor) & 256     & \{5, 8, 16, 32, 64, 128, 256, 512\}  \Tstrut\\
                    & $\lambda$    & 1     & Fixed to 1                          \\
                    & $\mu$        & 0     & Fixed to 0                          \\
\midrule
\multirow{3}{*}{2}  & k (neighbor) & 128   & \{5, 8, 16, 32, 64, 128, 256, 512\} \\
                    & $\lambda$    & 1     & Fixed to 1                          \\
                    & $\mu$        & 0     & Fixed to 0                          \\
\midrule
\multirow{3}{*}{4}  & k (neighbor) & 128    & \{5, 8, 16, 32, 64, 128, 256, 512\} \\
                    & $\lambda$    & 1     & Fixed to 1                          \\
                    & $\mu$        & 0     & Fixed to 0                          \\
\midrule
\multirow{3}{*}{8}  & k (neighbor) & 512   & \{5, 8, 16, 32, 64, 128, 256, 512\} \\
                    & $\lambda$    & 1     & Fixed to 1                          \\
                    & $\mu$        & 0     & Fixed to 0                          \\
\midrule
\multirow{3}{*}{16} & k (neighbor) & 256   & \{5, 8, 16, 32, 64, 128, 256, 512\} \\
                    & $\lambda$    & 1     & Fixed to 1                          \\
                    & $\mu$        & 0     & Fixed to 0                          \\
\bottomrule
\end{tabular}}
\label{tab: hyperameters_cobra_fgvc_vitb_16}
\end{table}

\begin{table}[h!]
\centering
\caption{Hyperparameters used for COBRA in Flowers-102.}
\vspace{0.1in}
\resizebox{0.5\textwidth}{!}{%
\begin{tabular}{cccc}
\toprule
\textbf{Shots}               & \textbf{Name}         & \textbf{Value} & \textbf{Tuning Range}                        \\
\midrule
\multirow{3}{*}{1}  & k (neighbor) &    128  & \{5, 8, 16, 32, 64, 128, 256, 512\}  \Tstrut\\
                    & $\lambda$    & 1     & Fixed to 1                          \\
                    & $\mu$        &  0.2    & \{0.1, 0.2, 0.5\}                           \\
\midrule
\multirow{3}{*}{2}  & k (neighbor) &  128 & \{5, 8, 16, 32, 64, 128, 256, 512\} \\
                    & $\lambda$    & 1     & Fixed to 1                          \\
                    & $\mu$        &    0.2  &   \{0.1, 0.2, 0.5\}                          \\
\midrule
\multirow{3}{*}{4}  & k (neighbor) &   16  & \{5, 8, 16, 32, 64, 128, 256, 512\} \\
                    & $\lambda$    & 1     & Fixed to 1                          \\
                    & $\mu$        &    0.5  &    \{0.1, 0.2, 0.5\}                         \\
\midrule
\multirow{3}{*}{8}  & k (neighbor) &   8 & \{5, 8, 16, 32, 64, 128, 256, 512\} \\
                    & $\lambda$    & 1     & Fixed to 1                          \\
                    & $\mu$        &   0.5   &    \{0.1, 0.2, 0.5\}                         \\
\midrule
\multirow{3}{*}{16} & k (neighbor) &   16 & \{5, 8, 16, 32, 64, 128, 256, 512\} \\
                    & $\lambda$    & 1     & Fixed to 1                          \\
                    & $\mu$        &    0.5  &  \{0.1, 0.2, 0.5\}                           \\
\bottomrule
\end{tabular}}
\label{tab: hyperameters_cobra_flowers_vitb_16}
\vspace{0.1in}
\end{table}

\begin{table}[h!]
\centering
\caption{Hyperparameters used for COBRA in DTD.}
\vspace{0.1in}
\resizebox{0.5\textwidth}{!}{%
\begin{tabular}{cccc}
\toprule
\textbf{Shots}               & \textbf{Name}         & \textbf{Value} & \textbf{Tuning Range}                        \\
\midrule
\multirow{3}{*}{1}  & k (neighbor) &    128  & \{5, 8, 16, 32, 64, 128, 256, 512\}  \Tstrut\\
                    & $\lambda$    & 1     & Fixed to 1                          \\
                    & $\mu$        &   0.5   &    \{0.1, 0.2, 0.5\}                        \\
\midrule
\multirow{3}{*}{2}  & k (neighbor) &  64  & \{5, 8, 16, 32, 64, 128, 256, 512\} \\
                    & $\lambda$    & 1     & Fixed to 1                          \\
                    & $\mu$        &   0.5   & \{0.1, 0.2, 0.5\}                            \\
\midrule
\multirow{3}{*}{4}  & k (neighbor) &   64  & \{5, 8, 16, 32, 64, 128, 256, 512\} \\
                    & $\lambda$    & 1     & Fixed to 1                          \\
                    & $\mu$        &   0.5   & \{0.1, 0.2, 0.5\}                           \\
\midrule
\multirow{3}{*}{8}  & k (neighbor) &  128  & \{5, 8, 16, 32, 64, 128, 256, 512\} \\
                    & $\lambda$    & 1     & Fixed to 1                          \\
                    & $\mu$        &  0.5    &  \{0.1, 0.2, 0.5\}                           \\
\midrule
\multirow{3}{*}{16} & k (neighbor) &   16 & \{5, 8, 16, 32, 64, 128, 256, 512\} \\
                    & $\lambda$    & 1     & Fixed to 1                          \\
                    & $\mu$        &  0.5    &  \{0.1, 0.2, 0.5\}                          \\
\bottomrule
\end{tabular}}
\label{tab: hyperameters_cobra_dtd_vitb_16}
\end{table}

\begin{table}[h!]
\centering
\caption{Hyperparameters used for COBRA in Food101.}
\vspace{0.1in}
\resizebox{0.5\textwidth}{!}{%
\begin{tabular}{cccc}
\toprule
\textbf{Shots}               & \textbf{Name}         & \textbf{Value} & \textbf{Tuning Range}                        \\
\midrule
\multirow{3}{*}{1}  & k (neighbor) &    256  & \{5, 8, 16, 32, 64, 128, 256, 512\}  \Tstrut\\
                    & $\lambda$    & 1     & Fixed to 1                          \\
                    & $\mu$        &   0   &   \{0, 0.2, 0.5\}                      \\
\midrule
\multirow{3}{*}{2}  & k (neighbor) &  512  & \{5, 8, 16, 32, 64, 128, 256, 512\} \\
                    & $\lambda$    & 1     & Fixed to 1                          \\
                    & $\mu$        &   0   & \{0, 0.2, 0.5\}                              \\
\midrule
\multirow{3}{*}{4}  & k (neighbor) &   32  & \{5, 8, 16, 32, 64, 128, 256, 512\} \\
                    & $\lambda$    & 1     & Fixed to 1                          \\
                    & $\mu$        &   0   & \{0, 0.2, 0.5\}                                \\
\midrule
\multirow{3}{*}{8}  & k (neighbor) &  8  & \{5, 8, 16, 32, 64, 128, 256, 512\} \\
                    & $\lambda$    & 1     & Fixed to 1                          \\
                    & $\mu$        &  0.5    &  \{0, 0.2, 0.5\}                                  \\
\midrule
\multirow{3}{*}{16} & k (neighbor) &  8 & \{5, 8, 16, 32, 64, 128, 256, 512\} \\
                    & $\lambda$    & 1     & Fixed to 1                          \\
                    & $\mu$        &  0.5    &  \{0, 0.2, 0.5\}                            \\
\bottomrule
\end{tabular}}
\label{tab: hyperameters_cobra_food_vitb_16}
\end{table}

\begin{table}[h!]
\centering
\caption{Hyperparameters used for COBRA for backbone ablation. All the remaining ones are kept the same as ViT-B/16}
\vspace{-0.1in}
\resizebox{0.5\textwidth}{!}{%
\begin{tabular}{cccc}
\toprule
\textbf{Shots}               & \textbf{Backbone} & $k$ \textbf{(neighbors)} & \textbf{Tuning Range}                        \\
\midrule
\multirow{3}{*}{16} & ViT-B/32 & 8             & \{5, 8, 16, 32, 64, 128, 256, 512\} \Tstrut \\
                    & ViT-B/16 & 5             & \{5, 8, 16, 32, 64, 128, 256, 512\} \\
                    & ViT-L/14 & 16            & \{5, 8, 16, 32, 64, 128, 256, 512\} \\
\bottomrule
\end{tabular}}
\label{tab: hyperameters_backbone_k}
\vspace{0.1in}
\end{table}

\begin{table}[h!]
\small
\centering
\caption{Hyperparameters used for ICL Experiments}
\vspace{-0.1in}
\resizebox{0.5\textwidth}{!}{%
\begin{tabular}{cccc}
\toprule
\textbf{Dataset}               & \textbf{Name} & \textbf{Value} & \textbf{Tuning Range}                        \\
\midrule
\multirow{3}{*}{MRPC} & k (neighbor)  & 256 & \{256 512 1024 2047\} \Tstrut \\
                    & $\mu$   & 1e-4  & \{0, 1e-4, 0.1\} \\
                    &  $\lambda$  & 0.1 & \{0, 1e-2, 0.1\} \\
\midrule
\multirow{3}{*}{SST-2} & k (neighbor) & 2047  & \{256 512 1024 2047\} \Tstrut \\
                    & $\mu$   & 0    & \{0, 1e-4, 0.1\} \\
                    &  $\lambda$ & 0  & \{0, 1e-2, 0.1\} \\
\midrule
\multirow{3}{*}{RTE} & k (neighbor)  & 512 & \{256 512 1024 2047\} \Tstrut \\
                    & $\mu$    & 0   & \{0, 1e-4, 0.1\} \\
                    &  $\lambda$  & 1e-2 & \{0, 1e-2, 0.1\} \\
\midrule
\multirow{3}{*}{TREC} & k (neighbor) & 2047 & \{256 512 1024 2047\} \Tstrut \\
                    & $\mu$  & 0     & \{0, 1e-4, 0.1\} \\
                    &  $\lambda$ & 0 & \{0, 1e-2, 0.1\} \\
\bottomrule
\end{tabular}}
\label{tab: hyperameters_icl}
\vspace{0.1in}
\end{table}

\section{Sensitivity Analysis}
\label{appen sec: sensitivity}
The hyperparameter associated with COBRA is the sparsity of the similarity matrix, which is governed by the top-$k$ entries in each row. To this end, for our first sensitivity analysis, we study the downstream accuracy for the 16-shot case with Imagenet as $\Dtar$ and CLIP-adapter (ViT-L/14 backbone) being the few-shot learning method. We vary $k \in \{5, 8, 16, 32, 64, 128, 256, 512\}$, the range we perform tuning and report its accuracy in \Cref{fig:sensitivity_k_vit}. We observe that the COBRA works well when $k$ is neither too small nor too large, which may be linked to the saturation of the submodular function, critical for its use. However, even in the extreme ranges, it doesn't underperform the baselines. We leave an in-depth investigation of saturation and its impact on FLMI/COBRA as a future work.  

As we outlined in the main paper, $\mu$ that provides a quality score function is optional, and indeed for datasets such as Imagenet that has a diverse set of images in the test set, a quality score is not needed. Moreover, on extreme ends $\mu=1$ just recovers \emph{Sim-score}. Therefore, for the 16-shot case with Imagenet as $\Dtar$ and CLIP-adapter (ViT-B/16 backbone), we consider varying $\mu \in \{0, 0.2, 0.4, 0.6, 0.8, 1.0\}$, and study its performance in~\cref{fig:sensitivity_mu_vit}.

\begin{figure}[h]
    \centering
    \begin{subfigure}{0.45\textwidth}
        \centering
        \includegraphics[scale=0.5]{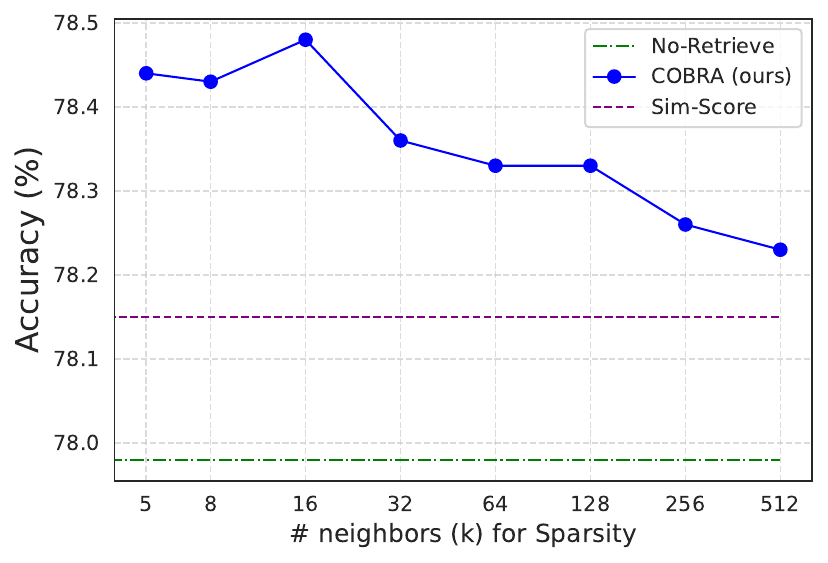}
        \caption{\textbf{Hyperparameter Sensitivity $k$:} \small We see that even with varying the $k$ in a broad range, COBRA never underperforms the baselines, making it robust to hyperparameter changes.}
        \label{fig:sensitivity_k_vit}
    \end{subfigure}
    \hfill 
    \begin{subfigure}{0.45\textwidth}
        \centering
        \includegraphics[scale=0.45]{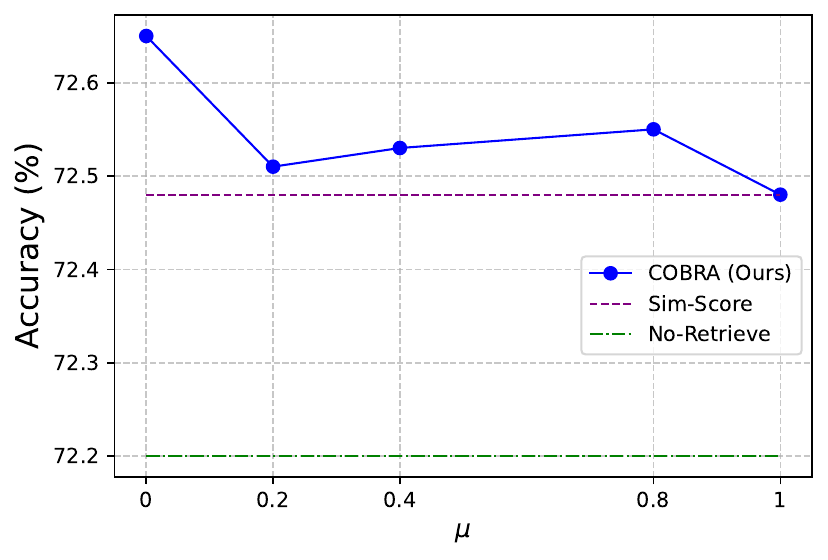}
        \caption{\textbf{Hyperparameter Sensitivity $\mu$:} \small We see that increasing weightage on quality term (which is done using Sim-Score) interpolates between COBRA without quality to Sim-score. For Imagenet where there is diversity in both training and test sets, having additional quality terms hampers the performance. However, for datasets such as Flowers-102, where there is a lack of diversity intrinsically, having this quality term makes sure we are not picking noisy samples (albeit optional).}
        \label{fig:sensitivity_mu_vit}
    \end{subfigure}
    \caption{Hyperparameter sensitivity analysis of COBRA for $k$ (sparsity) and $\mu$ (optional quality weightage).}
    \label{fig:sensitivity}
\end{figure}

\clearpage
\newpage
\section{More Qualitative Results}
\label{appen sec: more qualitative}
\vspace{-.1in}
\subsection{Imagenet Results}
\vspace{-.1in}
\begin{figure*}[htp!]
    \centering
    \begin{subfigure}{.28\textwidth}
        \includegraphics[width=\textwidth]{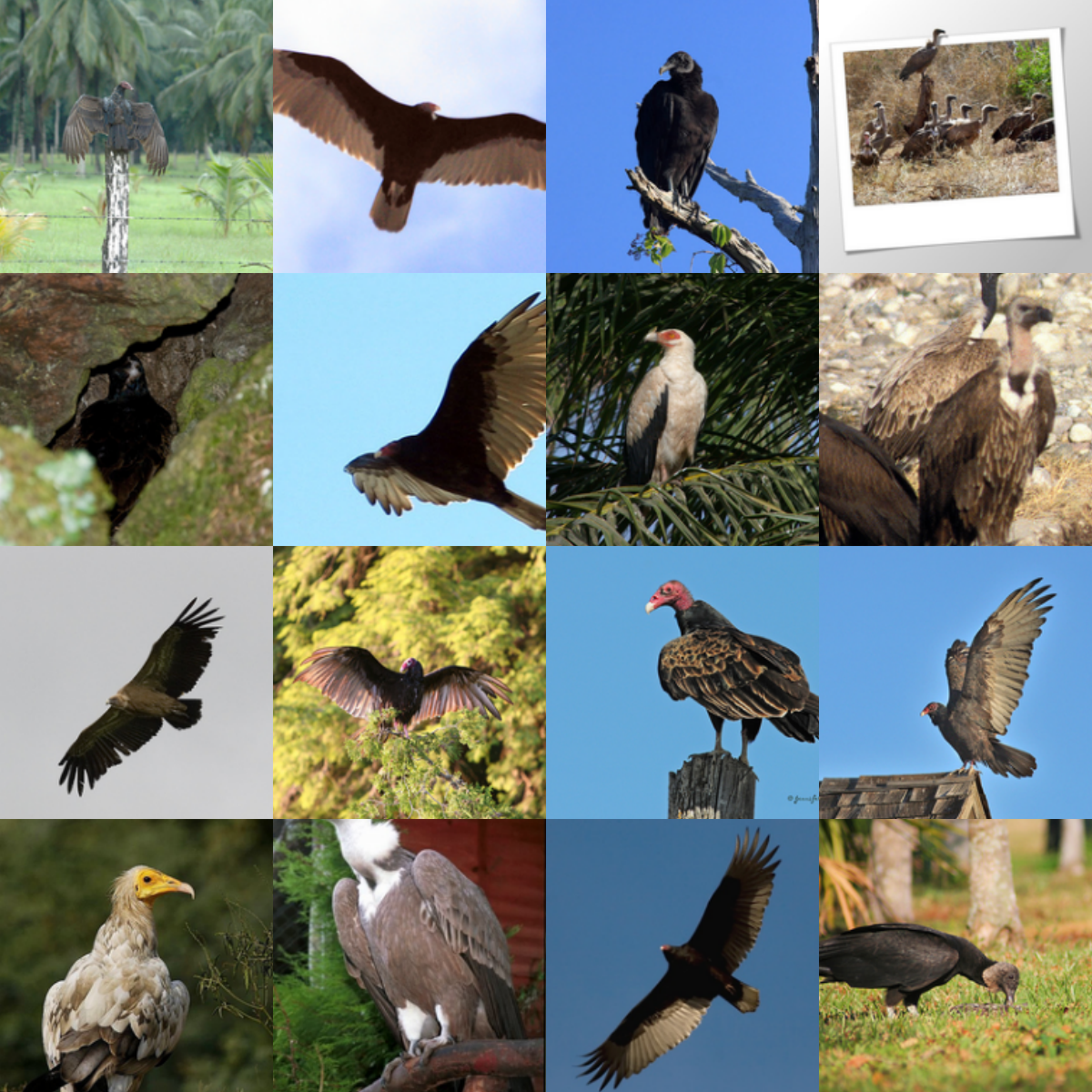}
        \caption{Target Images}
    \end{subfigure}\quad
    \begin{subfigure}{.28\textwidth}
        \includegraphics[width=\textwidth]{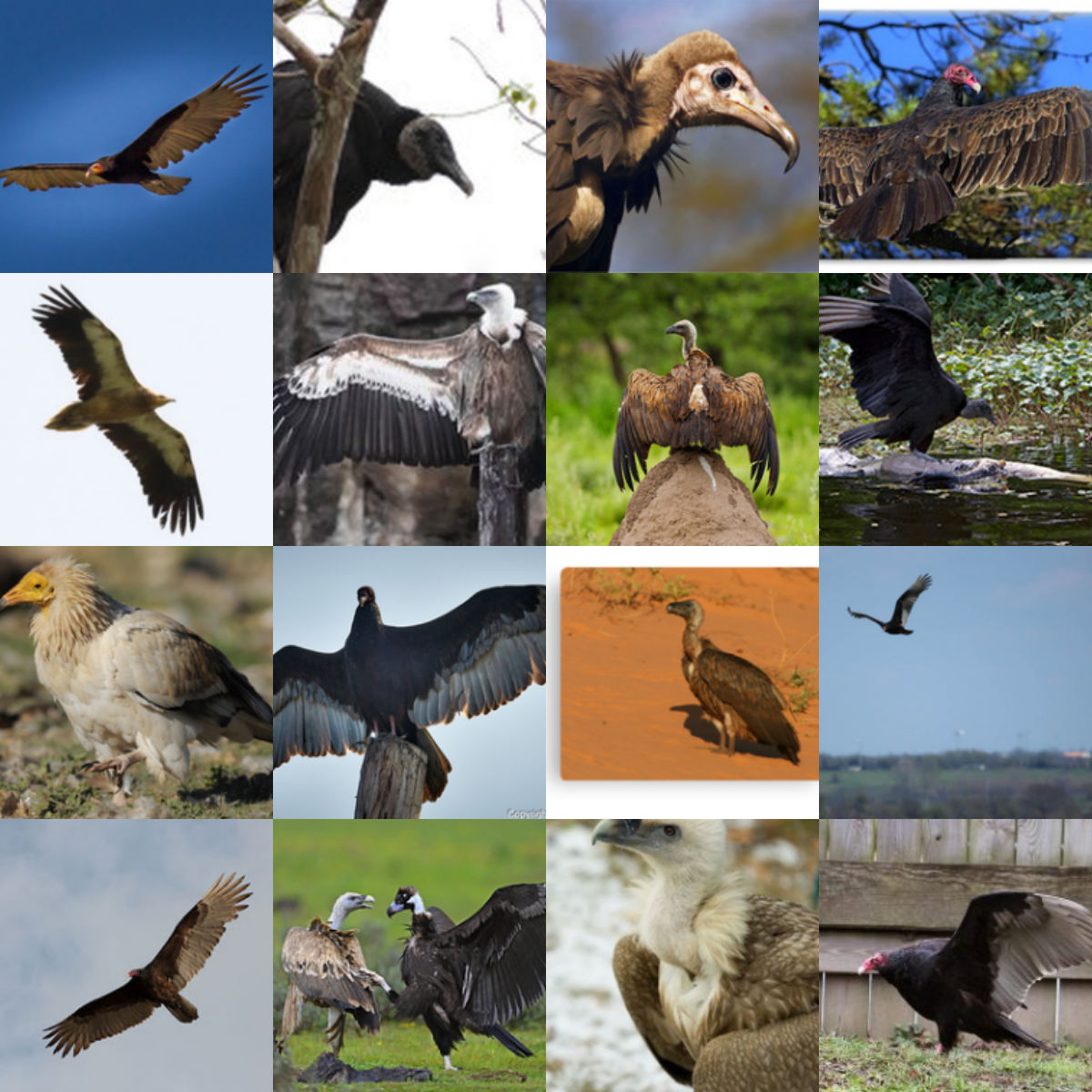}
        \caption{COBRA Images}
    \end{subfigure}\quad
    \begin{subfigure}{.28\textwidth}
        \includegraphics[width=\textwidth]{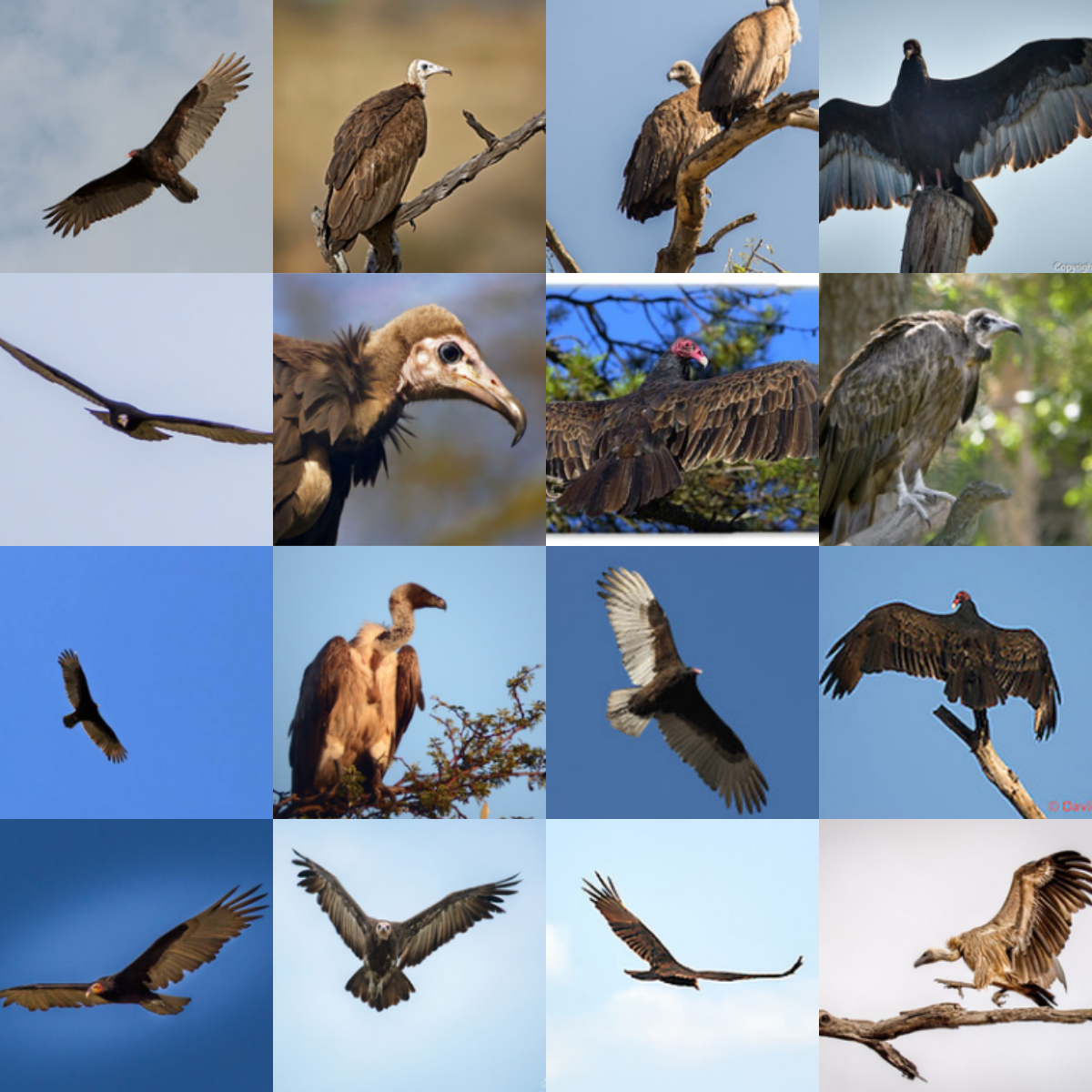}
        \caption{Sim-Score Images}
    \end{subfigure}
    
    
    \begin{subfigure}{.28\textwidth}
        \includegraphics[width=\textwidth]{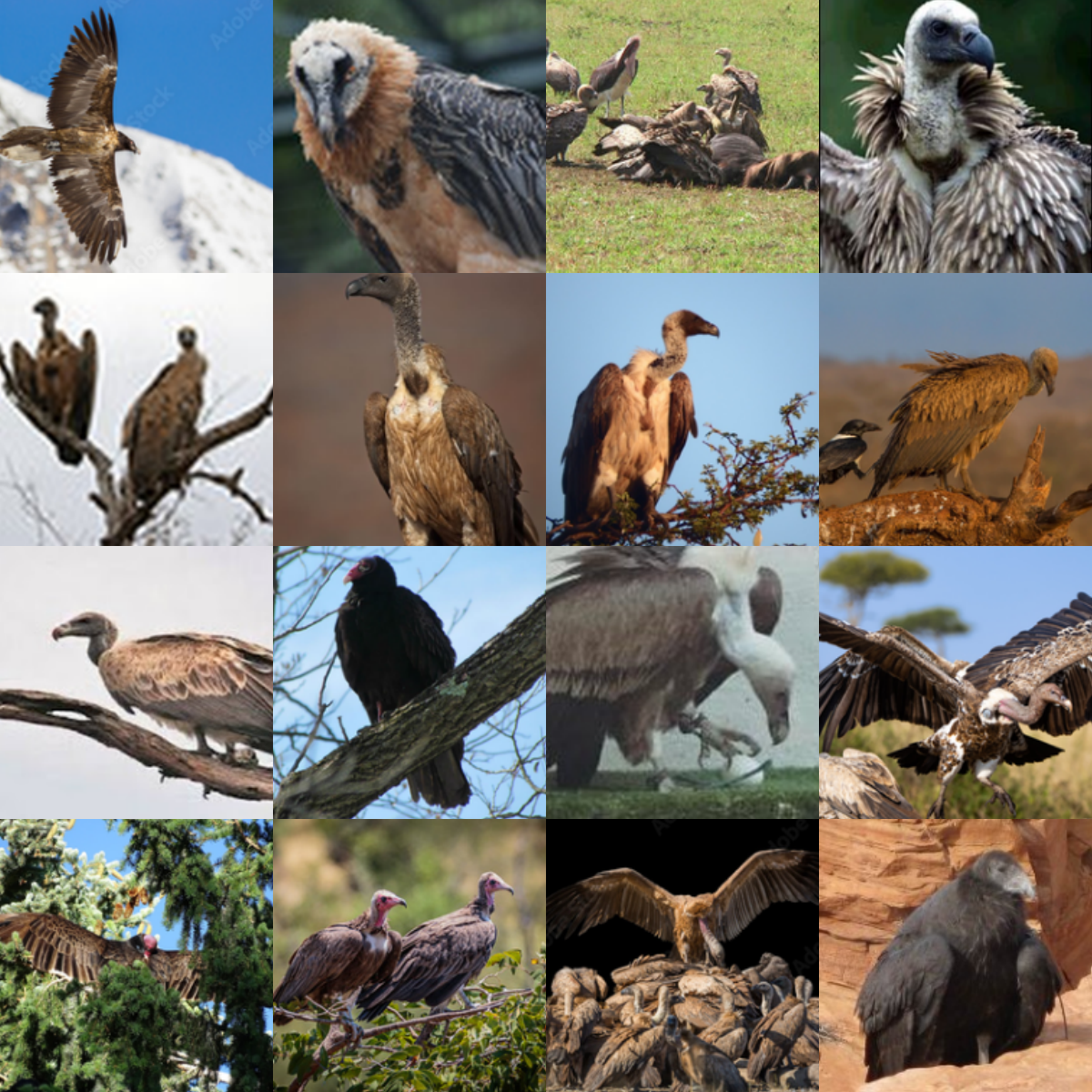}
        \caption{CLIP-Score Images}
    \end{subfigure}\quad
    \begin{subfigure}{.28\textwidth}
        \includegraphics[width=\textwidth]{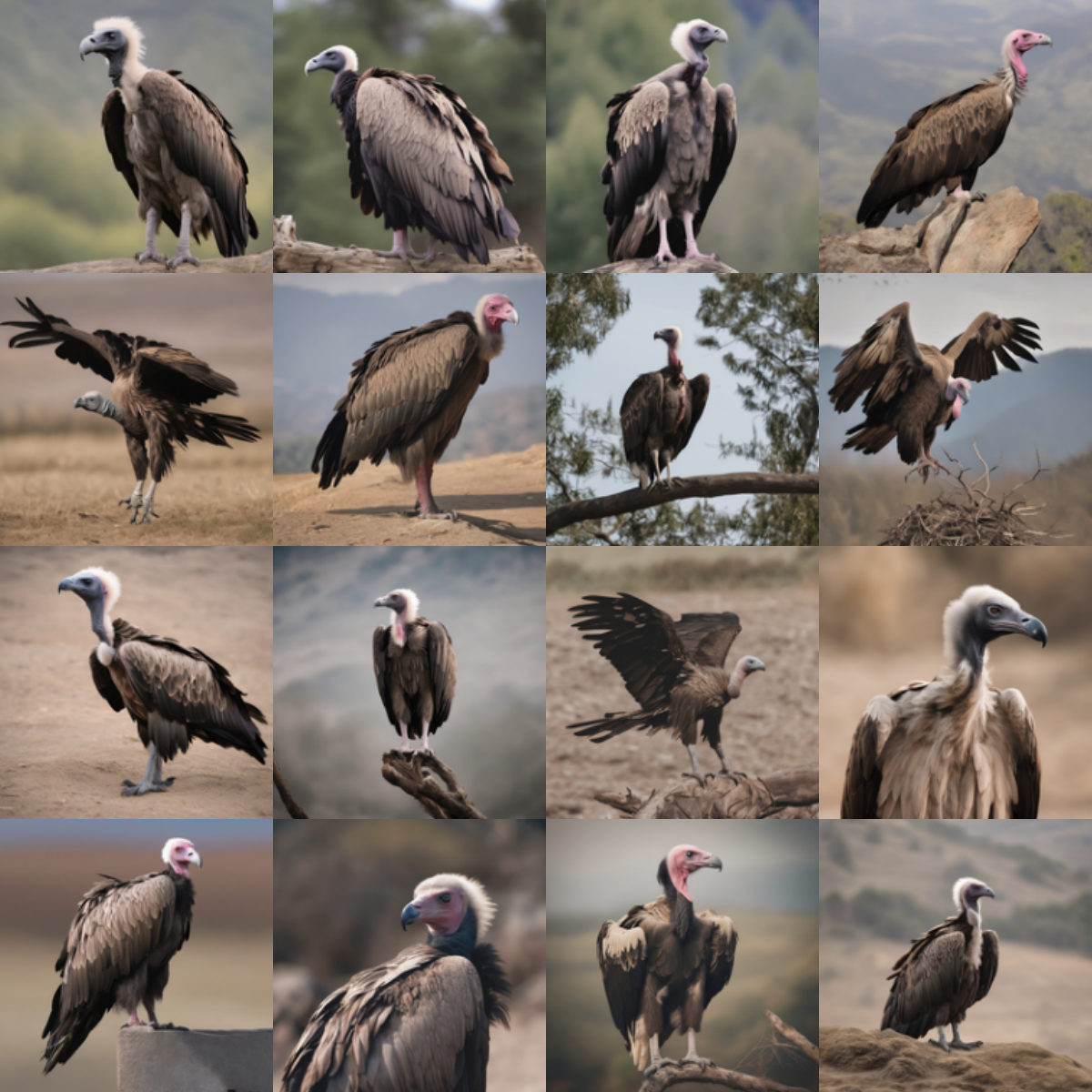}
        \caption{SDXL Images}
    \end{subfigure}\quad
    \begin{subfigure}{.28\textwidth}
        \includegraphics[width=\textwidth]{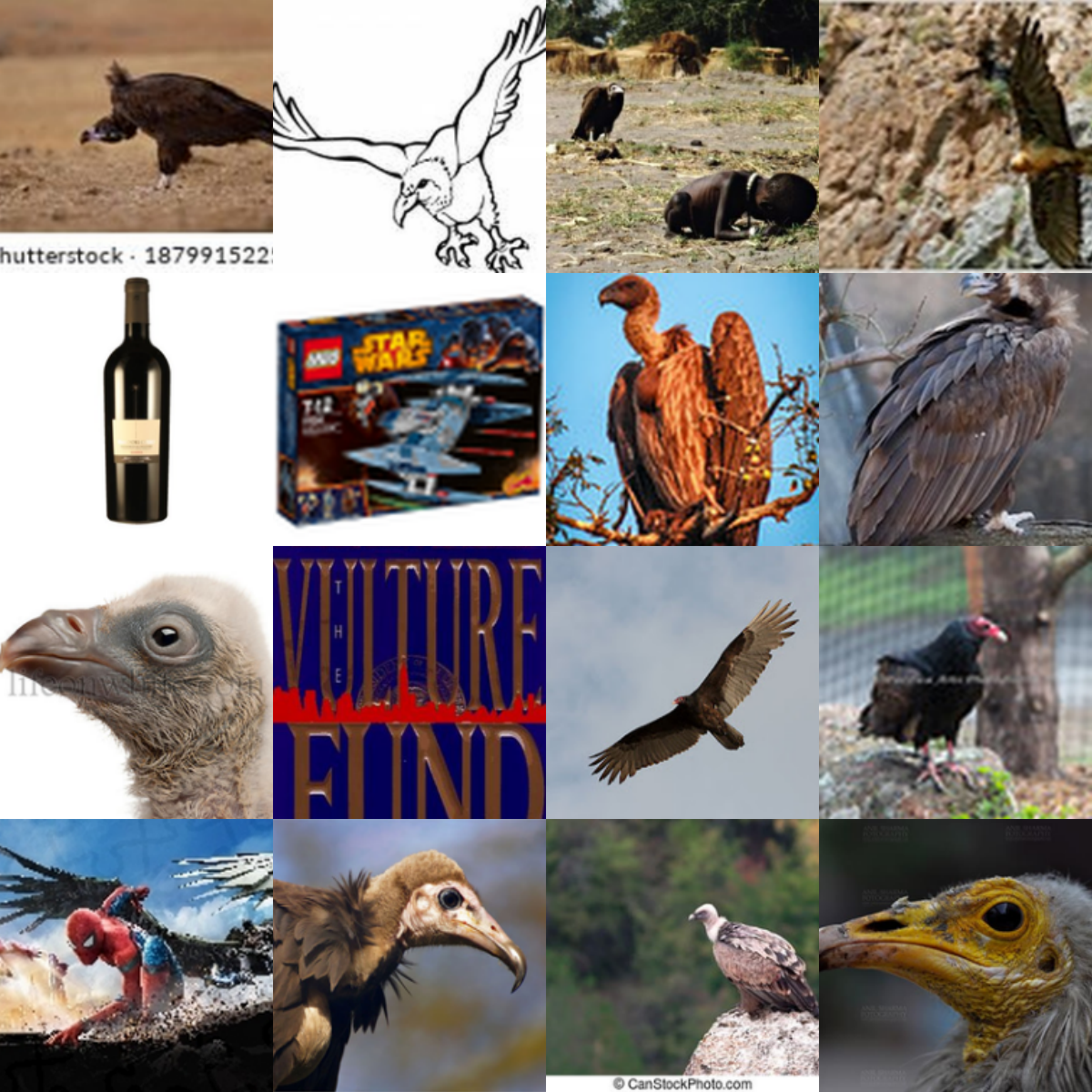}
        \caption{Random Images}
    \end{subfigure}
    \vspace{-.1in}
    \caption{\small \textbf{Visual Comparison of Retrieval Methods for Vultures (Imagenet)}.} 
    \label{fig:qualitative_results_vultures}
    \vspace{-.5in}
\end{figure*}
\begin{figure*}[htp!]
    \centering
    \begin{subfigure}{.3\textwidth}
        \includegraphics[width=\textwidth]{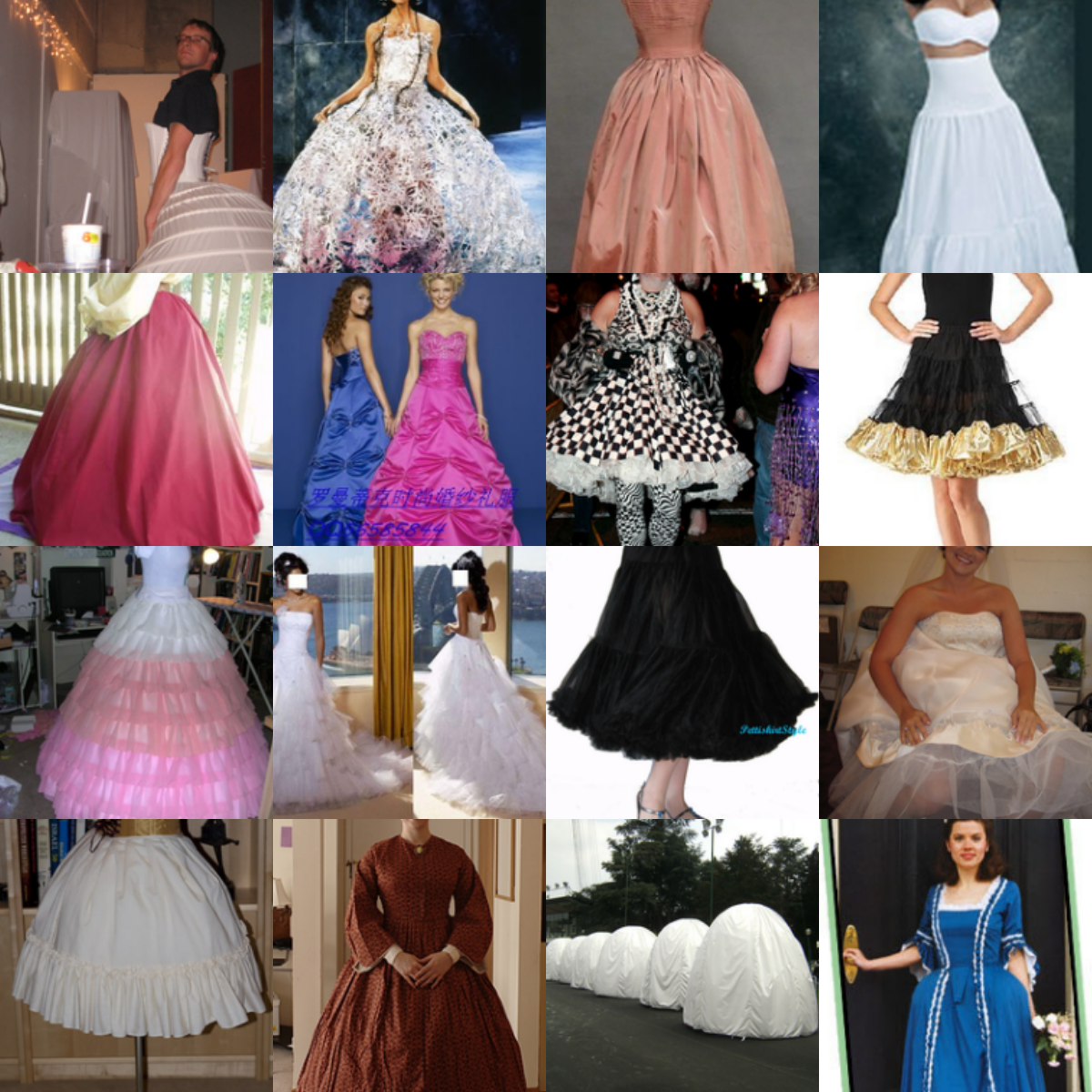}
        \caption{Target Images}
    \end{subfigure}\quad
    \begin{subfigure}{.3\textwidth}
        \includegraphics[width=\textwidth]{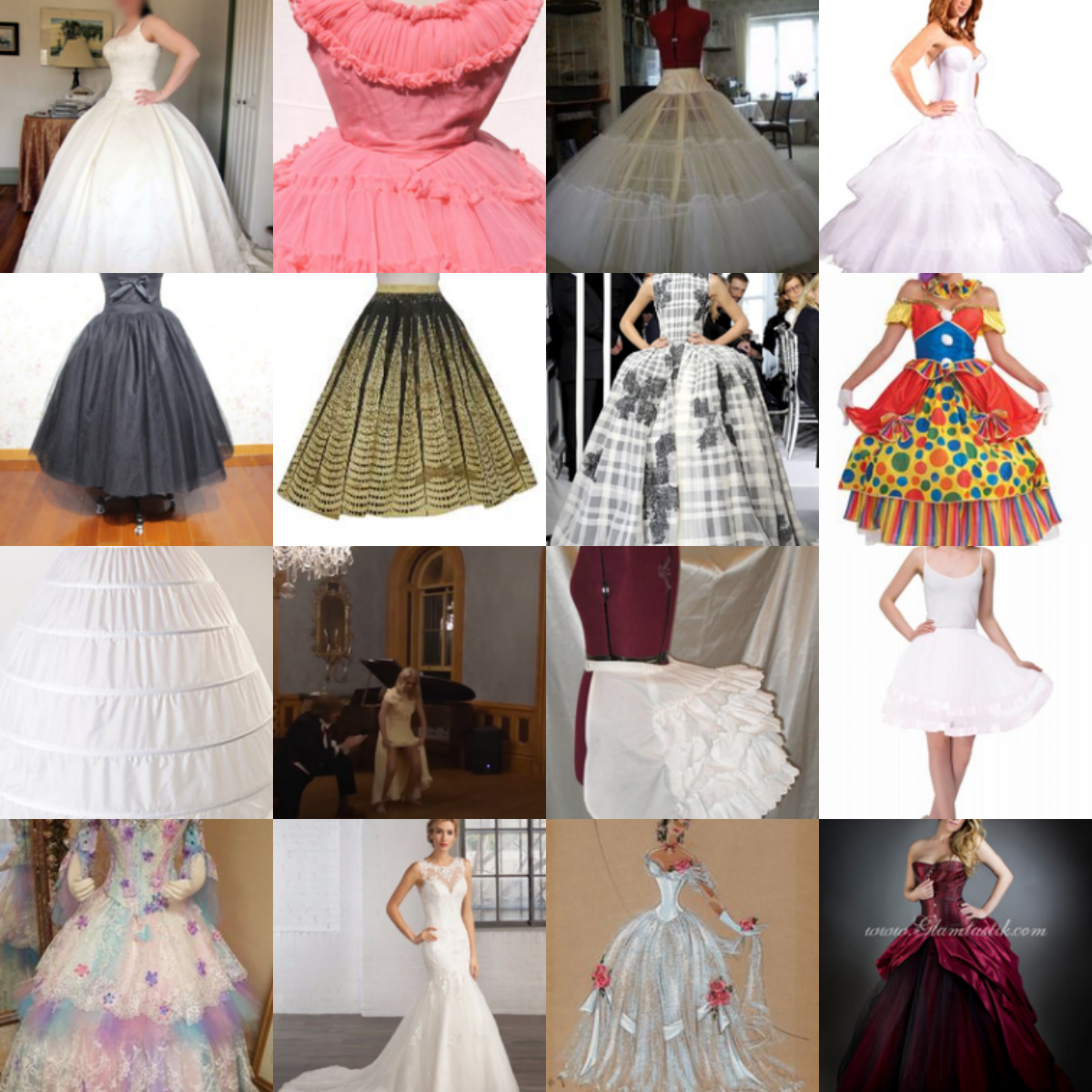}
        \caption{COBRA Images}
    \end{subfigure}\quad
    \begin{subfigure}{.3\textwidth}
        \includegraphics[width=\textwidth]{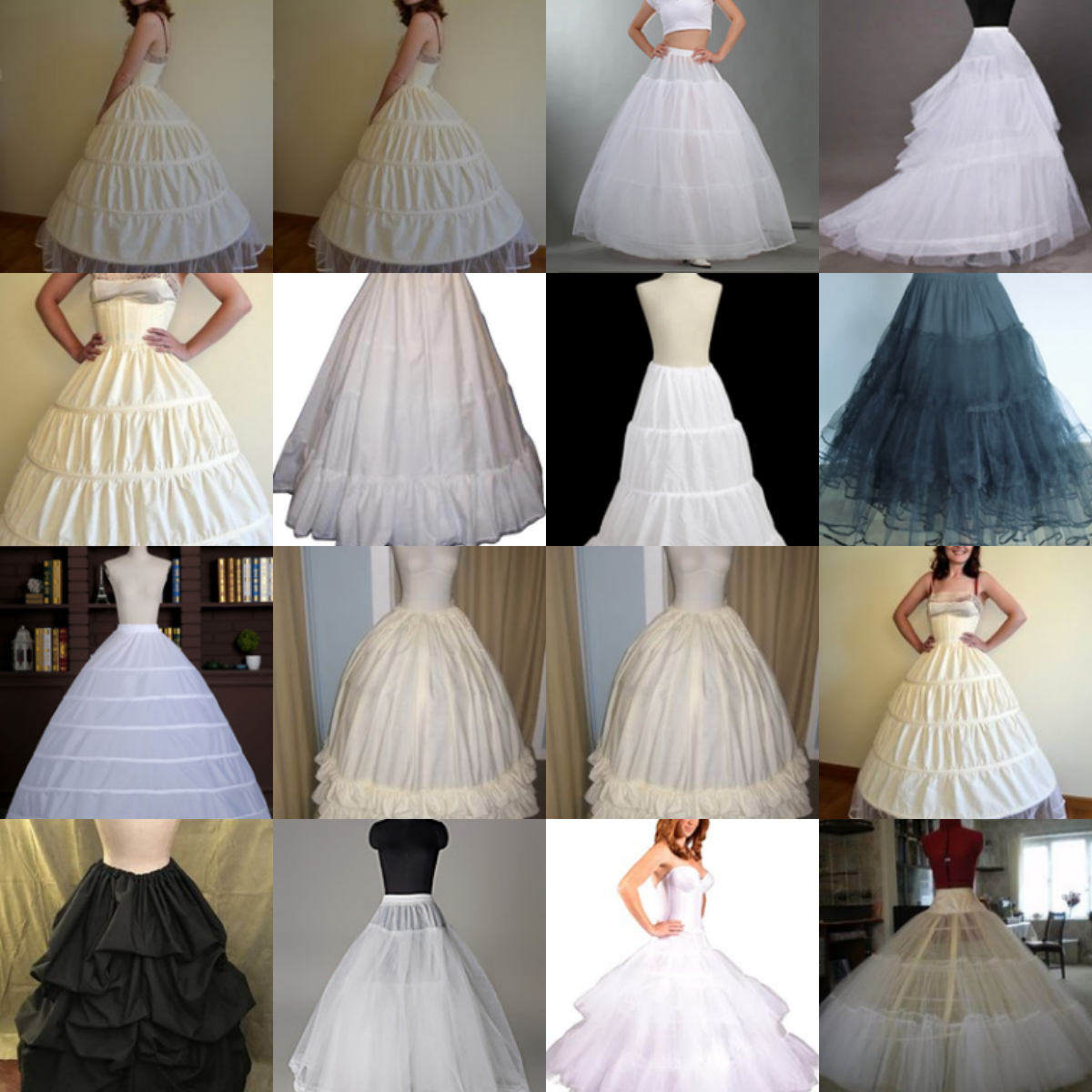}
        \caption{Sim-Score Images}
    \end{subfigure}
    
    \vspace{0.1cm}
    
    \begin{subfigure}{.3\textwidth}
        \includegraphics[width=\textwidth]{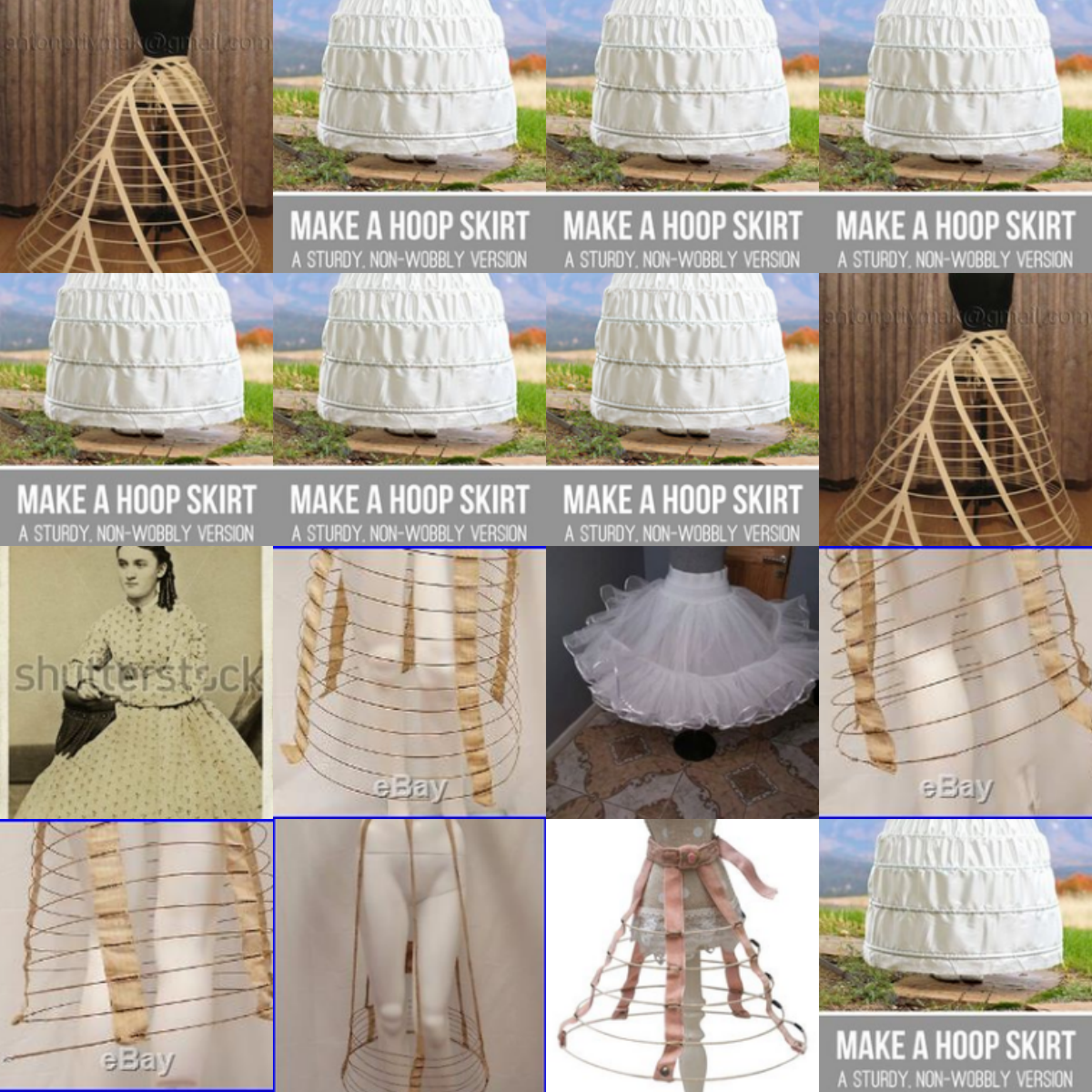}
        \caption{CLIP-Score Images}
    \end{subfigure}\quad
    \begin{subfigure}{.3\textwidth}
        \includegraphics[width=\textwidth]{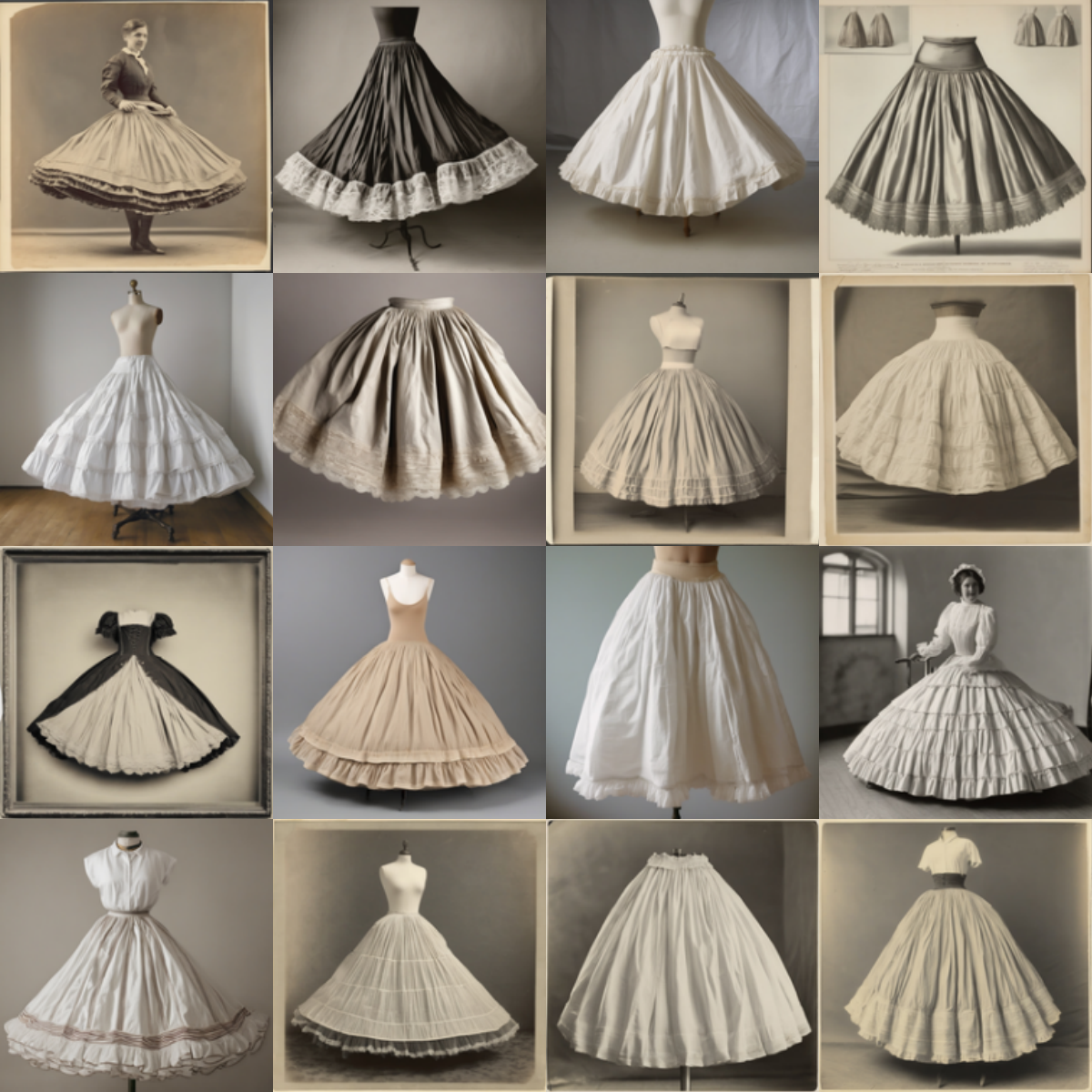}
        \caption{SDXL Images}
    \end{subfigure}\quad
    \begin{subfigure}{.3\textwidth}
        \includegraphics[width=\textwidth]{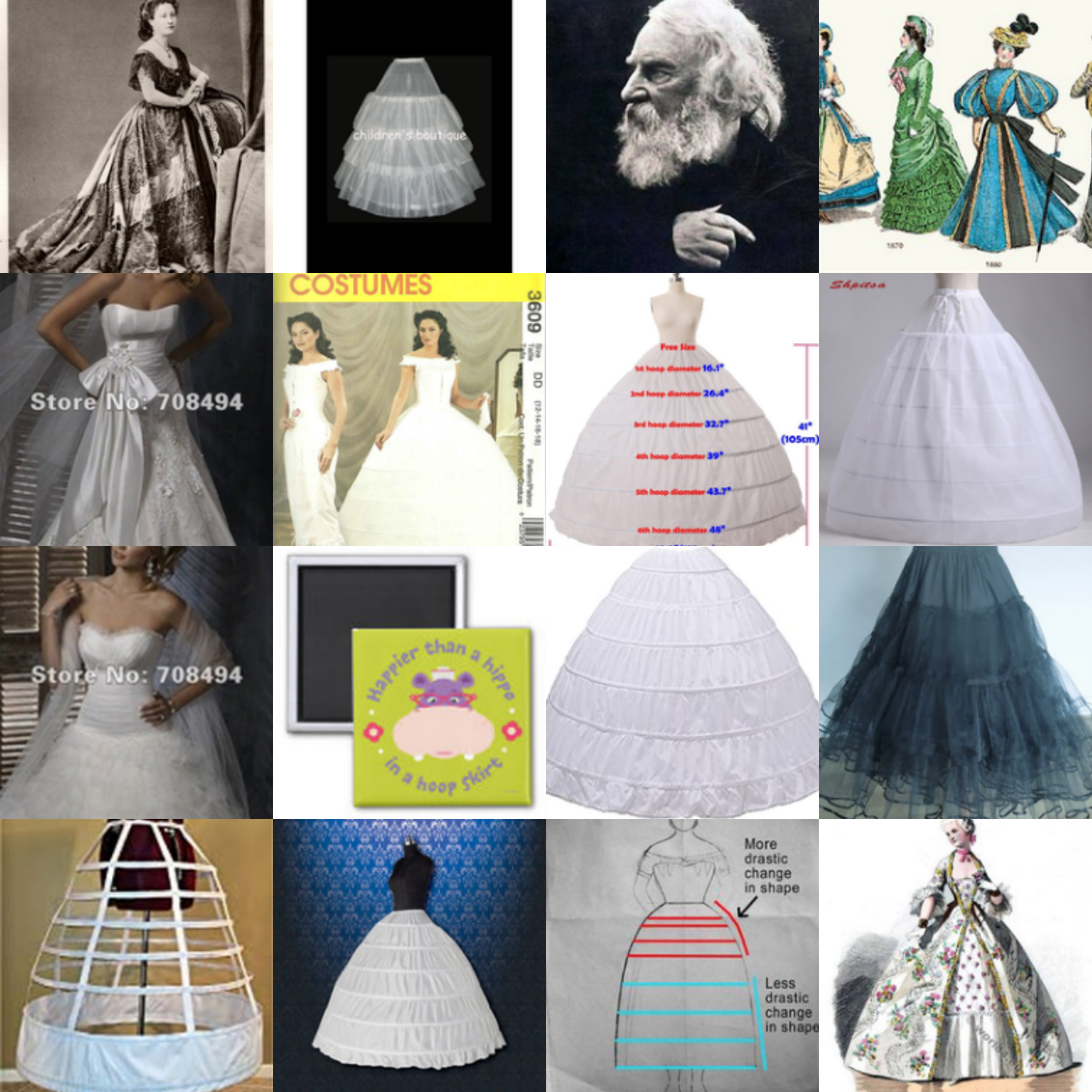}
        \caption{Random Images}
    \end{subfigure}
    
    \caption{\small \textbf{Visual Comparison of Retrieval Methods for Hoop Skirt (Imagenet)}.} 
    \label{fig:qualitative_results_hoopskirt}
    \vspace{-.1in}
\end{figure*}

\begin{figure*}[htp!]
    \centering
    \begin{subfigure}{.3\textwidth}
        \includegraphics[width=\textwidth]{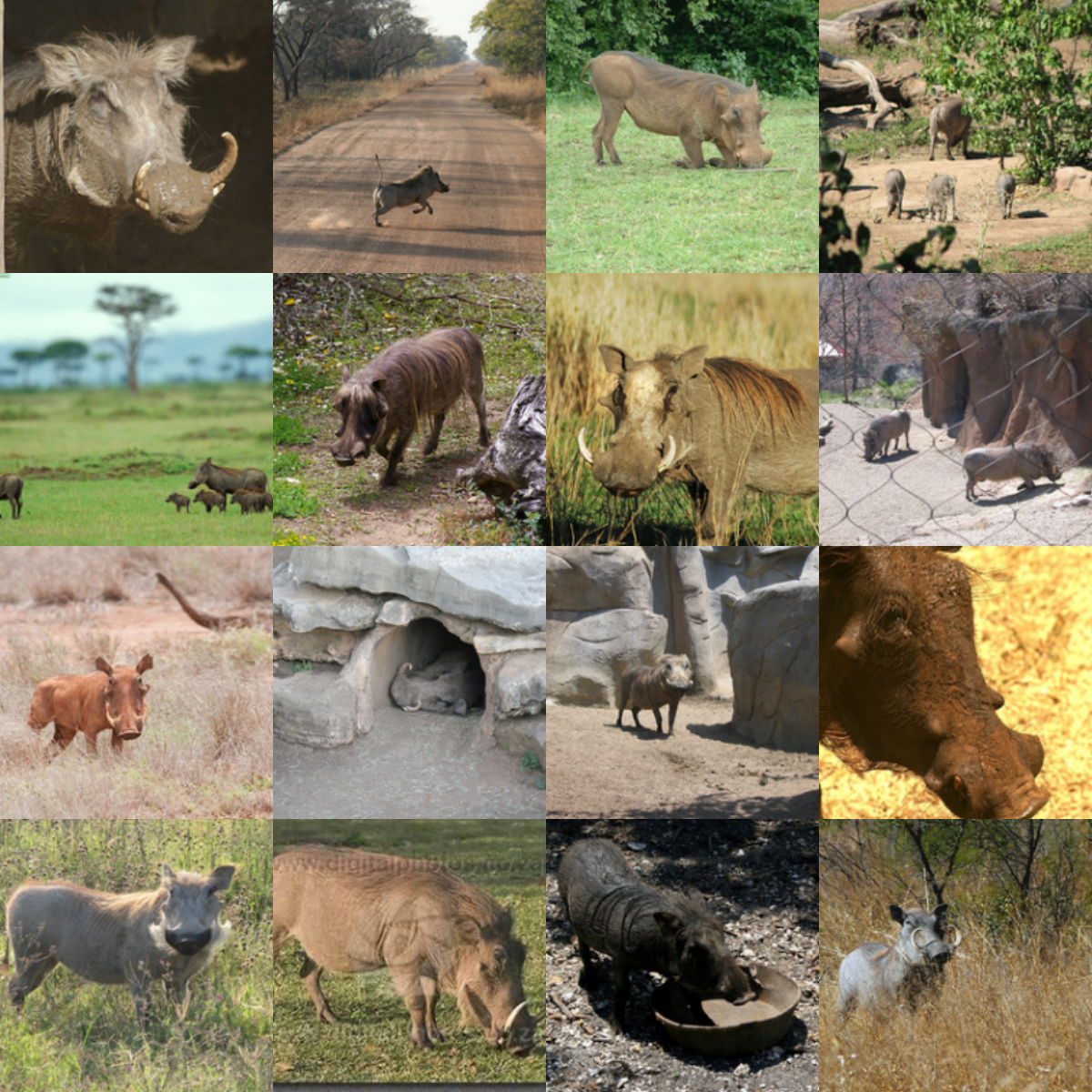}
        \caption{Target Images}
    \end{subfigure}\quad
    \begin{subfigure}{.3\textwidth}
        \includegraphics[width=\textwidth]{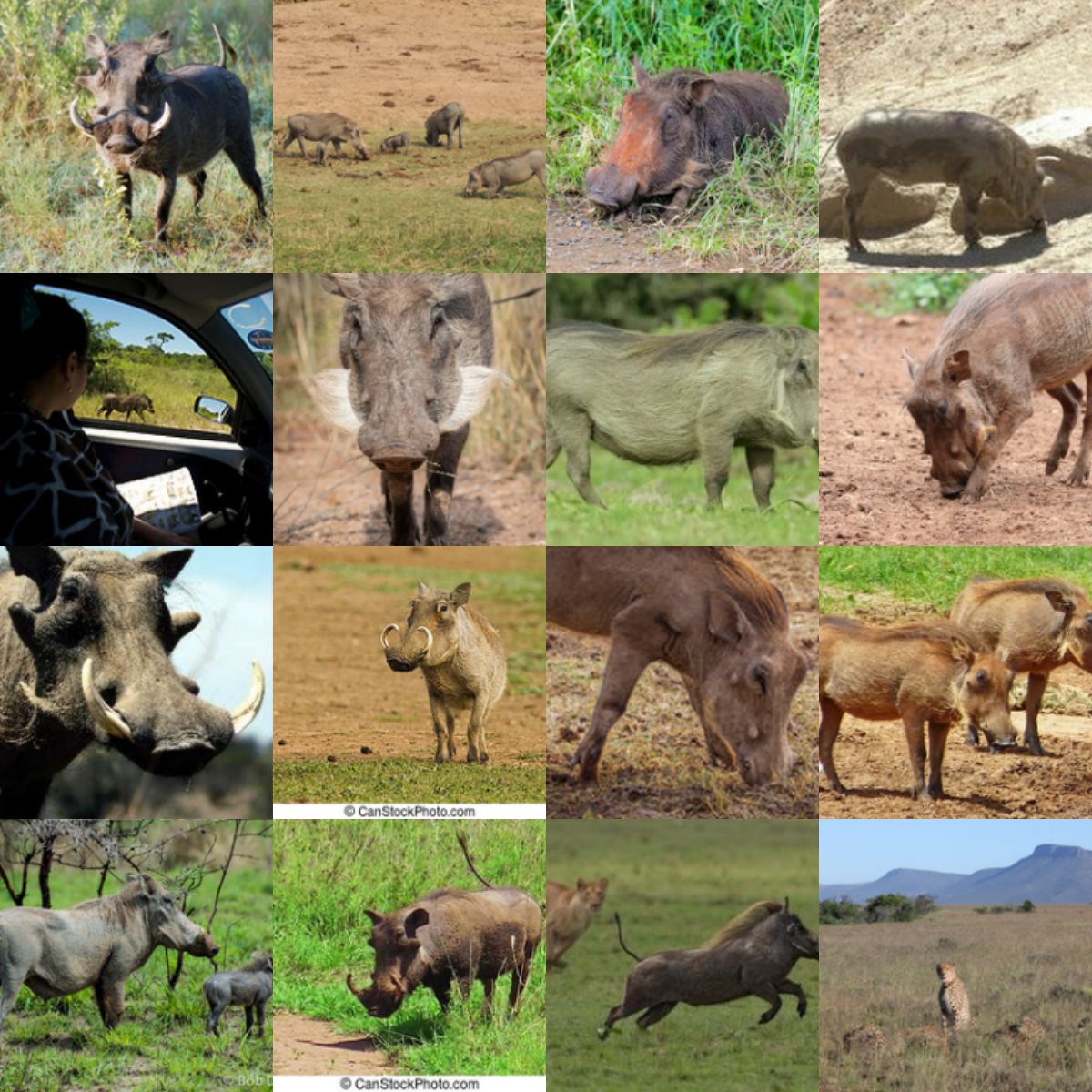}
        \caption{COBRA Images}
    \end{subfigure}\quad
    \begin{subfigure}{.3\textwidth}
        \includegraphics[width=\textwidth]{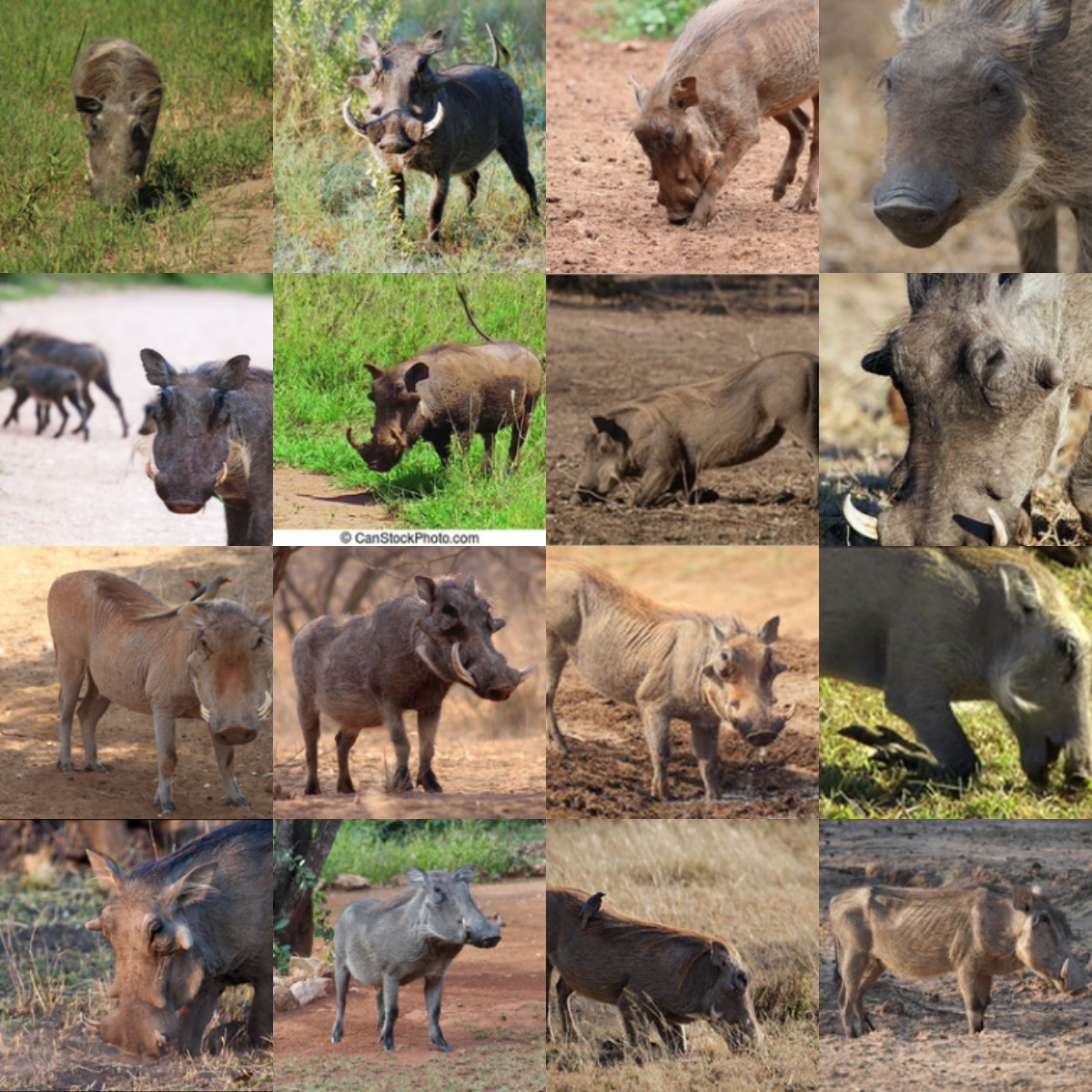}
        \caption{Sim-Score Images}
    \end{subfigure}
    
    \vspace{0.1cm}
    
    \begin{subfigure}{.3\textwidth}
        \includegraphics[width=\textwidth]{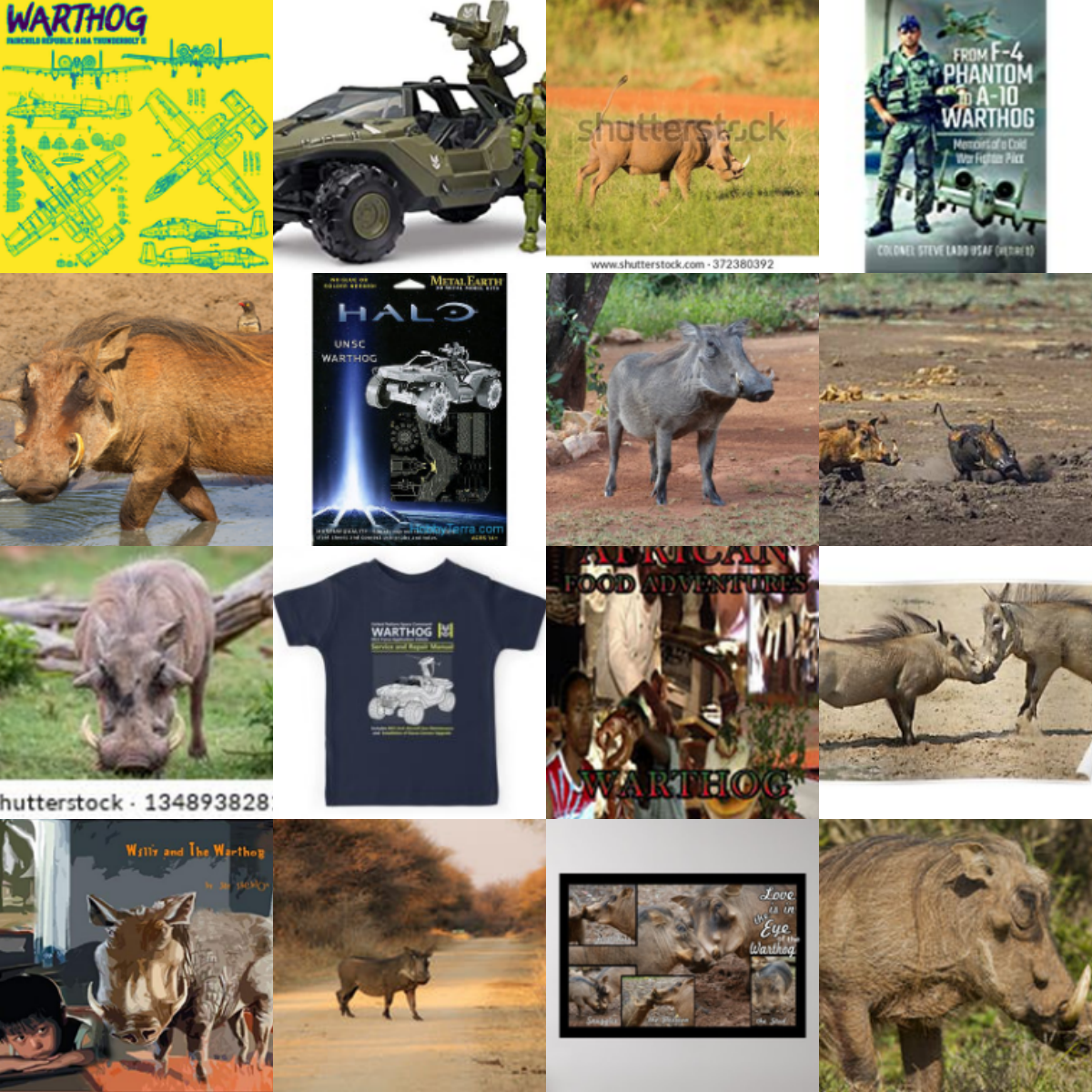}
        \caption{CLIP-Score Images}
    \end{subfigure}\quad
    \begin{subfigure}{.3\textwidth}
        \includegraphics[width=\textwidth]{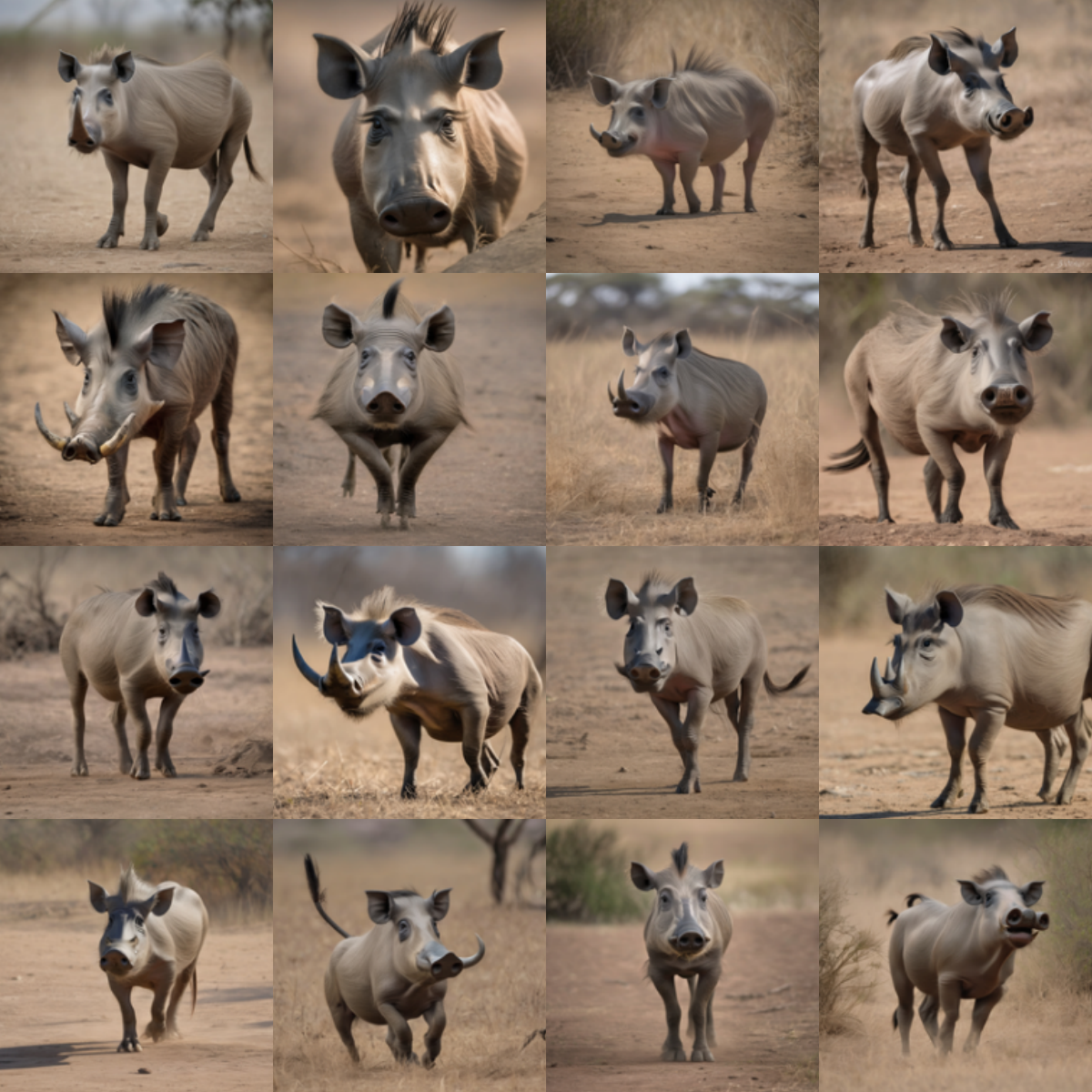}
        \caption{SDXL Images}
    \end{subfigure}\quad
    \begin{subfigure}{.3\textwidth}
        \includegraphics[width=\textwidth]{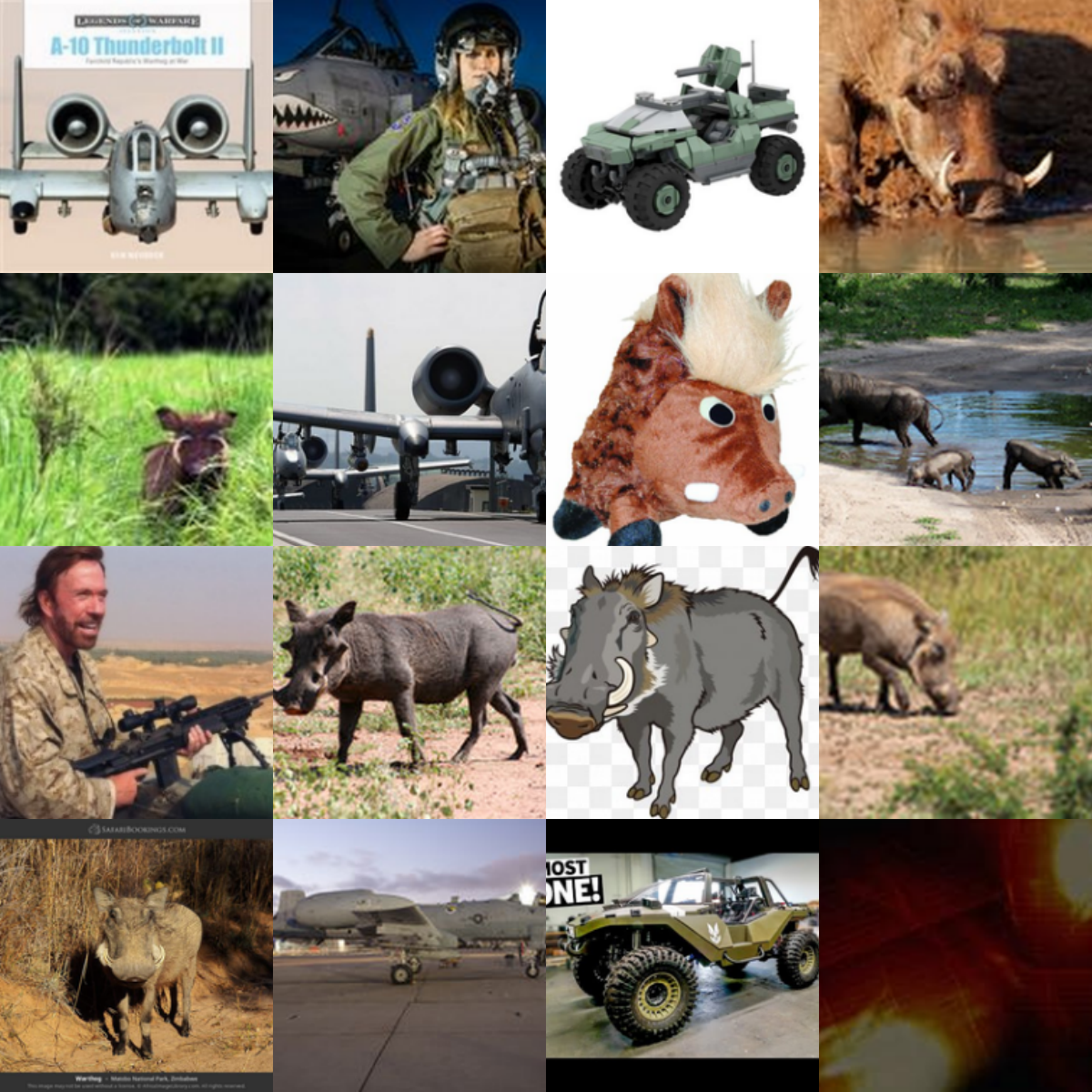}
        \caption{Random Images}
    \end{subfigure}
    
    \caption{\small \textbf{Visual Comparison of Retrieval Methods for Warthog (Imagenet)}.} 
    \label{fig:qualitative_results_Warthog}
    \vspace{-.1in}
\end{figure*}

\begin{figure*}[htp!]
    \centering
    \begin{subfigure}{.3\textwidth}
        \includegraphics[width=\textwidth]{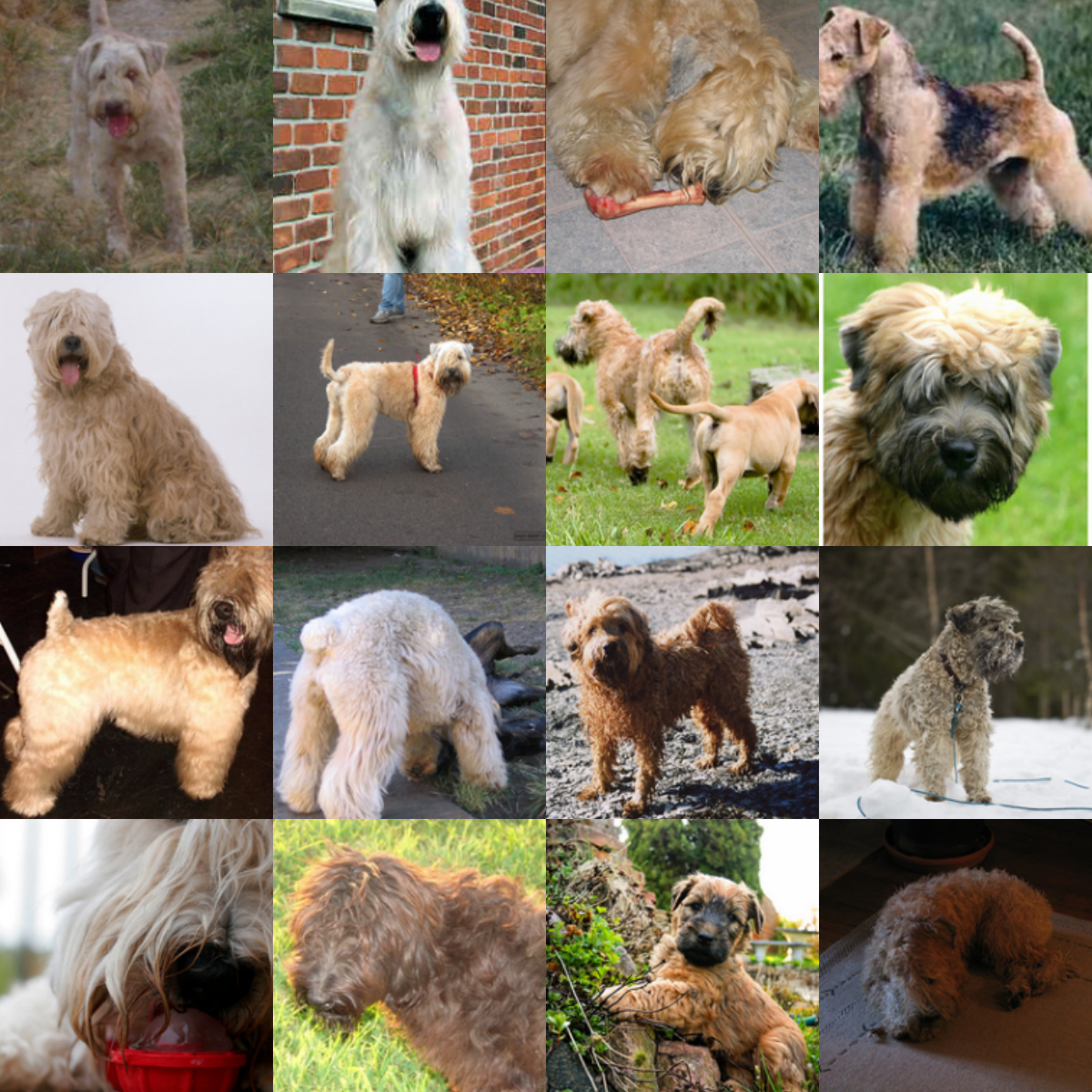}
        \caption{Target Images}
    \end{subfigure}\quad
    \begin{subfigure}{.3\textwidth}
        \includegraphics[width=\textwidth]{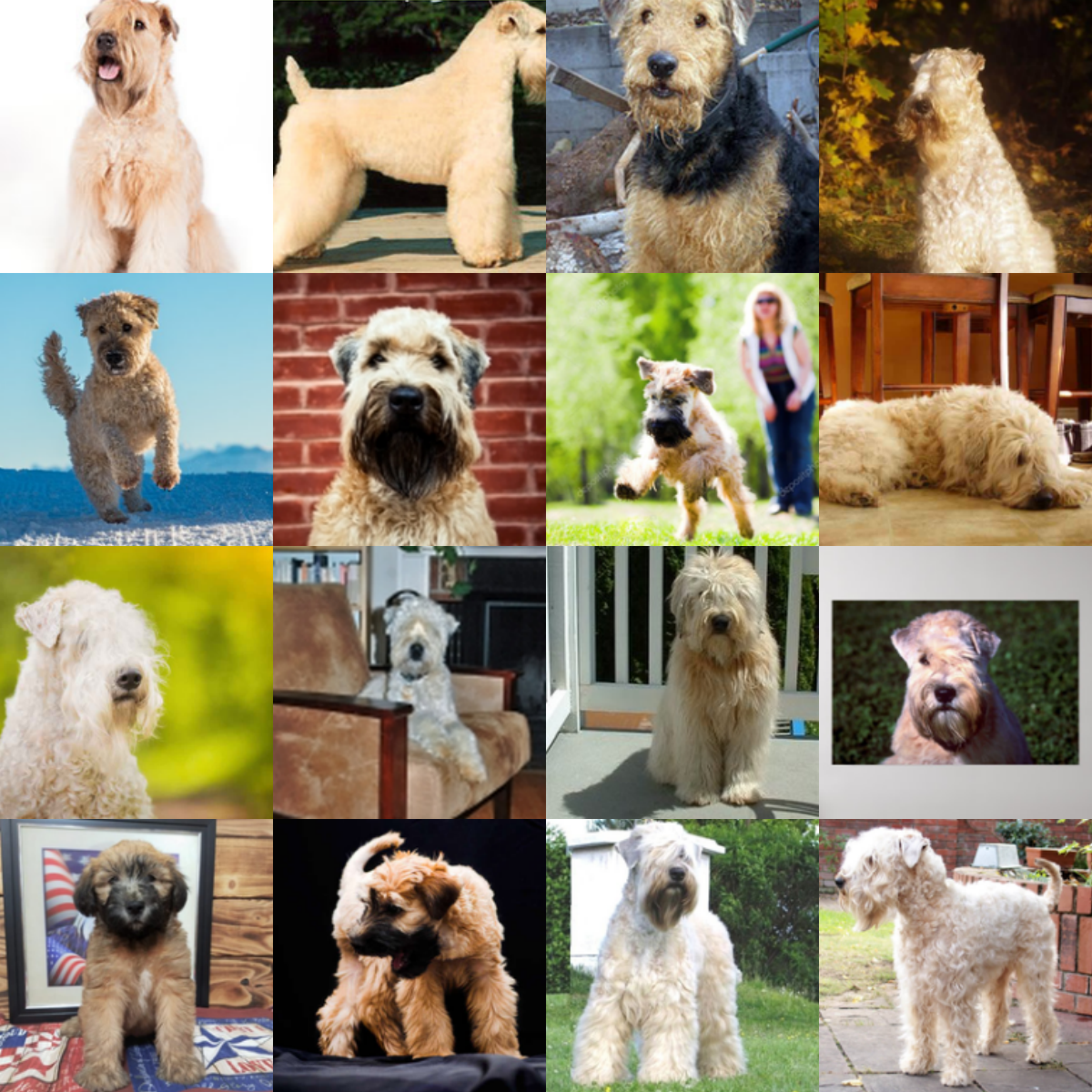}
        \caption{COBRA Images}
    \end{subfigure}\quad
    \begin{subfigure}{.3\textwidth}
        \includegraphics[width=\textwidth]{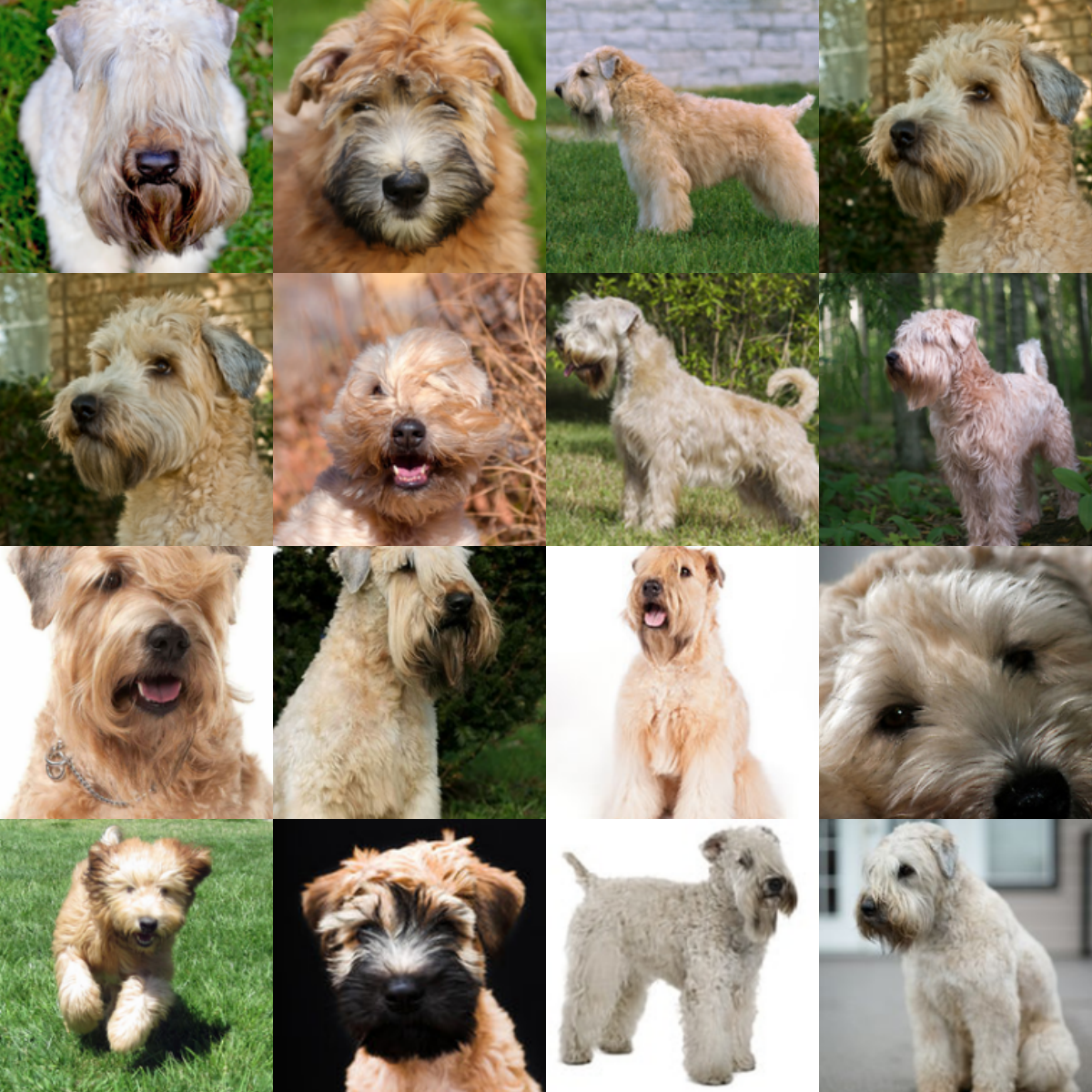}
        \caption{Sim-Score Images}
    \end{subfigure}
    
    \vspace{0.1cm}
    
    \begin{subfigure}{.3\textwidth}
        \includegraphics[width=\textwidth]{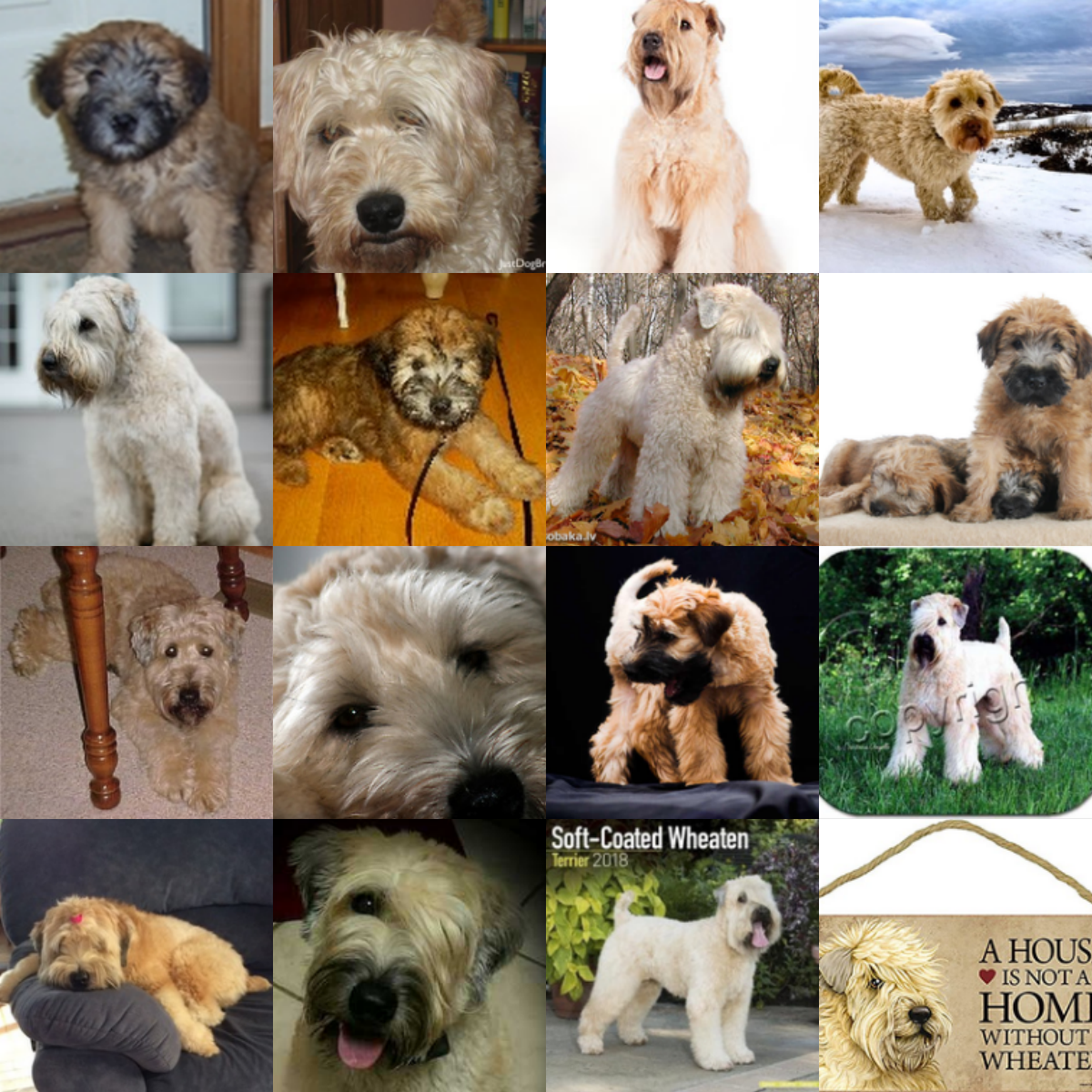}
        \caption{CLIP-Score Images}
    \end{subfigure}\quad
    \begin{subfigure}{.3\textwidth}
        \includegraphics[width=\textwidth]{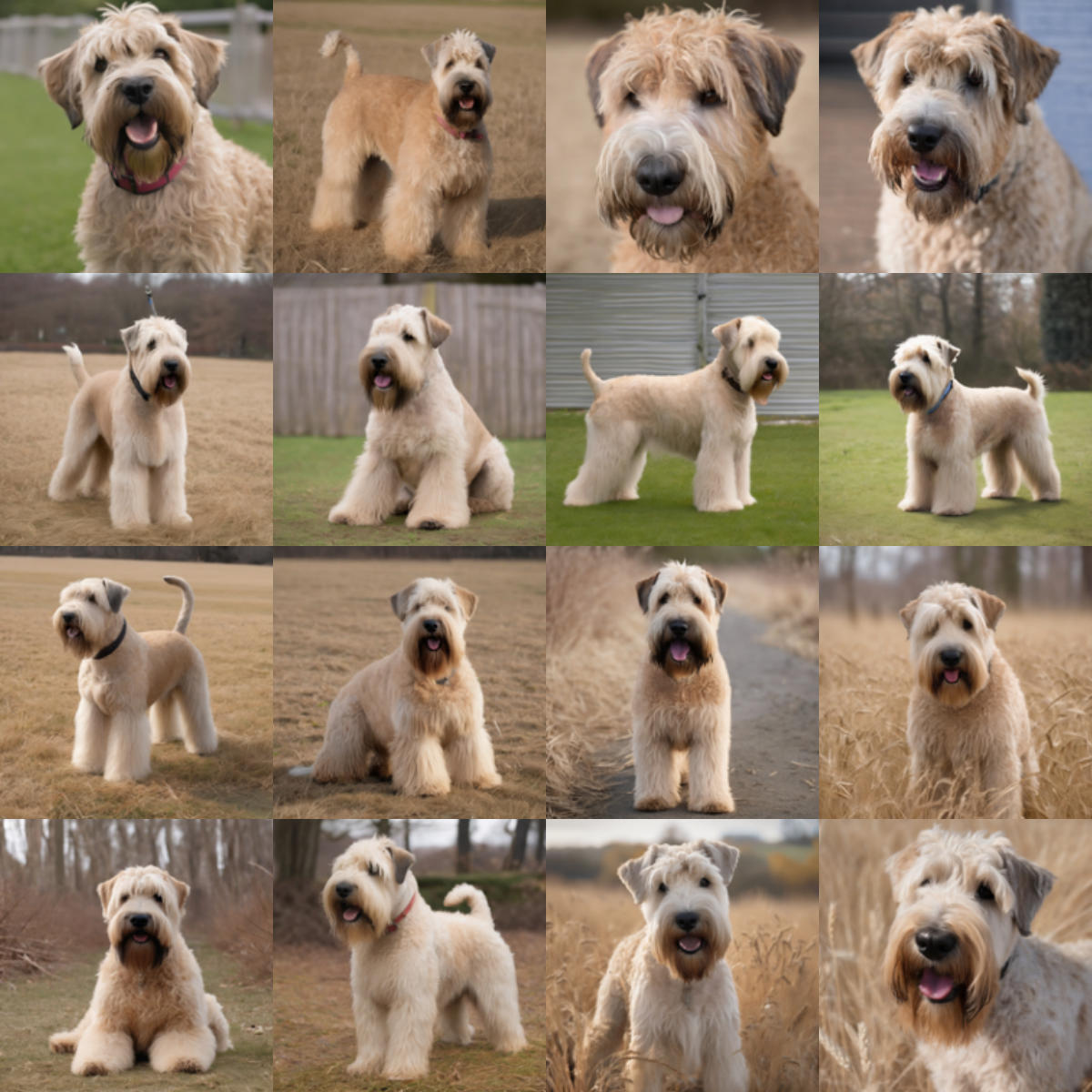}
        \caption{SDXL Images}
    \end{subfigure}\quad
    \begin{subfigure}{.3\textwidth}
        \includegraphics[width=\textwidth]{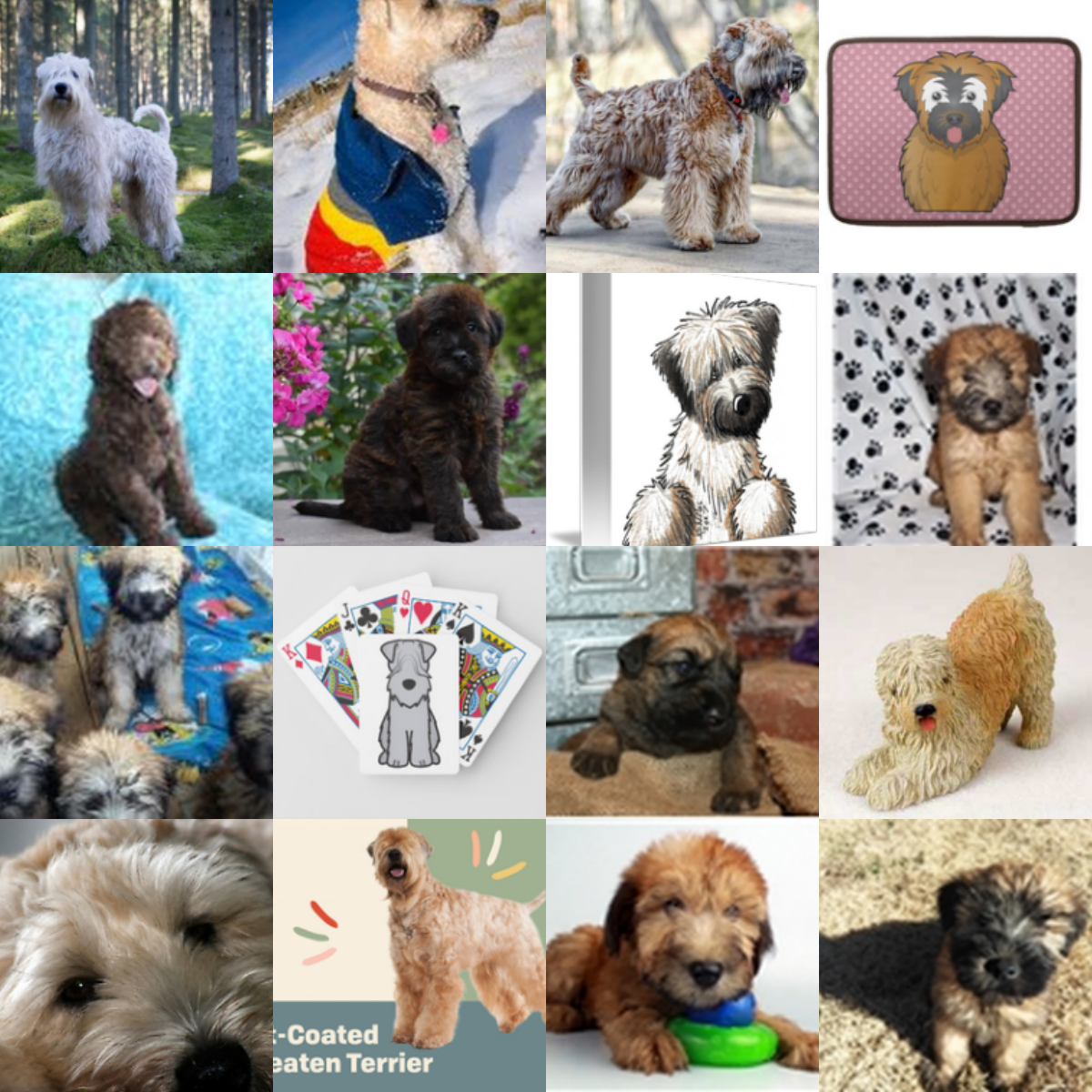}
        \caption{Random Images}
    \end{subfigure}
    
    \caption{\small \textbf{Visual Comparison of Retrieval Methods for Wheaton Terrier (Imagenet)}.} 
    \label{fig:qualitative_results_Wheaton}
    \vspace{-.1in}
\end{figure*}

\newpage
\subsection{Text Example}
See ~\Cref{fig:text_qual} for a qualitative example of COBRA-based text retrieval. 
\label{app:text_qual}
\begin{figure*}[h]
    \includegraphics[width=\textwidth]{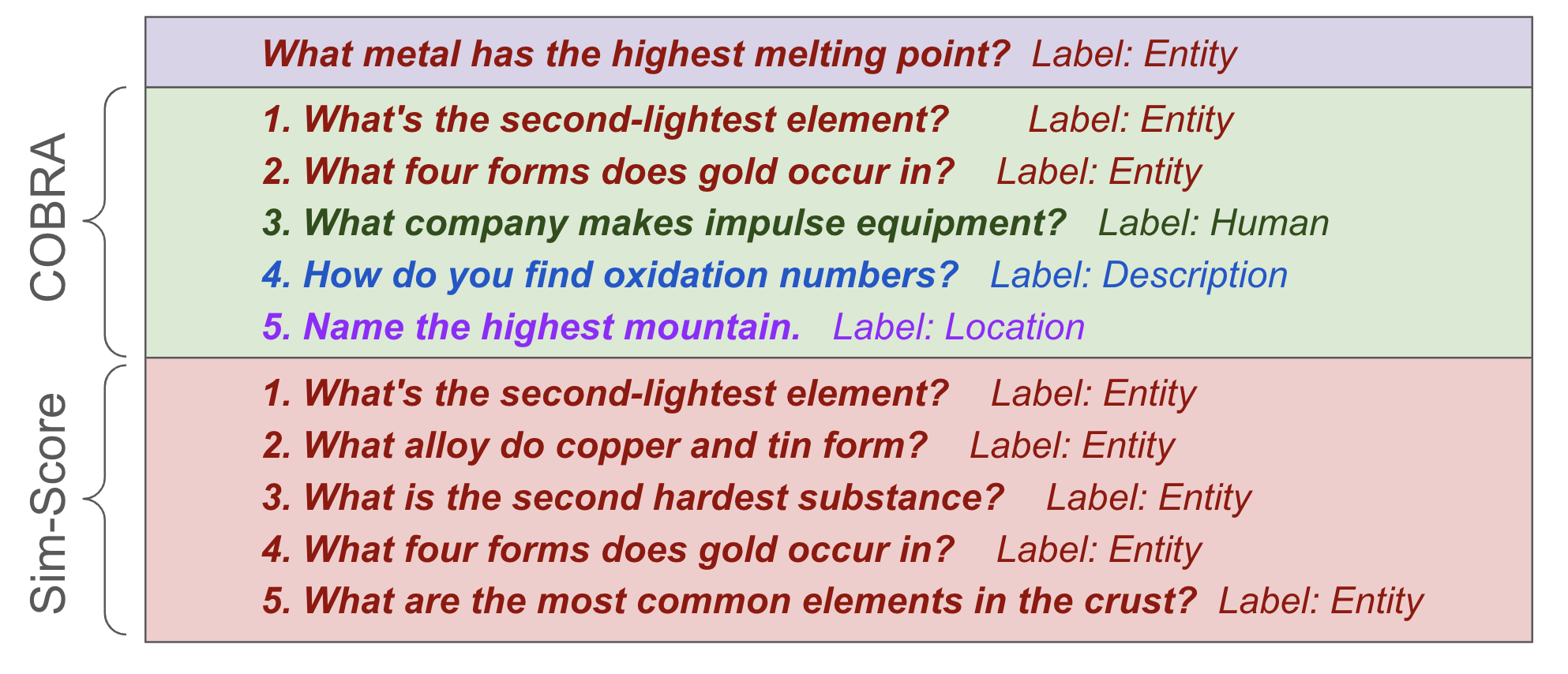}
    \caption{\small \textbf{ICL Qualitative Example} An example of COBRA vs Sim-Score on the TREC dataset~\cite{trec}}
    \label{fig:text_qual}
\end{figure*}

\newpage
\clearpage
\subsection{Flowers102 Qualitative Discussion}
\label{appen sec: qualitative flowers}
When using Flowers102 as the target dataset, we do not observe any statistically significant difference between using Sim-Score vs COBRA for retrieving samples from LAION-2B. We speculate that this is due the fact that there is very little difference in distribution between the train set and test set. Qualitatively in Fig. 12c, COBRA retrieves a very diverse set of images that do contain "Air Plant" in many different contexts and captures the naturally occurring variance in this semantic class. However since the true train and test set for Flowers-102 are very similar, the utility of diverse retrieval diminishes. 

\begin{figure*}[h!]
    \centering
    \begin{subfigure}{.45\textwidth}
        \includegraphics[width=\textwidth]{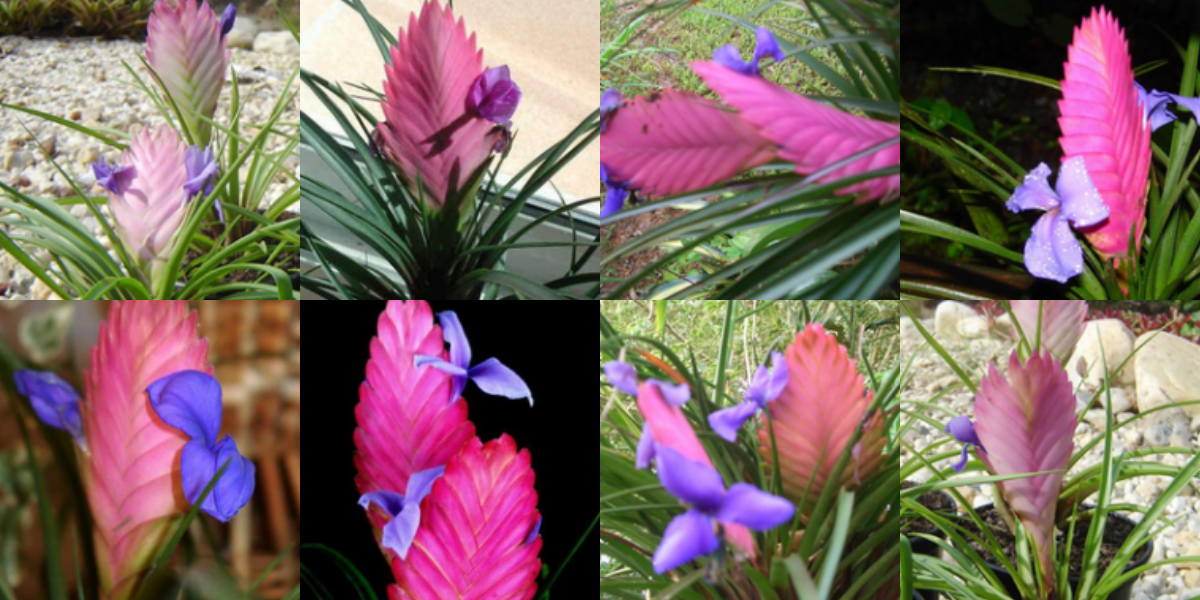}
        \caption{\textbf{Train Images}}
    \end{subfigure}\quad
    \begin{subfigure}{.45\textwidth}
        \includegraphics[width=\textwidth]{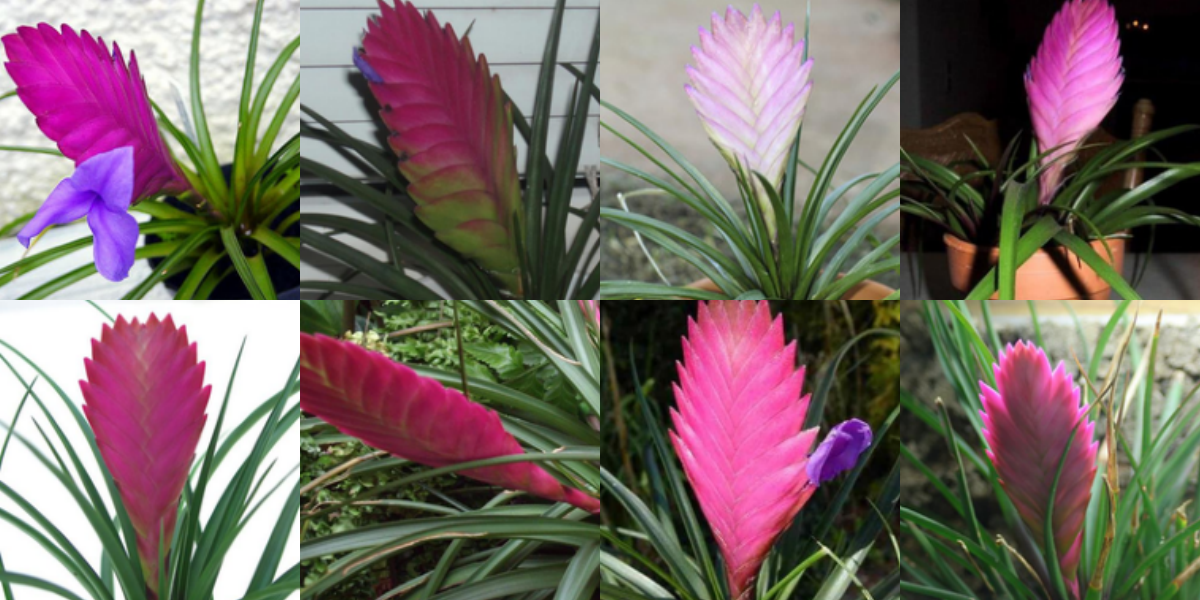}
        \caption{\textbf{Test Images}}
    \end{subfigure}\quad
    
    \vspace{0.1cm}

    \begin{subfigure}{.45\textwidth}
        \includegraphics[width=\textwidth]{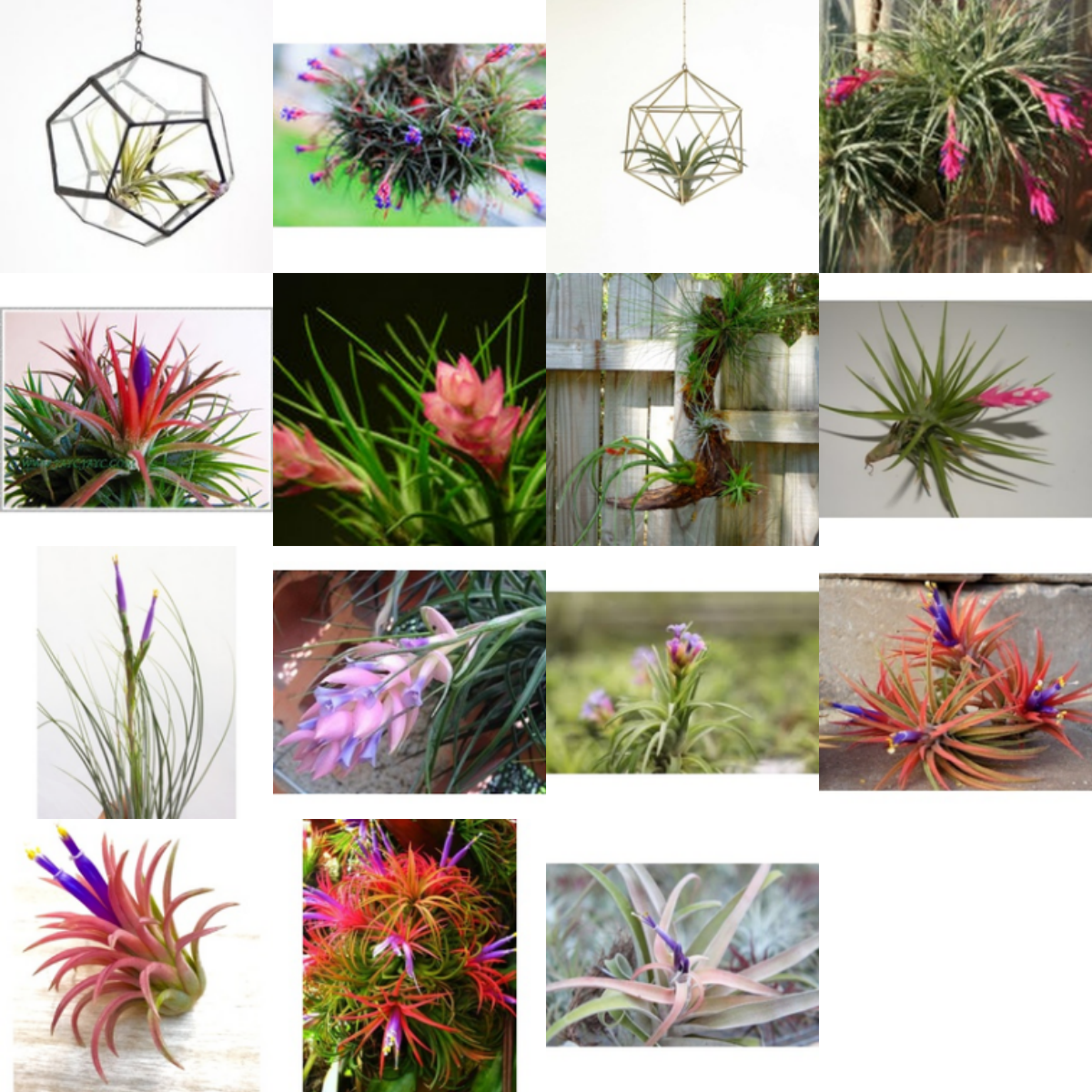}
        \caption{\textbf{COBRA Images}}
    \end{subfigure}\quad
    \begin{subfigure}{.45\textwidth}
        \includegraphics[width=\textwidth]{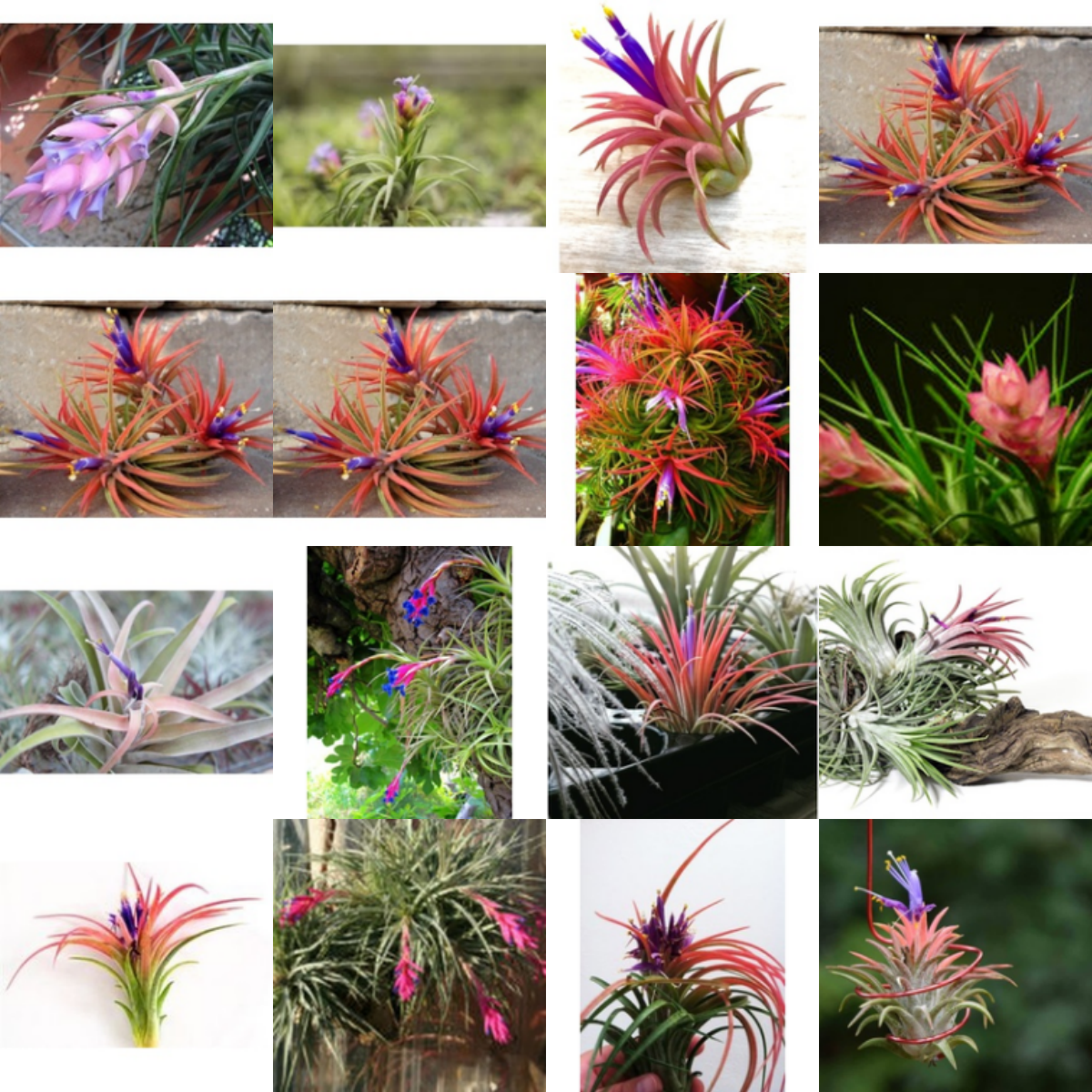}
        \caption{\textbf{Sim-Score Images}}
    \end{subfigure}\quad
    
    \caption{\small \textbf{Visual Comparison of Retrieval Methods for Class Air Plant (Flowers-102)}. We observe there is very little difference in the train and test distributions in Flowers-102. Therefore, the utility of diversity in retrieval is diminished so we do not see statistically significant differences between COBRA and Sim-Score in terms of performance. } 
    \label{fig:qualitative_results_flowers}
    \vspace{-.1in}
\end{figure*}

\newpage
\clearpage
\subsection{Other Measures of Diversity}
\paragraph{Measuring Diversity} In~\cref{fig:vendi}, we present analysis using the Vendi Score~\citep{friedman2022vendi} which creates a similarity matrix ($\Wmat$) between the datapoints using a kernel function. It then computes the score as the exponential of the entropy of eigen values of the normalized similarity matrix ($\Wmat/n$, where $n$ refers to the number of datapoints).  We find that nearest-neighbor methods (CLIP/Sim-Score) yield low diversity and that overly diverse methods (log-Det-MI, MMR)  underperform by including noisy images. We note performance is a function of both diversity and relevance, but the Vendi score only considers the former. 

\begin{figure}[h]
    \vspace{-0.05in} 
    \centering
    \includegraphics[width=.4\columnwidth, height=.33\columnwidth]{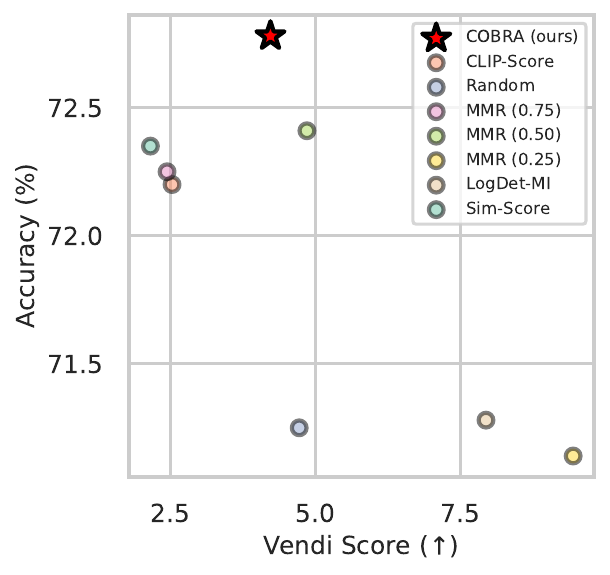}
    \caption{{Analysis using ~\cite{friedman2022vendi} w/ Tip-Adapter-F}}
    \label{fig:vendi}
\end{figure}





\end{document}